\documentclass[final,onefignum,onetabnum]{siamart171218}
\usepackage{lipsum}
\usepackage{amsfonts}
\usepackage{graphicx}
\usepackage{epstopdf}
\usepackage{algorithmic}
\ifpdf
  \DeclareGraphicsExtensions{.eps,.pdf,.png,.jpg}
\else
  \DeclareGraphicsExtensions{.eps}
\fi

\newsiamremark{remark}{Remark}
\newsiamremark{hypothesis}{Hypothesis}
\crefname{hypothesis}{Hypothesis}{Hypotheses}
\newsiamthm{claim}{Claim}

\headers{ALS in a Bayesian framework for low-rank tensor approximation}{Clara Menzen, Manon Kok, Kim Batselier}

\title{Alternating linear scheme in a Bayesian framework \\for low-rank tensor approximation}

\author{Clara Menzen\footnotemark[1] \and Manon Kok\footnotemark[1] \and Kim Batselier\thanks{Delft Center for Systems and Control, TU Delft, the Netherlands (\email{c.m.menzen@tudelft.nl}, \email{m.kok-1@tudelft.nl}, \email{k.batselier@tudelft.nl}}).}

\usepackage{amsopn}

\makeatletter
\newcommand*{\addFileDependency}[1]{
  \typeout{(#1)}
  \@addtofilelist{#1}
  \IfFileExists{#1}{}{\typeout{No file #1.}}
}
\makeatother

\newcommand{\y}{\mathbf{y}}
\newcommand{\Yt}{\mathbfcal{Y}}
\newcommand{\g}{\mathbf{g}}
\newcommand{\G}{\mathbf{G}}
\newcommand{\Gt}{\mathbfcal{G}}
\newcommand{\U}{\mathbf{U}}
\newcommand{\Q}{\mathbf{Q}}
\newcommand{\q}{\mathbf{q}}
\newcommand{\R}{\mathbf{R}}
\newcommand{\e}{\mathbf{e}}
\newcommand{\m}{\mathbf{m}}
\newcommand{\x}{\mathbf{x}}
\newcommand{\h}{\mathbf{h}}

\newcommand{\cov}{\mathbf{P}}
\newcommand{\id}{\mathbf{I}}
\DeclareMathAlphabet\mathbfcal{OMS}{cmsy}{b}{n}
\definecolor{blue}{RGB}{3, 128, 149}
\definecolor{green}{RGB}{59, 151, 52}
\definecolor{red}{RGB}{248, 60, 93}
\definecolor{yellow}{RGB}{247, 179, 51}
\definecolor{lightred}{RGB}{253, 223, 223}
\definecolor{lightblue}{RGB}{222, 243, 253}
\definecolor{lightgreen}{RGB}{191, 216, 213}
\definecolor{lightyellow}{RGB}{246, 234, 197}

\usepackage{pgfplots}
\usepgfplotslibrary{patchplots}
\pgfplotsset{compat=1.14}
\usepackage{array}
\usepackage{setspace}
\usepackage{url}
\usepackage{tikz}
\usepackage{tcolorbox}
\usepackage{xcolor}
\usepackage{blindtext}
\usepackage{mathtools}
\usepackage{amssymb}

\pgfplotsset{compat=newest}
\pgfplotsset{plot coordinates/math parser=false}

\newlength\fheight
\newlength\fwidth

\newlength\mylen
\makeatletter
\def\@cite#1#2{{\normalfont[{\mdseries#1\if@tempswa , #2\fi}]}}
\makeatother
\ifpdf
\hypersetup{
  pdftitle={ALS in a Bayesian framework},
  pdfauthor={C. Menzen, M. Kok, and K. Batselier}
}
\fi

\begin{document}

\maketitle

\begin{abstract}
  Multiway data often naturally occurs in a tensorial format which can be approximately represented by a low-rank tensor decomposition. This is useful because complexity can be significantly reduced and the treatment of large-scale data sets can be facilitated. In this paper, we find a low-rank representation for a given tensor by solving a Bayesian inference problem. This is achieved by dividing the overall inference problem into sub-problems where we sequentially infer the posterior distribution of one tensor decomposition component at a time. This leads to a probabilistic interpretation of the well-known iterative algorithm alternating linear scheme (ALS). In this way, the consideration of measurement noise is enabled, as well as the incorporation of application-specific prior knowledge and the uncertainty quantification of the low-rank tensor estimate. To compute the low-rank tensor estimate from the posterior distributions of the tensor decomposition components, we present an algorithm that performs the unscented transform in tensor train format. 
\end{abstract}



\section{Introduction}
\label{sec:intro}
Low-rank approximations of multidimensional arrays, also called tensors, have become a central tool in solving large-scale problems. The numerous applications include machine learning (e.g.\ tensor completion \cite{Song2019,Zhao2016,Grasedyck2015}, kernel methods \cite{Signoretto2011a,Chen2019} and deep learning \cite{Cohen2016,Novikov2015}), signal processing \cite{Sidiropoulos2017,Caiafa2015}, probabilistic modeling \cite{Izmailov2018a,Zhao2013}, non-linear system identification \cite{Favier2012,Batselier2017} and solving linear systems \cite{Oseledets2012,Dolgov2014}. An extensive overview of applications can be found, e.g., in \cite{Cichocki2017}.\\

In many applications, it is possible to represent a tensor with a low-rank approximation $\Yt_\mathrm{lr}$, without losing the most meaningful information \cite{Cichocki2016}. In the presence of uncorrelated noise $\mathbfcal{E}$, however, the tensor representing the data $\Yt$, loses the low-rank structure. The data tensor $\Yt\in\mathbb{R}^{I_1\times I_2 \times \dots \times I_N}$ can be modeled as
\begin{equation}
\Yt=\Yt_\mathrm{lr}+\mathbfcal{E},\qquad \operatorname{vec}(\mathbfcal{E}) \sim \mathcal{N}(\mathbf{0},\sigma^2\id),
\label{eq:lrnoisy}
\end{equation}
with the vectorized noise being modeled as Gaussian with zero mean and a variance of $\sigma^2$ and $\id$ denoting the identity matrix, which in this case is of size $I_1I_2\dots I_N\times I_1I_2\dots I_N$. In this work, we solve a Bayesian inference problem to seek a low-rank approximation of an observed noisy tensor $\Yt$ that corresponds to the underlying tensor~$\Yt_\mathrm{lr}$. A low-rank approximation can be found by using a tensor decomposition (TD). Examples of TDs are the CANDECOMP/PARAFAC (CP) decomposition \cite{Carroll1970,Harshman1970}, the Tucker decomposition \cite{Tucker1966} and the tensor train (TT) decomposition \cite{Oseledets2011}. In general, TDs solve an optimization problem of the form
\begin{equation}
\min_\Gt ||\Yt-\Gt||,
\label{eq:optproblem}
\end{equation}
where $\Yt$ is the measured, noisy tensor and $\Gt$ is a low-rank tensor decomposition. There exist multiple methods to find a decomposition for a given tensor. The approach that we are looking at in this paper is the well-known iterative method alternating linear scheme (ALS). The ALS has been studied extensively and has successfully been applied to find low-rank tensor decompositions. The ALS for the CP decomposition is described in \cite{Kolda2009,Comon2009}, the Tucker decomposition is also treated in \cite{Kolda2009} and the ALS for the TT decomposition is studied in \cite{Rohwedder2012,Rohwedder2013}. The ALS optimizes the sought tensor on a manifold with fixed ranks \cite[p.\ 1136]{Rohwedder2013}. Imposing the low-rank rank constraint is therefore easy to implement by choosing the ranks in advance. \\

The CP, Tucker and TT decomposition are all multilinear functions of all the TD components. This means that by assuming all TD components except the $n$th to be known, the tensor becomes a linear expression in the $n$th component \cite[p.\ 4]{Comon2009}. In the ALS all TD components are updated sequentially by making use of the TD's multilinearity. Each update step requires to solve a linear least squares problem given by
\begin{equation}
\min_{\g_n} ||\y - \U_{\setminus n} \g_n ||_\text{F}, \label{eq:mingn}
\end{equation}
where $\y\in\mathbb{R}^{I_1I_2\dots I_N\times1}$ and $\g_n\in\mathbb{R}^{K\times1}$ denote the vectorization of $\Yt$ and $\Gt_n$, respectively, $K$ being the number of elements in the $n$th TD component. The matrix $\U_{\setminus n} \in \mathbb{R}^{J \times K}$ is a function off all TD components except the $n$th, where $J$ is the number of elements in $\y$, and $||\cdot||_\text{F}$ denotes the Frobenius norm. \\

A drawback of the ALS is that it does not explicitly model the measurement noise~$\mathbfcal{E}$, which in real-life applications is usually present. In this work, we model the noise by approaching the tensor decomposition in a Bayesian framework, treating all components as probability distributions. In this way, finding a low-rank TD approximation can be solved as a Bayesian inference problem: given the prior distributions  of the TD components $p({\g_i})$ and the measurements $\y$, the posterior distribution $p(\{\g_i\}\mid \y)$ can be found by applying Bayes' rule
\begin{align}
p\left(\{\g_i\}\mid \y\right)
= \frac{\overbrace{p(\y|\{\g_i\})}^{\text{likelihood}}\overbrace{p(\{\g_i\})}^{\text{prior}}}{\underbrace{p(\y)}_{\text{evidence}}},
\label{eq:Bayes}
\end{align}
where $\{\g_i\}$ denotes the collection of all TD components $\g_i,\;\text{for}\; i=1,\dots,N$. We assume that the likelihood and prior are Gaussian and, likewise in the ALS, we apply a block coordinate descent \cite[p.\ 230]{Nocedal2006}, leading to a tractable inference, where the posterior density of the low-rank tensor estimate can be computed from \cref{eq:Bayes}. \\

Solving the low-rank tensor approximation problem in a Bayesian way has the following benefits. The assumptions on the measurement noise $\mathbfcal{E}$ are considered and the uncertainty of each tensor decomposition component $\g_n$ is quantified. Furthermore, prior knowledge can be explicitly taken into account and the resulting low-rank tensor estimate comes with a measure of uncertainty. We illustrate the benefits with numerical experiments.\\

Our main contribution is to approach the low-rank tensor approximation problem from a Bayesian perspective, treating all TD components as Gaussian random variables. This results in a probabilistic ALS algorithm. We ensure numerical stability by incorporating the orthogonalization step, present in the ALS algorithm for the TT decomposition, into the probabilistic framework. In addition, we propose an algorithm to approximate the mean and covariance of the low-rank tensor estimate's posterior density with the unscented transform in tensor train format. Our open-source MATLAB implementation can be found on \url{https://gitlab.tudelft.nl/cmmenzen/bayesian-als}.

\subsection*{Related Work}
Our work is related to inferring low-rank tensor decompositions with Bayesian methods for noisy continuous-valued multidimensional observations. While most literature considers either the CP or Tucker decomposition, our paper mainly focuses on the TT decomposition, but is also applicable to CP and Tucker. Also, in contrast to our paper, the related work mainly treats tensors with missing values. The main difference to the existing literature, however, are the modeling choices. While the existing work proposes different methods to approximate the inference of the TD components, our work allows us to perform tractable inference. This is mainly because we use the ALS, a block-coordinate descent method, to infer the TD components. Also, we assume that all TD components are Gaussian random variables and that they are all independent. Thus, our method is preferable when these assumptions can be made for a given application.

In \cite{Rai2015}, \cite{Rai2014} and \cite{Xiong2010} inference is performed with Gibbs sampling, using Gaussian priors for the columns of the CP decomposition's factor matrices. Variational Bayes is applied in \cite{Zhao2015} and \cite{Zhao2016}. The recovery of orthogonal factor matrices, optimizing on the Stiefel manifold with variational inference is treated by \cite{Cheng2017}. The Bayesian treatment of a low-rank Tucker decomposition for continuous data has been studied using variational inference \cite{Chu2009,Zhao2015a} and using Gibbs sampling \cite{Hoff2016}. Furthermore, an infinite Tucker decomposition based on a $t$-process, which is a kernel-based non-parametric Bayesian generalization of the low-rank Tucker decomposition, is proposed by \cite{Xu2015}. The first literature about the probabilistic treatment of the tensor train decomposition using von-Mises-Fisher priors on the orthogonal cores and variational approximation with evidence lower bound is introduced by \cite{Hinrich2019}. Recently, \cite{Hinrich2020} published the probabilistic tensor decomposition toolbox for MATLAB, providing inference with variational Bayes and with Gibbs sampling.

\section{Tensor basics and notation}
An $N$-way tensor $\Yt \in \mathbb{R}^{I_1\times I_2 \times \dots \times I_N}$ is a generalization of a vector or a matrix to higher dimensions, where $N$ is often referred to as the order of the tensor. We denote tensors by calligraphic, boldface, capital letters (e.g.\ $\Yt$) and matrices, vectors and scalars by boldface capital (e.g.\ $\mathbf{Y}$), boldface lowercase (e.g.\ $\y$) and italic lower case (e.g.\ $y$) letters, respectively. To facilitate the description and computation of tensors, we use a graphical notation as depicted in \cref{fig:Nwaytensor}. The nodes represent a scalar, a vector, a matrix and an $N$-way tensor and edges correspond to a specific index. The number of edges is equal to how many indices need to be specified to identify one element in the object, e.g.\ row and column index for matrices. An identity matrix is generally denoted by $\id$. Its size is either specified in the context or as a subscript. \\

Often it is easier to avoid working with the tensors directly, but rather with a matricized or vectorized version of them. Therefore, we revise some useful definitions. In this context, a mode of a tensor refers to a dimension of the tensor.
\begin{definition} [\textbf{mode-$n$-unfolding} {\cite[p.\ 459-460]{Kolda2009}}]
\label{def:modenunf}
The transformation of an $N$-way tensor into a matrix with respect to a specific mode is called the mode-$n$ unfolding. It is denoted by 
\begin{equation*}
    \mathbf{Y}_{(n)} \in \mathbb{R}^{I_n \times I_1\dots I_{n-1}I_{n+1} \dots I_N}.
\end{equation*}
\end{definition}

The vectorization is a special case of the unfolding, denoted by the operator name vec() and defined as

\begin{equation*}
    \operatorname{vec}(\Yt)= \y \in \mathbb{R}^{I_1I_2\dots I_N \times 1}.    
\end{equation*}

Tensors can be multiplied with matrices defined as follows.
\begin{definition}[\textbf{$n$-mode product} {\cite[p.\ 460]{Kolda2009}}]
The $n$-mode product is defined as the multiplication of a tensor $\mathbfcal{X}\in\mathbb{R}^{I_1\times \dots \times I_n\times \dots \times I_N}$ with a matrix $\mathbf{A}\in \mathbb{R}^{J \times I_n}$ in mode $n$, written as
\begin{equation*}
\mathbfcal{X}\times_n \mathbf{A} \in \mathbb{R}^{I_1 \times \dots \times I_{n-1} \times J \times I_{n+1} \times \dots \times I_N}.
\end{equation*}
Element-wise, the ($i_1,i_2,..,i_{n-1},j,i_{n+1},...,i_N$)-th entry of the result can be computed as
\begin{equation*}
\sum_{i_n=1}^{I_n} \mathbfcal{X}(i_1, i_2, \dots ,i_N) \mathbf{A}(j, i_n).
\end{equation*}
\end{definition}

\begin{definition}[\textbf{Kronecker product} {\cite[p.\ 461]{Kolda2009}}]
The Kronecker product of matrices $\mathbf{A} \in \mathbb{R}^{I\times J}$ and $\mathbf{B} \in \mathbb{R}^{K\times L}$ is denoted by $\mathbf{A} \otimes \mathbf{B}$. The result is a matrix of size $(KI) \times (LJ)$ and is defined by
\begin{equation*}
\begin{aligned}
\mathbf{A} \otimes \mathbf{B} &=\left[\begin{array}{cccc}
a_{11} \mathbf{B} & a_{12} \mathbf{B} & \cdots & a_{1 J} \mathbf{B} \\
a_{21} \mathbf{B} & a_{22} \mathbf{B} & \cdots & a_{2 J} \mathbf{B} \\
\vdots & \vdots & \ddots & \vdots \\
a_{I 1} \mathbf{B} & a_{I 2} \mathbf{B} & \cdots & a_{I J} \mathbf{B}
\end{array}\right].
\end{aligned}
\end{equation*}
\end{definition}

\begin{definition}[\textbf{Kathri-Rao product}
{\cite[p.\ 462]{Kolda2009}}]
The Khatri–Rao product of matrices $\mathbf{A}\in \mathbb{R}^{I\times K}$ and $\mathbf{B} \in \mathbb{R}^{J\times K}$ is denoted by $\mathbf{A} \odot \mathbf{B}$. The result is a matrix of size $(JI) \times K$ defined by
\begin{equation*}
\mathbf{A} \odot \mathbf{B}=\left[\begin{array}{llll}
\mathbf{a}_1 \otimes \mathbf{b}_1 & \mathbf{a}_2 \otimes \mathbf{b}_2 & \cdots & \mathbf{a}_K \otimes \mathbf{b}_K
\end{array}\right].
\end{equation*}
\end{definition}
The visual depictions of two important matrix operations are shown in \cref{fig:contraction}. On the left, a product between matrices $\mathbf{A}\in \mathbb{R}^{I\times K}$ and $\mathbf{B}\in \mathbb{R}^{K\times J}$ is shown, where the summation over the middle index $K$, also called contraction, is represented as an edge that connects both nodes. On the right, an outer product between matrices $\mathbf{A}\in \mathbb{R}^{I_1\times I_2}$ and $\mathbf{B}\in \mathbb{R}^{J_1\times J_2}$ is shown, where the dotted lines represent a rank-1 connection. The resulting matrix is the Kronecker product $\mathbf{A} \otimes \mathbf{B}\in \mathbb{R}^{I_1J_1\times I_2J_2}$.\\

\begin{figure}
\centering
\begin{tikzpicture}
\draw  [line width=0.3mm](-3.4,2) ellipse (0.4 and 0.4);
\draw  [line width=0.3mm](-2.2,2) ellipse (0.4 and 0.4);
\draw  [line width=0.3mm](-1,2) ellipse (0.4 and 0.4);
\draw  [line width=0.3mm](0.5,2) ellipse (0.4 and 0.4);
\node at (-3.4,2) {$y$};
\node at (-2.2,2) {$\mathbf{y}$};
\node at (-1,2) {$\mathbf{Y}$};
\node at (0.5,2) {$\mathbfcal{Y}$};
\draw [line width=0.3mm](-1.1,1.62) -- (-1.1,1);
\draw [line width=0.3mm](-2.2,1.6) -- (-2.2,1);
\draw [line width=0.3mm](-0.9,1.62) -- (-0.9,1);
\draw [line width=0.3mm](0.33,1.63) -- (0.1,1) node (v1) {};
\draw (v1);
\draw [line width=0.3mm](0.6,1.6) -- (0.7,1);
\draw [line width=0.3mm](0.2,1.7) -- (-0.1,1.1);
\draw [line width=0.3mm](0.78,1.72) -- (1.18,1.22);
\node at (0.46,1.2) {...};
\node at (0.9,1.2) {...};
\node at (1.4,1.1) {$I_N$};
\node at (0.88,0.85) {$I_n$};
\node at (-0.2,1) {$I_1$};
\node at (0.1,0.8) {$I_2$};
\node at (3.5,1.9) {$\mathbfcal{Y}\in \mathbb{R}^{I_1\times I_2\times ...\times I_n\times ...\times I_N}$};
\end{tikzpicture}
\caption{Visual depictions of a scalar, a vector, a matrix and an $N$-way tensor, where the nodes represent the object and the edges correspond to a specific index. The number of edges is equal to how many indices need to be specified to identify one element in the object, e.g.\ row and column index for matrices.}
\label{fig:Nwaytensor}
\end{figure}
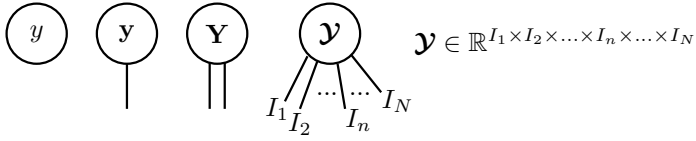

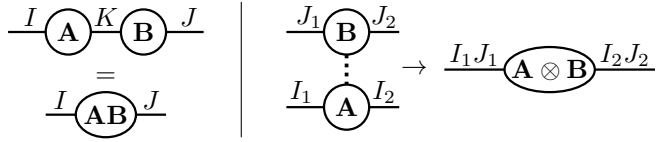
\begin{figure}
\centering
\begin{tikzpicture}

\draw [line width=0.3mm] (-2.9,2.7) ellipse (0.3 and 0.3);
\draw  [line width=0.3mm](-1.9,2.7) ellipse (0.3 and 0.3);
\draw [line width=0.3mm](-2.6,2.7) -- (-2.2,2.7);
\draw [line width=0.3mm](-3.2,2.7) -- (-3.7,2.7);
\draw [line width=0.3mm](-1.6,2.7) -- (-1.1,2.7);
\node at (-2.9,2.7) {$\mathbf{A}$};
\node at (-1.9,2.7) {$\mathbf{B}$};
\node at (-2.4,2.9) {$K$};

\draw  [line width=0.3mm](0.8,2.7) ellipse (0.3 and 0.3);
\draw  [line width=0.3mm](0.8,1.7) ellipse (0.3 and 0.3);
\draw [dotted, line width=0.5mm](0.8,2.4) -- (0.8,2);
\draw [line width=0.3mm](1.1,2.7) -- (1.5,2.7);
\node at (0.8,2.7) {$\mathbf{B}$};
\node at (0.8,1.7) {$\mathbf{A}$};

\draw [line width=0.3mm](0.5,2.7) -- (0,2.7);
\draw [line width=0.3mm](1.5,1.7) -- (1.1,1.7);
\draw [line width=0.3mm](0.5,1.7) -- (0,1.7);
\draw (-0.6,3.1) -- (-0.6,1.3);
\node at (1.7,2.2) {$\rightarrow$};
\draw  [line width=0.3mm](3.5,2.2) ellipse (0.6 and 0.3);
\node at (3.5,2.2) {$\mathbf{A}\otimes\mathbf{B}$};
\draw [line width=0.3mm](2.1,2.2) -- (2.9,2.2);
\draw [line width=0.3mm](4.1,2.2) -- (4.9,2.2);
\node at (0.3,2.9) {$J_1$};
\node at (1.3,2.9) {$J_2$};
\node at (0.2,1.9) {$I_1$};
\node at (1.3,1.9) {$I_2$};
\node at (2.5,2.4) {$I_1J_1$};
\node at (4.5,2.4) {$I_2J_2$};
\node at (-2.4,2.1) {=};
\draw  [line width=0.3mm](-2.4,1.6) ellipse (0.4 and 0.3);
\node at (-2.4,1.6) {$\mathbf{AB}$};
\draw [line width=0.3mm](-2,1.6) -- (-1.6,1.6);
\draw [line width=0.3mm](-2.8,1.6) -- (-3.2,1.6);
\node at (-3.4,2.9) {$I$};
\node at (-1.3,2.9) {$J$};
\node at (-3,1.8) {$I$};
\node at (-1.8,1.8) {$J$};
\end{tikzpicture}
\caption{Left: Visual depictions of an index contraction between matrices $\mathbf{A}$ and $\mathbf{B}$. Right: Visual depictions of an outer product between matrices $\mathbf{A}$ and $\mathbf{B}$. The dotted line represents a summation over a rank-1 one connection. The resulting matrix is computed as the Kronecker product~ $\mathbf{A}\otimes\mathbf{B}$.}
\label{fig:contraction}
\end{figure}

A tensor can be expressed as a function of simpler tensors that form a tensor decomposition. An extensive review about TDs can be found in \cite{Kolda2009}. The most notable are the CP decomposition, the Tucker decomposition, and the TT decomposition.

\begin{definition}[\textbf{CP decomposition} \cite{Carroll1970,Harshman1970}] 
\label{def:cpd}
The CP decomposition consists of a set of matrices $\G_i\in \mathbb{R}^{I_i \times R}$, $i=1,..,N$, called factor matrices and a weight vector $\boldsymbol\lambda\in\mathbb{R}^{R\times 1}$  that represent a given $N$-way tensor $\Yt$. Element-wise, the ($i_1,i_2,..,i_N$)-th entry of $\Yt$ can be computed as
\begin{equation*}
\sum_{r=1}^{R} \boldsymbol{\lambda}(r) \G_1(i_1, r) \cdots \G_N(i_N, r),
\end{equation*}
where $R$ denotes the rank of the decomposition.
\end{definition}

\begin{definition}[\textbf{Tucker decomposition} \cite{Tucker1966}]
\label{def:tucker} 
The Tucker decomposition consists of an $N$-way tensor $\mathbfcal{C}\in\mathbb{R}^{R_1\times\dots \times R_N}$, called core tensor, and a set of matrices $\G_i\in \mathbb{R}^{I_i \times R_i}$, $i=1,..,N$, called factor matrices, that represent a given $N$-way tensor $\Yt$. Element-wise, the ($i_1,i_2,..,i_N$)-th entry of $\Yt$ can be computed as
\begin{equation*}
\sum_{r_{1}=1}^{R_1} \cdots \sum_{r_N=1}^{R_{N}} \mathbfcal{C}(r_1, \ldots, r_N) \G_1(i_1, r_1) \cdots \G_N(i_N, r_N),
\end{equation*}
where $R_1,\dots R_N$ denote the ranks of the decomposition. The factor matrices can be orthogonal, such that the Frobenius norm of the entire tensor is contained in the core tensor.
\end{definition}

\begin{definition}[\textbf{The TT decomposition} \cite{Oseledets2011}]
\label{def:tt}
The tensor train decomposition consists of a set of three-way tensors $\Gt_i\in \mathbb{R}^{R_i\times I_i \times R_{i+1}}$, $i=1,..,N$ called TT-cores, that represent a given $N$-way tensor $\Yt$. Element-wise, the ($i_1,i_2,..,i_N$)-th entry of $\Yt$ can be computed as
\begin{equation*}
\sum_{r_1=1}^{R_1}\sum_{r_2=1}^{R_2} \cdots \sum_{r_{N+1}=1}^{R_{N+1}} \Gt_1(r_1,i_1,r_2)\Gt_2(r_2,i_2,r_3)\cdots\Gt_N(r_N,i_N,r_{N+1}),
\end{equation*}
where $R_1,\dots,R_{N+1}$ denote the ranks of the TT-cores and by definition $R_1= R_{N+1}=1$.
\end{definition}
If the tensor is only approximately represented by a TD, then the ranks determine the accuracy of the approximation.\\

As mentioned in \cref{sec:intro}, to formulate the linear least squares problem for one update of the ALS, the TD's property of multi-linearity is exploited and it is expressed as $\y = \U_{\setminus n}\g_n$, with $\U_{\setminus n} \in \mathbb{R}^{J \times K}$ and $\g_n\in \mathbb{R}^{K\times 1}$, where $J$ and $K$ are the number of elements of  $\Yt$ and $\Gt_n$, respectively. The following three examples describe how $\U_{\setminus n}$ is built for the CP decomposition, the Tucker decomposition, and the TT decomposition. 

\begin{example}
If a tensor is represented in terms of a CP decomposition, the matrix $\U_{\setminus n}$ can be written as
\begin{equation}
    \U_{\setminus n} = \left(\G_N \odot\dots\odot\G_{n+1}\odot\G_{n-1}\odot\dots\odot \G_1 \right)\otimes \id_{I_n}
    \label{eq:U_cp}
\end{equation}
Note that $\U_{\setminus n}$ is of size $I_nI_1\dots I_{n-1}I_{n+1}\dots I_N\times I_nR$, so the first dimension needs to be permuted in order to match $\y\in\mathbb{R}^{I_1I_2\dots I_N}$. The weight vector $\boldsymbol\lambda$ is absorbed into the factor matrix that is being updated. After each update, the columns of $\G_n$ are normalized and the norms are stored in $\boldsymbol\lambda$. The CP-ALS algorithm can be found in \cite[p.\ 471]{Kolda2009}.
\end{example}

\begin{example}[\textbf{Tucker decomposition}]
If a tensor is represented in terms of a Tucker decomposition, the matrix $\U_{\setminus n}$ can be written as
\begin{equation}
    \U_{\setminus n} = \left[\left( \G_N \otimes\dots\otimes\G_{n+1}\otimes\G_{n-1}\otimes\dots \G_1 \right) \mathbf{C}_{(n)}^\top\right]\otimes \id_{I_n}.
    \label{eq:U_tucker}
\end{equation}
Note that $\U_{\setminus n}$ is of size $I_nI_1\dots I_{n-1}I_{n+1}\dots I_N\times I_nR_n$, so the first dimension needs to be permuted in order to match $\y\in\mathbb{R}^{I_1I_2\dots I_N}$. After every update the core tensor is recomputed by solving
\begin{equation*}
    \y = (\G_N\otimes\dots\otimes\G_1)\operatorname{vec}(\mathbfcal{C}).
\end{equation*}
\end{example}

\begin{example}
\label{ex:ttd}
If the tensor is represented in terms of a TT decomposition, the matrix $\U_{\setminus n}$ can be written as 
\begin{equation}
\U_{\setminus n} = \Gt_{i>n}\otimes \id_{I_n} \otimes \Gt_{i<n}^\top \in\mathbb{R}^{I_1I_2\dots I_N\times R_nI_nR_{n+1}},
\label{eq:multilin}
\end{equation}
where $\Gt_{i<n}$ ($\Gt_{i>n}$) denotes a tensor obtained by contracting the TD components, left (right) of the $n$th core.
\end{example}   

From here on, we will focus on the tensor train decomposition. We, therefore, review some of the main concepts. A tensor train can be represented by a diagram with nodes as the TT-cores and the edges as the modes of the approximated tensor. Connected edges are the summation over the ranks between two cores (\cref{fig:TT_gen}). 
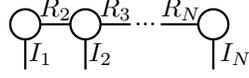
\begin{figure}
\centering
\begin{tikzpicture}
\draw  [line width=0.3mm](-3.1,2) ellipse (0.2 and 0.2);
\draw  [line width=0.3mm](-2.3,2) ellipse (0.2 and 0.2);
\draw  [line width=0.3mm](-0.6,2) ellipse (0.2 and 0.2);
\node at (-2.7,2.2) {$R_2$};
\node at (-2.9,1.6) {$I_1$};
\node at (-1.9,2.2) {$R_3$};
\node at (-2.1,1.6) {$I_2$};
\node at (-0.3,1.6) {$I_N$};
\node at (-1.03,2.2) {$R_{N}$};
\node at (-1.5,2) {...};
\draw [line width=0.3mm](-3.1,1.8) -- (-3.1,1.4);
\draw [line width=0.3mm](-2.3,1.8) -- (-2.3,1.4);
\draw [line width=0.3mm](-0.6,1.8) -- (-0.6,1.4);
\draw [line width=0.3mm](-0.8,2) -- (-1.3,2);
\draw [line width=0.3mm](-2.9,2) -- (-2.5,2);
\draw [line width=0.3mm](-2.1,2) -- (-1.7,2);
\end{tikzpicture}
\caption{Visual depiction of a tensor train decomposition with $N$ TT-cores.}
\label{fig:TT_gen}
\end{figure}
To introduce a notion of orthonormality for TT-cores, a special case of \cref{def:modenunf} is used, creating unfoldings of the TT-cores defined as follows.

\begin{definition}[\textbf{Left- and right-unfolding} {\cite[p.\ A689]{Rohwedder2012}}]
 The left-unfolding $\G^\mathrm{L}_n$ and right-unfolding $\G^\mathrm{R}_n$ of a TT-core $\Gt_n$ are the unfoldings of a core with respect to the first and last mode, respectively (\cref{fig_lr_unf}). Please note that the definition by \cite{Rohwedder2012} of the right-unfolding is the transposed version of this definition.
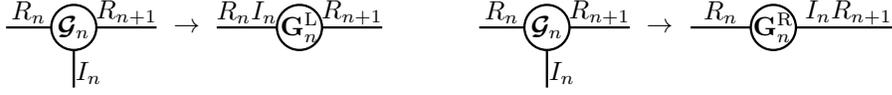
\begin{figure}
\centering
\begin{tikzpicture}
\draw  [line width=0.3mm](-2.6,3.6) ellipse (0.3 and 0.3);
\draw  [line width=0.3mm](3.69,3.6) ellipse (0.3 and 0.3);
\draw  [line width=0.3mm](0.4,3.6) ellipse (0.3 and 0.3);
\draw  [line width=0.3mm](6.7,3.6) ellipse (0.3 and 0.3);
\draw [line width=0.3mm](-2.9,3.6) -- (-3.5,3.6);
\draw [line width=0.3mm](-2.3,3.6) -- (-1.5,3.6);
\draw [line width=0.3mm](-2.6,3.3) -- (-2.6,2.8);
\draw [line width=0.3mm](3.39,3.6) -- (2.79,3.6);
\draw [line width=0.3mm](3.99,3.6) -- (4.79,3.6);
\draw [line width=0.3mm](3.69,3.3) -- (3.69,2.8);
\draw [line width=0.3mm](0.7,3.6) -- (1.5,3.6);
\draw [line width=0.3mm](0.1,3.6) -- (-0.7,3.6);
\draw [line width=0.3mm](6.4,3.6) -- (5.6,3.6);
\draw [line width=0.3mm](7,3.6) -- (8.3,3.6);
\node at (-2.6,3.6) {$\mathbfcal{G}_n$};
\node at (-3.2,3.8) {$R_n$};
\node at (-1.9,3.8) {$R_{n+1}$};
\node at (-2.4,3) {$I_n$};
\node at (3.69,3.6) {$\mathbfcal{G}_n$};
\node at (3.09,3.8) {$R_n$};
\node at (4.39,3.8) {$R_{n+1}$};
\node at (3.89,3) {$I_n$};
\node at (0.4,3.6) {$\mathbf{G}^\text{L}_n$};
\node at (6.7,3.6) {$\mathbf{G}^\text{R}_n$};
\node at (-0.3,3.8) {$R_n I_n$};
\node at (1.1,3.8) {$R_{n+1}$};
\node at (6,3.8) {$R_n$};
\node at (7.7,3.8) {$I_nR_{n+1}$};
\node at (-1.1,3.6) {$\rightarrow$};
\node at (5.19,3.6) {$\rightarrow$};
\end{tikzpicture}
\caption{Left: Visual depiction of a left-unfolding of a TT-core. Right: Right-unfolding of a TT-core.}
\label{fig_lr_unf}
\end{figure}
\end{definition}

\begin{definition}[\textbf{Left-orthogonal and right-orthogonal} {\cite[p.\ A689]{Rohwedder2012}}]
A TT-core $\Gt_n$ is called left-orthogonal, if the left-unfolding $\G^\mathrm{L}_n$ satisfies
\begin{equation*}
\left(\G_n^\mathrm{L}\right)^\top\G_n^\mathrm{L} = \id_{R_{n+1}}.
\end{equation*}
Analogously, a TT-core $\Gt_n$ is called right-orthogonal, if the right-unfolding $\G^\mathrm{R}_n$ satisfies
\begin{equation*}
\G_n^\mathrm{R} \left(\G_n^\mathrm{R}\right)^\top = \id_{R_n}.
\end{equation*}
\end{definition}

\begin{definition}[\textbf{site-$n$-mixed-canonical form} {\cite[p.\ 113]{Schollwock2011}}]
A tensor train is in site-$n$-mixed-canonical form if the TT-cores $\{\Gt_i\}_{i<n}$ are left-orthogonal and the TT-cores $\{\Gt_i\}_{i>n}$ are right-orthogonal. The $n$th TT-core is not orthogonal and it can be easily shown that
\begin{equation*}
    ||\Yt||_\mathrm{F}=||\Gt_n||_\mathrm{F}.
\end{equation*}
\end{definition}
\Cref{fig:TT_can} depicts different site-$n$-mixed-canonical forms for an exemplary three-way tensor train. On the left (right) figure, the Frobenius norm is contained in the first (last) core and all other cores are right- (left-) orthogonal, represented by the diagonal in the node.\\
\begin{figure}
\centering
\begin{tikzpicture}
\draw  [line width=0.3mm](-3.1,2) ellipse (0.2 and 0.2);
\draw  [line width=0.3mm](-2.5,2) ellipse (0.2 and 0.2);
\node (v1) at (-3.02,2) {};
\node (v2) at (-2.58,2) {};
\draw  [line width=0.3mm](v1) edge (v2);
\draw  [line width=0.3mm](-1.9,2) ellipse (0.2 and 0.2);
\node (v3) at (-2.42,2) {};
\node (v4) at (-1.98,2) {};
\draw [line width=0.3mm] (v3) edge (v4);
\node (v5) at (-3.1,1.92) {};
\node (v6) at (-3.1,1.4) {};
\node (v7) at (-2.5,1.92) {};
\node (v8) at (-2.5,1.4) {};
\node (v9) at (-1.9,1.92) {};
\node (v10) at (-1.9,1.4) {};
\draw  [line width=0.3mm](v5) edge (v6);
\draw  [line width=0.3mm](v7) edge (v8);
\draw  [line width=0.3mm](v9) edge (v10);
\node (v11) at (-2.24,1.74) {};
\node (v12) at (-2.76,2.26) {};
\draw  [line width=0.3mm](v11) edge (v12);
\node (v13) at (-1.64,1.74) {};
\node (v14) at (-2.16,2.26) {};
\draw  [line width=0.3mm](v13) edge (v14);
\end{tikzpicture}
\begin{tikzpicture}
\draw  [line width=0.3mm](-3.1,2) ellipse (0.2 and 0.2);
\draw  [line width=0.3mm](-2.5,2) ellipse (0.2 and 0.2);
\node (v1) at (-3.02,2) {};
\node (v2) at (-2.58,2) {};
\draw  [line width=0.3mm](v1) edge (v2);
\draw [line width=0.3mm] (-1.9,2) ellipse (0.2 and 0.2);
\node (v3) at (-2.42,2) {};
\node (v4) at (-1.98,2) {};
\draw [line width=0.3mm] (v3) edge (v4);
\node (v5) at (-3.1,1.92) {};
\node (v6) at (-3.1,1.4) {};
\node (v7) at (-2.5,1.92) {};
\node (v8) at (-2.5,1.4) {};
\node (v9) at (-1.9,1.92) {};
\node (v10) at (-1.9,1.4) {};
\draw  [line width=0.3mm](v5) edge (v6);
\draw [line width=0.3mm] (v7) edge (v8);
\draw  [line width=0.3mm](v9) edge (v10);
\node (v11) at (-3.37,1.74) {};
\node (v12) at (-2.83,2.27) {};
\draw  [line width=0.3mm](v11) edge (v12);
\node (v13) at (-1.64,1.74) {};
\node (v14) at (-2.16,2.26) {};
\draw [line width=0.3mm] (v13) edge (v14);
\end{tikzpicture}
\begin{tikzpicture}
\draw  [line width=0.3mm](-3.1,2) ellipse (0.2 and 0.2);
\draw  [line width=0.3mm](-2.5,2) ellipse (0.2 and 0.2);
\node (v1) at (-3.01,2) {};
\node (v2) at (-2.59,2) {};
\draw [line width=0.3mm] (v1) edge (v2);
\draw  [line width=0.3mm](-1.9,2) ellipse (0.2 and 0.2);
\node (v3) at (-2.41,2) {};
\node (v4) at (-1.99,2) {};
\draw [line width=0.3mm] (v3) edge (v4);
\node (v5) at (-3.1,1.91) {};
\node (v6) at (-3.1,1.4) {};
\node (v7) at (-2.5,1.91) {};
\node (v8) at (-2.5,1.4) {};
\node (v9) at (-1.9,1.91) {};
\node (v10) at (-1.9,1.4) {};
\draw [line width=0.3mm] (v5) edge (v6);
\draw [line width=0.3mm] (v7) edge (v8);
\draw  [line width=0.3mm](v9) edge (v10);
\node (v11) at (-3.37,1.74) {};
\node (v12) at (-2.83,2.27) {};
\draw [line width=0.3mm] (v11) edge (v12);
\node (v13) at (-2.76,1.73) {};
\node (v14) at (-2.23,2.26) {};
\draw [line width=0.3mm] (v13) edge (v14);
\end{tikzpicture}
\caption{Visual depiction of tensor trains with three TT-cores in site-$n$-mixed-canonical form. Left: Norm in the first core and other cores left-orthogonal. Middle: Norm in the second core, first and last cores are left- and right-orthogonal, respectively. Right: Norm in the last core and other cores right-orthogonal.}
\label{fig:TT_can}
\end{figure}
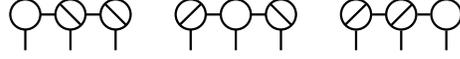
     
A special case of the TT decomposition format is the tensor train matrix (\cref{fig:TTm}), which represents a large matrix in TT format. Tensor train matrices arise in the context of the unscented transform in \cref{sec:UT}.
\begin{definition}[\textbf{Tensor train matrix} \cite{Oseledets2010}] \label{def:TTm} A tensor train matrix (TTm) consists of a set of four-way tensors $\Gt_i\in\mathbb{R}^{R_i\times I_i \times J_i \times R_{i+1}}$, $i=1,\dots,N$ with $R_1=R_{N+1}=1$ that represents a matrix $\mathbf{A}\in\mathbb{R}^{I \times J}$. The row and column indices are split into multiple row indices $I = I_1,\dots,I_N$ and column indices $J = J_1,\dots,J_N$, respectively and the matrix is transformed into a $2N$-way tensor $\Yt_\mathbf{A}\in \mathbb{R}^{I_1\times J_1 \times \dots \times I_N \times J_N}$. Element-wise, the ($i_1,j_1,i_2,j_2,\dots,i_N,j_N$)-th entry of $\Yt_\mathbf{A}$ is computed as
\begin{equation*}
    \sum_{r_1=1}^{R_1} \sum_{r_2=1}^{R_2} \dots \sum_{r_{N+1}=1}^{R_{N+1}} \Gt_1(r_1,i_1,j_1,r_2)\Gt_2(r_2,i_2,j_2,r_3)\cdots\Gt_N(r_N,i_N,j_N,r_{N+1}).
\end{equation*}
\end{definition}
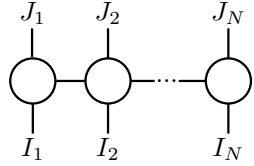
\begin{figure}
    \centering
    \begin{tikzpicture}
\draw [line width=0.3mm] (4.6,2.8) ellipse (0.3 and 0.3);
\draw [line width=0.3mm] (5.6,2.8) ellipse (0.3 and 0.3);
\draw [line width=0.3mm] (7.2,2.8) ellipse (0.3 and 0.3);
\draw [line width=0.3mm](4.91,2.8) -- (5.29,2.8);
\draw [line width=0.3mm](6.54,2.8) -- (6.9,2.8);
\draw [line width=0.3mm](4.6,2.5) -- (4.6,2.1);
\draw [line width=0.3mm](4.59,3.5) -- (4.59,3.12);

\draw [line width=0.3mm](5.6,2.5) -- (5.6,2.1);
\draw [line width=0.3mm](5.6,3.5) -- (5.6,3.1);
\draw [line width=0.3mm](7.2,2.5) -- (7.2,2.1);
\draw [line width=0.3mm](7.2,3.5) -- (7.2,3.1);
\draw [line width=0.3mm](5.91,2.8) -- (6.22,2.8);
\node at (4.6,1.9) {$I_1$};
\node at (4.6,3.7) {$J_1$};
\node at (5.6,1.9) {$I_2$};
\node at (5.6,3.7) {$J_2$};
\node at (6.38,2.8) {...};
\node at (7.2,1.9) {$I_N$};
\node at (7.2,3.7) {$J_N$};
\end{tikzpicture}
    \caption{Visual depiction of a tensor train matrix. The row indices $I_1,\dots, I_N$ point downwards, and the column indices $J_1,\dots, J_N$ point upwards.}
    \label{fig:TTm}
\end{figure}
A TTm arises e.g.\ from an outer product between two vectors $\mathbf{a}$ and $\mathbf{b}$, which corresponds to computing the product of one vector with the transpose of the other. If vector $\mathbf{a}$ is represented by a TT with cores $\mathbfcal{A}_1,\dots,\mathbfcal{A}_N$, the resulting TTm is achieved by summing over a rank-1 connection between one of the TT-cores, e.g.\ the first, and vector $\mathbf{b}$ (\cref{fig:ttmops} top). This result is a special case of the general TTm, where only one of the TT-cores has a double index. This means that only the row index is very large and therefore split into multiple indices, while the column index is not split. If both vectors in the outer product are represented by tensor trains with cores $\mathbfcal{A}_1,\dots,\mathbfcal{A}_N$ and $\mathbfcal{B}_1,\dots,\mathbfcal{B}_N$, respectively, then each core is summed over a rank-1 connection with the core of the other TT's transpose (\cref{fig:ttmops} middle). All cores have then a row and column indices. The product of a matrix $\mathbf{C}$ in TTm format with cores $\mathbfcal{C}_1,\dots,\mathbfcal{C}_N$ with a vector $\mathbf{b}$ is computed by summing over the column index of one TTm-core, e.g.\ the first, and the row index of the vector (\cref{fig:ttmops} bottom).

\begin{figure}
    \centering
    \begin{tikzpicture}
\draw [line width=0.3mm] (-1.3,5.1) rectangle (-1.1,4.3);
\draw  [line width=0.3mm](-0.9,5.1) rectangle (-0.1,4.9);
\draw  [line width=0.3mm](0.6,5.1) rectangle (1.4,4.3);
\node at (0.3,4.7) {=};
\node at (-2.89,4.9) {\small Outer product of };
\node at (-2.79,4.5) {TT and vector};
\draw  [line width=0.3mm](2.4,5.42) ellipse (0.25 and 0.25);
\draw [line width=0.3mm](2.4,5.86) -- (2.4,5.66);
\draw [line width=0.3mm](2.64,4.6) -- (2.94,4.6);
\draw [line width=0.3mm](3.42,4.6) -- (3.62,4.6);
\draw [line width=0.3mm](4.02,4.6) -- (4.22,4.6);
\draw [line width=0.3mm](2.4,4.35) -- (2.4,4.15);
\draw [line width=0.3mm](3.19,4.36) -- (3.19,4.16);
\draw [line width=0.3mm](4.49,4.36) -- (4.49,4.16);
\draw  [line width=0.3mm](2.4,4.6) ellipse (0.25 and 0.25);
\draw  [line width=0.3mm](3.19,4.6) ellipse (0.25 and 0.25);
\draw  [line width=0.3mm](4.49,4.6) ellipse (0.27 and 0.25);
\node  at (3.82,4.6) {$...$};
\draw [densely dotted,line width=0.5mm](2.4,5.15) -- (2.4,4.84);
\node at (-1.2,5.3) {$\mathbf{a}$};
\node at (-0.5,5.35) {$\mathbf{b}^\top$};
\draw  [line width=0.3mm](-1.3,3.3) rectangle (-1.1,2.5);
\draw  [line width=0.3mm](-0.9,3.3) rectangle (-0.1,3.1);
\node at (0.3,2.9) {=};
\draw  [line width=0.3mm](0.6,3.3) rectangle (1.4,2.5);
\node at (-2.79,3.14) {\small Outer product };
\node at (-2.79,2.74) {of two TTs};
\draw [densely dotted,line width=0.5mm](3.17,3.16) -- (3.17,2.82);
\draw [densely dotted,line width=0.5mm](2.4,3.16) -- (2.4,2.82);
\draw [densely dotted,line width=0.5mm](4.5,3.16) -- (4.5,2.82);
\draw [line width=0.3mm](2.4,3.66) -- (2.4,3.86);
\draw [line width=0.3mm](3.17,3.66) -- (3.17,3.86);
\draw [line width=0.3mm](4.5,3.66) -- (4.5,3.86);
\draw  [line width=0.3mm](2.4,3.4) ellipse (0.25 and 0.25);
\draw  [line width=0.3mm](3.17,3.4) ellipse (0.25 and 0.25);
\draw  [line width=0.3mm](4.5,3.4) ellipse (0.27 and 0.25);
\draw [line width=0.3mm](2.64,3.4) -- (2.94,3.4);
\draw [line width=0.3mm](3.42,3.4) -- (3.62,3.4);
\draw [line width=0.3mm](4.22,3.4) -- (3.98,3.4);
\node at (3.82,3.4) {$...$};
\node at (3.82,2.61) {$...$};
\draw  [line width=0.3mm](2.4,2.6) ellipse (0.25 and 0.25);
\draw  [line width=0.3mm](3.17,2.6) ellipse (0.25 and 0.25);
\draw  [line width=0.3mm](4.5,2.6) ellipse (0.27 and 0.25);
\draw [line width=0.3mm](2.4,2.37) -- (2.4,2.17);
\draw [line width=0.3mm](3.17,2.37) -- (3.17,2.17);
\draw [line width=0.3mm](4.5,2.37) -- (4.5,2.17);
\node at (-1.2,3.5) {$\mathbf{a}$};
\node at (-0.2,1.9) {$\mathbf{b}$};
\node at (-2.79,1.41) {\small Product of };
\node at (-2.89,1.01) {TTm and vector};
\draw  [line width=0.3mm](0.6,1.7) rectangle (0.8,0.9);
\draw  [line width=0.3mm](-1.3,1.7) rectangle (-0.5,0.9);
\draw  [line width=0.3mm](-0.3,1.7) rectangle (-0.1,0.9);
\node at (0.3,1.3) {=};
\draw  [line width=0.3mm](2.4,0.93) ellipse (0.25 and 0.25);
\draw  [line width=0.3mm](3.18,0.93) ellipse (0.25 and 0.25);
\draw  [line width=0.3mm](4.5,0.93) ellipse (0.27 and 0.25);
\draw  [line width=0.3mm](2.4,1.65) ellipse (0.25 and 0.25);
\draw [line width = 0.3mm] (2.4,1.41) -- (2.4,1.18);
\draw [line width=0.3mm](2.64,0.93) -- (2.94,0.93);
\draw [line width=0.3mm](3.42,0.93) -- (3.62,0.93);
\draw [line width=0.3mm](2.4,0.69) -- (2.4,0.49);
\draw [line width=0.3mm](3.18,0.69) -- (3.18,0.49);
\draw [line width=0.3mm](4.5,0.69) -- (4.5,0.49);
\draw [line width=0.3mm](3.98,0.93) -- (4.22,0.93);

\draw [line width=0.3mm](2.64,2.6) -- (2.94,2.6);
\draw [line width=0.3mm](3.42,2.6) -- (3.62,2.6);
\draw [line width=0.3mm](3.98,2.6) -- (4.22,2.6);
\node at (3.82,0.93) {$...$};
\node at (-0.9,1.9) {$\mathbf{C}$};
\node at (-0.5,3.54) {$\mathbf{b}^\top$};

\node at (2.4,5.42) {\small$\mathbf{b}^\top$};
\node at (2.4,4.6) {\small $\mathbfcal{A}_1$};
\node at (3.19,4.6) {\small $\mathbfcal{A}_2$};
\node at (4.51,4.6) {\small $\mathbfcal{A}_N$};
\node at (2.4,3.4) {\small $\mathbfcal{B}_1$};
\node at (3.17,3.4) {\small $\mathbfcal{B}_2$};
\node at (4.51,3.4) {\small $\mathbfcal{B}_N$};
\node at (2.4,2.6) {\small $\mathbfcal{A}_1$};
\node at (3.17,2.6) {\small $\mathbfcal{A}_2$};
\node at (4.52,2.6) {\small $\mathbfcal{A}_N$};
\node at (2.4,1.65) {\small $\mathbf{b}$};
\node at (2.4,0.93) {\small $\mathbfcal{C}_1$};
\node at (3.18,0.93) {\small $\mathbfcal{C}_2$};
\node at (4.5,0.93) {\small $\mathbfcal{C}_N$};
\end{tikzpicture}
    \caption{Operations with matrices and vectors. Top: Outer product between a vector $\mathbf{a}$, represented by a TT with cores $\mathbfcal{A}_1,\dots,\mathbfcal{A}_N$, and a vector $\mathbf{b}$. A rank-1 connection (dotted line) is summed over between the first TT-core and vector $\mathbf{b}^\top$. Middle: Outer product between two vectors $\mathbf{a}$ and $\mathbf{b}$ represented by tensor trains with cores $\mathbfcal{A}_1,\dots,\mathbfcal{A}_N$ and $\mathbfcal{B}_1,\dots,\mathbfcal{B}_N$, respectively. A rank-1 connection is summed over between each core of the TTs. Bottom: The product between a matrix $\mathbf{C}$ in TTm format with cores $\mathbfcal{C}_1,\dots,\mathbfcal{C}_N$ and a vector $\mathbf{b}$. The column index of the first TTm-core is summed over with the row index of the vector $\mathbf{b}$.}
    \label{fig:ttmops}
\end{figure}
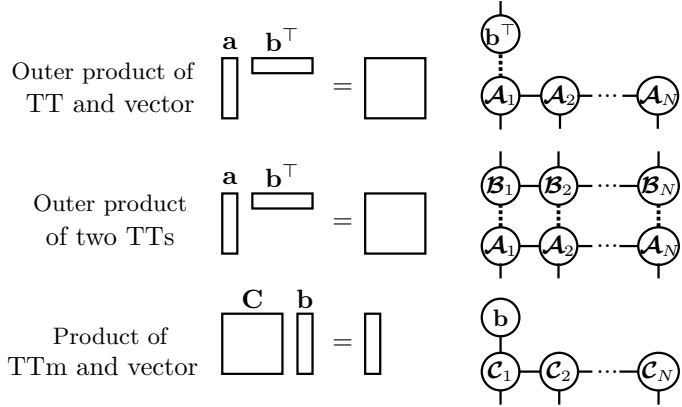
\section{Bayesian inference for low-rank tensor approximation}
\label{sec:BayALS}
In this section, we present a method to find a low-rank tensor decomposition using a similar strategy as in the ALS by solving a Bayesian inference problem. In this context, the vectorization of each TD component is treated as a Gaussian random variable, expressed in terms of a mean and a covariance. Generally, we denote a Gaussian probability distribution as $\mathcal{N}(\m,\cov)$, where $\m$ is the mean and $\cov$ is the covariance. This section is organized as follows. First, we define the prior for the inference problem. Before computing the joint posterior distribution, we look at a simpler inference problem, stated in \cref{lem:post}, where the posterior distribution of only one TD component is computed. Then, \cref{thm:jointpost} describes the computation of the joint posterior by applying a block coordinate descent method and simplifying the inference problem to iteratively applying \cref{lem:post}. Finally, our resulting \cref{alg:noqr} is applied in an example.\\

To initialize the Bayesian inference problem, a multi-variate Gaussian prior is assigned to every TD component 
\begin{equation*}
p\left(\g_i\right)= \mathcal{N}\left(\m_i^0,\cov_i^0\right),\qquad i=1,\dots, N,
\end{equation*} 
where $\m_i^0$ and $\cov_i^0$ are the prior mean and covariance matrix, respectively.
The TD components $\g_i\in\mathbb{R}^{R_iI_iR_{i+1}\times1}$ are assumed to be statistically independent. Therefore, the joint prior distribution is given by
\begin{equation*}
p(\{\g_i\})=
\mathcal{N}\left(
\begin{bmatrix}
\m_1^0\\
\m_2^0\\
\vdots\\
\m_N^0
\end{bmatrix},
\begin{bmatrix}
\cov_1^0 & 0 & \dots & 0 \\
0 & \cov_2^0 & \ddots & \vdots \\
\vdots & \ddots & \ddots & 0 \\
0 & \dots & 0 & \cov_N^0
\end{bmatrix}
\right),
\end{equation*}
where $\{\g_i\}$ denotes the priors of all TD components. Because of the statistical independence, the joint prior distribution and the prior on one TD component conditioned on the other TD components, can be written as
\begin{align}
p(\{\g_i\})&=p(\g_1)p(\g_2)\dots p(\g_N) \;\; \mathrm{and} \label{eq:allpriorsind}\\
p(\g_n\mid\{\g_i\}_{i\neq n})&=p(\g_n), \label{eq:onepriorind}
\end{align}
respectively, where $\{\g_i\}_{i\neq n}$ denotes the collection of all TD components except \newline the~$n$th. \\

The joint posterior distribution $p(\{\g_i\}\mid \y)$ is found by applying Bayes' rule. However, before solving this inference problem and inspired by a result described in \cite[p.\ 29]{Sarkka2010}, we first look at the simpler problem to find the posterior distribution of one component, given the measurement and the other components.

\begin{lemma}
Let the prior distribution $p(\g_n)=\mathcal{N}(\m_n^0,\cov_n^0)$ and the likelihood $p(\y\mid\{\g_i\})=\mathcal{N}(\m_{\y},\sigma^2\id)$ be Gaussian, where $\m_{\y}=\U_{\setminus n}\g_n$. Further, let all TD components be statistically independent and let the TD be multilinear. Then, the posterior distribution $p\left(\g_n\mid\{\g_i\}_{i\neq n},\y\right)=\mathcal{N}(\m_n^+,\cov_n^+)$ of the $n$th component given the measurements and the other components is also Gaussian with mean $\m_n^+$ and covariance $\cov_n^+$ 
\begin{align}
\m_n^+&=\left[(\cov_n^0)^{-1}+\frac{\U_{\setminus n}^\top \U_{\setminus n}}{\sigma^2}\right]^{-1}\left[\frac{\U_{\setminus n}^\top \y}{\sigma^2}+(\cov_n^0)^{-1} \m_n^0\right]\label{eq:post_m}\\
\cov_n^+&=\left[(\cov_n^0)^{-1}+\frac{\U_{\setminus n}^\top \U_{\setminus n}}{\sigma^2}\right]^{-1}. \label{eq:post_P}
\end{align}
\label{lem:post}
\end{lemma}
\begin{proof}
The posterior distribution of one TD component conditioned on the other TD components and the measurements $p\left(\g_n\mid\y,\{\g_i\}_{i\neq n}\right)$ can be found by applying Bayes' rule. Assuming that all components are statistically independent \cref{eq:onepriorind} leads to 
\begin{align}
p\left(\g_n\mid\y,\{\g_i\}_{i\neq n}\right)
= \frac{p(\y\mid\{\g_i\})p(\g_n)}{p(\y\mid \{\g_i\}_{i\neq n})}.
\label{eq:posterior}
\end{align}
Since the likelihood $p(\y\mid\{\g_i\})$ and prior $p(\g_n)$ are Gaussian, also the posterior will be Gaussian \cite[p.\ 28-29,\; 209-210]{Sarkka2010} with mean \cref{eq:post_m} and covariance \cref{eq:post_P}.
\end{proof}
\begin{corollary}
\label{cor:normaleq}
For $\lim \cov_n^{0} \rightarrow \infty $, \cref{eq:post_m} reduces to the normal equations of the least squares problem and therefore the update equation of the conventional ALS
\begin{equation}
\m_n^+=\left( \U_{\setminus n}^\top \U_{\setminus n} \right)^{-1}\U_{\setminus n}^\top \y.
\label{eq:normaleq}
\end{equation}
\end{corollary}
\cref{cor:normaleq} describes the case where there is no useful prior information available for the $n$th TD component. Thus, the certainty on the prior mean is zero, and $\lim \cov_n^{0} \rightarrow \infty$.\\

Now, we can use \cref{lem:post} to find the joint posterior distribution of all TD components as described in the following theorem.

\begin{theorem}
\label{thm:jointpost}
Let $p\left(\{\mathbf{g}_i\}\mid\mathbf{y}\right)$ be the posterior joint distribution of all TD components given $\y$. Further, let the prior distribution $p(\g_n)=\mathcal{N}(\m_n^0,\cov_n^0)$ of any component as well as the likelihood $p(\y\mid\{\g_i\})=\mathcal{N}(\m_{\y},\sigma^2\id)$ be Gaussian, where the mean $\m_{\y}$ is a nonlinear function of all the TD components. Further, let all TD components be statistically independent and let the TD be multilinear. Then, by applying block coordinate descent to find the posterior density, every step of the block coordinate descent corresponds to applying \cref{lem:post}.
\end{theorem}

\begin{proof}
Bayes' rule and statistical independence \cref{eq:allpriorsind} gives
\begin{align}
p\left(\{\g_i\}\mid\y\right) =  \frac{p(\y\mid\{\g_i\})p(\g_1)p(\g_2)\dots p(\g_N)}{p(\y)}.
\label{eq:posterior_gen}
\end{align}
As in the conventional ALS, a block coordinate descent method is applied by conditioning the posterior distribution of one TD component on all the others. In this way, the TD components can be computed sequentially with \cref{eq:posterior}. In addition, due to the multilinearity of the TD the mean of the likelihood becomes a linear function of the~$n$th TD component, $\m_\y = \U_{\setminus n}\g_n$. Thus, every TD update corresponds to applying \cref{lem:post}.
\end{proof}

With \cref{cor:normaleq}, \cref{thm:jointpost} gives a Bayesian interpretation of the ALS by deriving its update equation from the TD components defined as probability distributions.
The following example shows how the distributions change with every update.
\begin{example}[\textbf{Distribution updates for a TD with three components}]
Assume we would like to apply \cref{thm:jointpost} to find a TD with three components. First, the three TD components are initialized with a prior distribution. Then, the distributions are updated sequentially by computing the posterior with Bayes' rule, as shown in \cref{fig:ex1}. After updating the three TD components, the updates are repeated until a stopping criterion is met.  
\begin{figure}
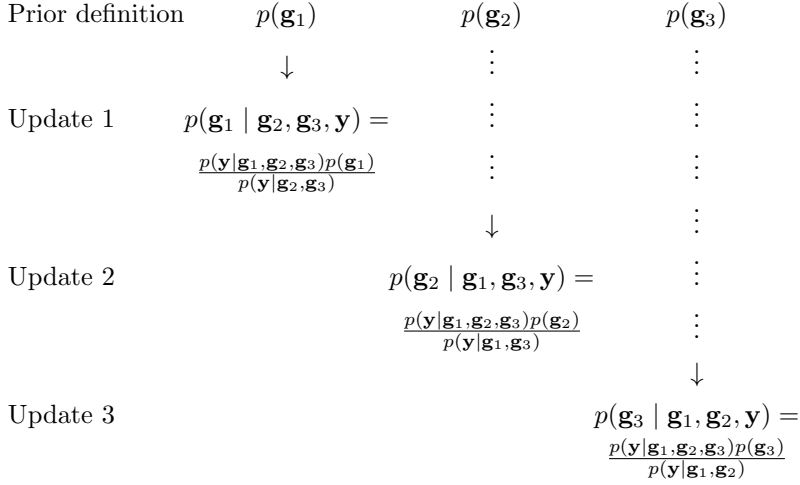

    \centering
    \renewcommand\arraystretch{1.33}
    \[
    \begin{array}{l @{{}{}} c @{{}{}} c @{{}{}} c}
    \text{Prior definition} & p(\g_1) & p(\g_2) & p(\g_3) \\
    & \downarrow & \vdots & \vdots  \\
    \text{Update 1} & p(\g_1\mid\g_2,\g_3,\y)= & \vdots & \vdots\\
    & \frac{p(\y\mid\g_1,\g_2,\g_3)p(\g_1)}{p(\y\mid\g_2,\g_3)} & \vdots & \vdots\\
     & & \downarrow & \vdots \\
    \text{Update 2} & & p(\g_2\mid\g_1,\g_3,\y)= & \vdots\\
     & & \frac{p(\y\mid\g_1,\g_2,\g_3)p(\g_2)}{p(\y\mid\g_1,\g_3)} & \vdots \\
     & & & \downarrow \\
    \text{Update 3} & & & p(\g_3\mid\g_1,\g_2,\y)=\\
     & & & \frac{p(\y\mid\g_1,\g_2,\g_3)p(\g_3)}{p(\y\mid\g_1,\g_2)}
    \end{array}
    \]
    \caption{Distribution updates for example with three TD components.}
    \label{fig:ex1}
\end{figure}

\end{example}

\Cref{alg:noqr} summarizes the steps of the ALS in a Bayesian framework. The mean and covariance of each TD component are sequentially updated, followed by the computation of $\U_{\setminus n}$ which is computed from $\{\g_i\}_{i\neq n}$. The stopping criterion is defined by the user, e.g.\ as a maximum number of iterations or the convergence of the residuals between the measurement and estimate, as used in the convectional ALS. It is also possible to consider the convergence of the TT-core's covariance matrices as a stopping criterion since these are additionally computed in the ALS in the Bayesian framework. Another possibility is to look at the convergence of the numerator of Bayes' rule. The computational cost and storage requirements are given in \cref{tab:complexity_1}. For the complexity analysis, we use the following notation. The largest rank or the CP-rank is denoted by $R$ and the largest dimension of the approximated tensor is denoted by $I$. In comparison, the conventional ALS has the same computational cost for every TD component update, and a total storage requirement of $\mathcal{O}(NRI)$ for CP, $\mathcal{O}(NRI+R^N)$ for Tucker, and $\mathcal{O}(NR^2I)$ for TTs. The ALS in a Bayesian framework has an additional term in the storage requirements, because it computes the covariance matrix for every TD component.\\
\begin{table}[]
    \centering
    \caption{Computational cost per update and overall storage requirements for \cref{alg:noqr}.}
    \begin{tabular}{c|c|c}
        TD & computational cost & storage \\ \hline 
        CP & $\mathcal{O}(R^3I^3)$ & $\mathcal{O}(NRI+NR^2I^2)$ \\
        Tucker & $\mathcal{O}(R^3I^3)$ & $\mathcal{O}(NRI+R^N+NR^2I^2+R^{2N})$ \\
        TT & $\mathcal{O}(R^6I^3)$ & $\mathcal{O}(NR^2I+NR^4I^2)$
    \end{tabular}
    \label{tab:complexity_1}
\end{table}

Our method also opens up the possibility to recursively estimating the mean and covariance of the TD components. In case a new noisy measurement $\y$ of the same underlying tensor becomes available, \cref{alg:noqr} can applied repeatedly with the output mean and covariance from the previous execution as the prior for the new execution.
\begin{algorithm}
\caption{ALS in a Bayesian framework}
\label{alg:noqr}
\begin{algorithmic}[1]
\REQUIRE{Prior mean $\{\m_i^0\}$ and covariance $\{\cov_i^0\}$, $i=1,\dots N$, measurement $\y$ and noise variance $\sigma^2$.}
\ENSURE{Posterior mean $\{\m_i^+\}$ and covariance $\{\cov_i^+\}$, $i=1,\dots N$.}
\STATE{Compute $\U_{\setminus 1}$ with \cref{eq:U_cp} for CP, \cref{eq:U_tucker} for Tucker or \cref{eq:multilin} for TT with $\{\m_i^0\}_{i\neq1}$.}
\WHILE{Stopping criterion is not true}
\setstretch{1.4}
\FOR{$n=1,...,N$}
\STATE{$\cov_n^+ \leftarrow \left[(\cov_n^0)^{-1}+\frac{\U_{\setminus n}^\top \U_{\setminus n}}{\sigma^{2}}\right]^{-1}$}
\STATE{$\m_n^+ \leftarrow \cov_n^+\left[\frac{\U_{\setminus n}^\top \y} {\sigma^2}+(\cov_n^{0})^{-1} \m_n^0\right]$}

\STATE{Recompute $\U_{\setminus n+1}$ with~$\m_n^+$, where $N+1=1$.}
\ENDFOR
\ENDWHILE
\end{algorithmic}
\end{algorithm}

\section{Orthogonalization step in Bayesian framework for a TT}
Every iteration of \cref{alg:noqr} requires the inversion 
\begin{equation}
    \left[(\cov_n^0)^{-1}+\frac{\U_{\setminus n}^\top \U_{\setminus n}} {\sigma^2}\right]^{-1}, \;\; \text{corresponding to }\;\; \left[\U_{\setminus n}^\top \U_{\setminus n}\right]^{-1}
\end{equation}
in the conventional ALS update. To avoid the propagation of numerical errors and ensure numerical stability, some ALS algorithms, e.g., the one for the TT decomposition, include an orthogonalization step after every update. In this way, the condition number of each sub-problem can not become worse than the one of the overall problem \cite[p.\ A701]{Rohwedder2012}. In this section, we present how we integrate the orthogonalization procedure into the ALS in a Bayesian framework for a TT decomposition in site-$n$-mixed-canonical form. The same can also be applied to a Tucker decomposition with orthogonal factor matrices. \\

We first describe, how the orthogonalization step is performed in the conventional ALS and then how we integrate it into the ALS in a Bayesian framework. Here, we differentiate between the prior distributions of each TT-core and the initial guess for each TT-core, which initializes the conventional ALS. In the conventional ALS with orthogonalization step, the initial TT is transformed into the site-$1$-mixed-canonical form, where the Frobenius norm of the first TT-core corresponds to the Frobenius norm of the entire tensor train. The update is always performed on the core that contains the Frobenius norm. The procedure, therefore, requires transformations that separate the Frobenius norm from the updated TT-core and moves it to the next TT-core to be updated. \\

The initial TT is transformed into site-$1$-mixed-canonical form, by orthogonalizing the $N$th up to the $2$nd TT-core as illustrated in \cref{fig:intositek} for a TT with three cores.
\begin{figure}
    \centering
    \begin{tikzpicture}
\definecolor{gray}{gray}{0.9}
\draw  (-4.2,3.2) ellipse (0.2 and 0.2);
\draw  (-3.6,3.2) ellipse (0.2 and 0.2);
\draw  (-2.6,3.2) ellipse (0.2 and 0.2);
\draw  (-0.4,3.2) ellipse (0.2 and 0.2);
\draw  [fill=gray](-1,3.2) ellipse (0.2 and 0.2);
\draw  (-1.6,3.2) ellipse (0.2 and 0.2);
\draw  (2.2,3.2) ellipse (0.2 and 0.2);
\draw  (1.6,3.2) ellipse (0.2 and 0.2);
\draw  (0.6,3.2) ellipse (0.2 and 0.2);
\draw  (4.4,3.2) ellipse (0.2 and 0.2);
\draw  (3.8,3.2) ellipse (0.2 and 0.2);
\draw [fill=gray]  (3.2,3.2) ellipse (0.2 and 0.2);
\draw  [fill=gray](-5.2,3.2) ellipse (0.2 and 0.2);
\draw  (-5.8,3.2) ellipse (0.2 and 0.2);
\draw  (-6.4,3.2) ellipse (0.2 and 0.2);
\draw (-4,3.2) -- (-3.8,3.2);
\draw (-3.4,3.2) -- (-2.8,3.2);
\draw (-4.2,3) -- (-4.2,2.8);
\draw (-3.6,3) -- (-3.6,2.8);
\draw (-2.6,3) -- (-2.6,2.8);
\draw (-1.6,3) -- (-1.6,2.8);
\draw (-1,3) -- (-1,2.8);
\draw (-0.4,3) -- (-0.4,2.8);
\draw (0.6,3) -- (0.6,2.8);
\draw (1.6,3) -- (1.6,2.8);
\draw (2.2,3) -- (2.2,2.8);
\draw (3.2,3) -- (3.2,2.8);
\draw (3.8,3) -- (3.8,2.8);
\draw (4.4,3) -- (4.4,2.8);
\draw (-6.4,3) -- (-6.4,2.8);
\draw (-5.8,3) -- (-5.8,2.8);
\draw (-5.2,3) -- (-5.2,2.8);
\draw (-1.4,3.2) -- (-1.2,3.2);
\draw (-0.8,3.2) -- (-0.6,3.2);
\draw (0.8,3.2) -- (1.4,3.2);
\draw (1.8,3.2) -- (2,3.2);
\draw (3.4,3.2) -- (3.6,3.2);
\draw (4,3.2) -- (4.2,3.2);
\draw (-6.2,3.2) -- (-6,3.2);
\draw (-5.6,3.2) -- (-5.4,3.2);
\draw  [fill=black](-3.1,3.2) ellipse (0.1 and 0.1);
\draw (-2.75,3.35) -- (-2.45,3.05);
\draw (-0.53,3.35) -- (-0.23,3.05);
\draw (1.46,3.35) -- (1.76,3.05);
\draw (2.07,3.35) -- (2.37,3.05);
\draw (3.67,3.35) -- (3.97,3.05);
\draw (4.27,3.35) -- (4.57,3.05);
\draw [fill=black] (1.1,3.2) ellipse (0.1 and 0.1);
\draw [->](-4.9,3.2) -- (-4.5,3.2);
\draw [->](-2.3,3.2) -- (-1.9,3.2);
\draw [->](-0.1,3.2) -- (0.3,3.2);
\draw [->](2.5,3.2) -- (2.9,3.2);
\end{tikzpicture}
    \caption{Visual depiction of a TT transformation into site-1-mixed-canonical form.}
    \label{fig:intositek}
\end{figure}
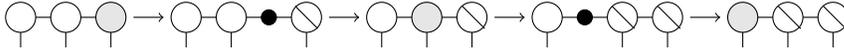
To move the Frobenius norm from the $n$th TT-core to the $(n-1)$th, the $n$th TT-core is orthogonalized by applying the thin $\mathbf{QR}$-decomposition on
\begin{equation}
    \left(\G_n^\mathrm{R}\right)^\top=\Q^\mathrm{R}_n\R^\mathrm{R}_n.
    \label{eq:qrstep}
\end{equation} 
Then, $\G_n^\mathrm{R}$ is replaced by 
\begin{equation}
   \G_n^\mathrm{R}\leftarrow \left(\Q^\mathrm{R}_n\right)^\top
   \label{eq:orthocore}
\end{equation} 
and the non-orthogonal part, illustrated by the small circle in \cref{fig:intositek}, is absorbed into the $(n-1)$th core with
\begin{equation}
    \mathbfcal{G}_{n-1} \leftarrow \mathbfcal{G}_{n-1} \times_3 \R^\mathrm{R}_n.
    \label{eq:absorb}
\end{equation} 
\Cref{eq:qrstep,eq:orthocore,eq:absorb} are applied to the $N$th until the $2$nd TT-core, leading to the TT in site-1-mixed-canonical form. Then, the first core is updated, followed by a transformation to move the Frobenius norm to the second core, and so on. Since the Frobenius norm moves to the right, the orthogonalization step consists of applying the thin $\mathbf{QR}$-decomposition on the left-unfolding
\begin{equation}
    \G_n^\mathrm{L}=\Q^\mathrm{L}_n\R^\mathrm{L}_n.
    \label{eq:qrstep_L}
\end{equation}
The $n$th and $(n+1)$th core are replaced by
\begin{equation}
    \G_n^\mathrm{L}\leftarrow \Q^\mathrm{L}_n \;\;\; \text{and} \;\;\; \mathbfcal{G}_{n+1} \leftarrow \mathbfcal{G}_{n+1} \times_1 \R^\mathrm{L}_n,
    \label{eq:ortho_absorb}
\end{equation}
respectively. After the $N$th core is updated, the updating scheme goes backwards, using again \cref{eq:qrstep,eq:orthocore,eq:absorb} for the orthogonalization step. When the Frobenius norm is absorbed back into the first core, one back and forth sweep of the ALS algorithm is concluded.\\

In the following, we describe how the transformation steps affect the distributions representing the TT-cores in the ALS in a Bayesian framework. The transformation of the random variables can be derived from \cref{eq:qrstep,eq:orthocore,eq:absorb,eq:qrstep_L,eq:ortho_absorb}. When the Frobenius norm is moved to the left, the mean of the $n$th core becomes 
\begin{equation*}
    \m_n \leftarrow \operatorname{vec}\left(\left(\Q^\mathrm{R}_n\right)^\top\right),
\end{equation*}
where $\left(\Q^\mathrm{R}_n\right)^\top$ is computed from \cref{eq:orthocore}. To obtain the transformed covariance of the $n$th TT-core, \cref{eq:qrstep} is rewritten as
\begin{equation*}
    \left(\Q^\mathrm{R}_n\right)^\top = \left(\R^\mathrm{R}_n\right)^{-\top}\G_n^\mathrm{R}.
\end{equation*}
Now, the right-hand side, is vectorized by summing over a rank-1 connection between $\left(\R^\mathrm{R}_n\right)^{-\top}$ and an identity matrix of size $I_nR_{n+1}\times I_nR_{n+1}$ that has a connected edge with $\G_n^\mathrm{R}$ as depicted in \cref{fig:qr_rl}. This leads to a transformation term
\begin{equation}
    \id \otimes \left(\R^\mathrm{R}_n\right)^{-\top}  \label{eq:trafo1}
\end{equation}
that orthogonalizes $\g_n$. The diagram in \cref{fig:covorth} shows how this transformation is applied to the covariance matrix. The transformation term and its transpose are multiplied on the left and right side of $\cov_n$, respectively, resulting in 
\begin{align*}
    \cov_n \leftarrow & \;\;\left(\id \otimes \left(\R^\mathrm{R}_n\right)^{-\top}\right) \cov_n \; \left(\id \otimes \left(\R^\mathrm{R}_n\right)^{-\top}\right)^\top\\
    =& \;\; \left(\id \otimes \left(\R^\mathrm{R}_n\right)^{-\top}\right) \cov_n \; \left(\id \otimes \left(\R^\mathrm{R}_n\right)^{-1}\right).
\end{align*}

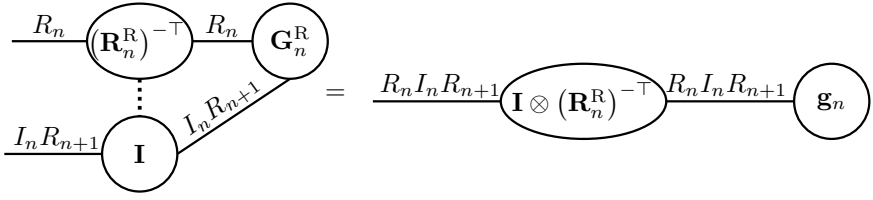
\begin{figure}
\centering
\begin{tikzpicture}
\draw  [line width=0.3mm](-2.5,4.6) ellipse (0.7 and 0.5);
\draw  [line width=0.3mm](-0.5,4.6) ellipse (0.5 and 0.5);
\draw  [line width=0.3mm](-2.5,3.1) ellipse (0.5 and 0.5);
\draw [line width=0.3mm](-0.5,4.1) -- (-2,3.1);
\draw [line width=0.3mm](-1,4.6) -- (-1.8,4.6);
\draw [line width=0.3mm](-3.2,4.6) -- (-4.2,4.6);
\draw [dotted,line width=0.5mm] (-2.5,4.1) -- (-2.5,3.6);
\node at (0.1,3.9) {=};
\draw  [line width=0.3mm](3.4,3.8) ellipse (1.1 and 0.5);
\draw  [line width=0.3mm](6.7,3.8) ellipse (0.5 and 0.5);
\draw [line width=0.3mm](4.5,3.8) -- (6.2,3.8);
\draw [line width=0.3mm](2.3,3.8) -- (0.6,3.8);
\node at (-1.4,4.8) {$R_n$};
\node at (-3.7,4.8) {$R_n$};
\node at (-2.54,4.6) {$\left(\mathbf{R}^\mathrm{R}_n\right)^{-\top}$};
\node at (-2.5,3.1) {$\mathbf{I}$};
\node at (-0.5,4.6) {$\mathbf{G}^\mathrm{R}_n$};
\node [rotate=36] at (-1.41,3.74) {$I_nR_{n+1}$};
\node at (3.39,3.8) {$\mathbf{I}\otimes\left(\mathbf{R}^\mathrm{R}_n\right)^{-\top}$};
\node at (6.7,3.8) {$\mathbf{g}_n$};
\node at (1.5,4) {$R_nI_nR_{n+1}$};
\node at (5.3,4) {$R_nI_nR_{n+1}$};
\draw [line width=0.3mm](-3,3.1) -- (-4.3,3.1);
\node at (-3.6,3.3) {$I_nR_{n+1}$};
\end{tikzpicture}
\caption{Visual depiction of how the non-orthogonal part is separated from the TD component's mean.}
\label{fig:qr_rl}
\end{figure}

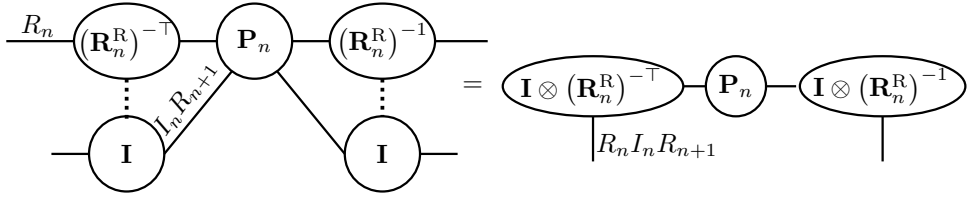
\begin{figure}
\centering
\begin{tikzpicture}
\draw  [line width=0.3mm](-1.4,3) ellipse (0.5 and 0.5);
\draw [line width=0.3mm] (0.3,3) ellipse (0.7 and 0.5);
\draw  [line width=0.3mm](-3.1,3) ellipse (0.7 and 0.5);
\draw [line width=0.3mm](-2.4,3) -- (-1.9,3);
\draw[line width=0.3mm] (-0.9,3) -- (-0.4,3);
\draw [line width=0.3mm](-3.8,3) -- (-4.7,3);
\draw [line width=0.3mm](1,3) -- (1.7,3);
\draw [line width=0.3mm](-1.7,2.6) -- (-2.6,1.5);
\draw [line width=0.3mm](-1.1,2.6) -- (-0.2,1.5);
\node at (-3.1,3) {$\left(\mathbf{R}_n^\mathrm{R}\right)^{-\top}$};
\node at (-1.4,3) {$\mathbf{P}_n$};
\node at (0.3,3) {$\left(\mathbf{R}_n^\mathrm{R}\right)^{-1}$};
\draw [line width=0.3mm] (-3.1,1.5) ellipse (0.5 and 0.5);
\draw [line width=0.3mm] (0.3,1.5) ellipse (0.5 and 0.5);
\draw [line width=0.3mm](-3.6,1.5) -- (-4.1,1.5);
\draw[line width=0.3mm] (0.8,1.5) -- (1.3,1.5);
\draw [dotted,line width=0.5mm] (-3.1,2.5) -- (-3.1,1.95);
\draw [dotted,line width=0.5mm] (0.3,2.5) -- (0.3,2);
\node at (-3.1,1.5) {$\mathbf{I}$};
\node at (0.3,1.5) {$\mathbf{I}$};
\node at (-4.28,3.2) {$R_n$};
\node [rotate=55] at (-2.31,2.24) {$I_nR_{n+1}$};
\node at (1.5,2.4) {=};
\node at (5,2.4) {$\mathbf{P}_n$};
\draw   [line width=0.3mm](5,2.4) ellipse (0.4 and 0.4);
\draw   [line width=0.3mm](3.1,2.4) ellipse (1.2 and 0.4);
\draw   [line width=0.3mm](6.95,2.4) ellipse (1.1 and 0.4);
\node at (3.08,2.39) {$\mathbf{I} \otimes \left(\mathbf{R}_n^\mathrm{R}\right)^{-\top}$};
\draw [line width=0.3mm] (4.3,2.4) -- (4.6,2.4);
\draw [line width=0.3mm] (5.4,2.4) -- (5.8,2.4);
\node at (6.93,2.39) {$\mathbf{I} \otimes \left(\mathbf{R}_n^\mathrm{R}\right)^{-1} $};
\draw  [line width=0.3mm](3.1,2) -- (3.1,1.4);
\draw  [line width=0.3mm](6.95,2) -- (6.95,1.4);
\node at (3.95,1.6) {$R_nI_nR_{n+1}$};
\end{tikzpicture}
\caption{Visual depiction of how the covariance matrix is transformed in the orthogonalization step.}
\label{fig:covorth}
\end{figure}

The transformations of the $(n-1)$th core to absorb the Frobenius norm, can be derived from \cref{eq:absorb} in a similar way as explained above, resulting in a transformation term 
\begin{equation}
\mathbf{R}_n^\mathrm{R} \otimes \id.    \label{eq:trafo2}
\end{equation}
When the Frobenius norm is moved to the right during the orthogonalization step, the transformations for the updated core and the next core to be updated become
\begin{equation}
    \left( \mathbf{R}_n^\mathrm{L} \right)^{-\top} \otimes \id \quad \mathrm{and} \quad  \id \otimes \mathbf{R}_n^\mathrm{L}, \label{eq:trafo3}
\end{equation}
respectively.
It can be easily shown that the transformations for the orthogonalization step, do not affect the statistical independence of the joint distribution of the random variables, since the transformations are performed on each variable individually. The following example shows the updating for the transformation scheme of the random variables that represent an exemplary three core TT.

\begin{example}[\textbf{Distribution updates and orthogonalization transformations for a TT with three cores}]
Assume we would like to apply \cref{thm:jointpost} to find a TT with three cores and keep the TD in site-$n$-mixed- canonical form. The three TT-cores are initialized with a prior distribution and transformed such that the corresponding TT is in site-1-mixed-canonical form. The random variables that represent the orthogonal cores are denoted by $\q_i,\;i=1,2,3$ and the random variable representing the TT-core that contains the Frobenius norm is denoted by $\x_i,\;i=1,2,3$. After the random variables are transformed into site-$1$-mixed-canonical form using \eqref{eq:trafo1} and \eqref{eq:trafo2}, the first core is updated followed by moving the Frobenius norm to the second core. Then the second core is updated and the Frobenius norm is moved to the last. When this half-sweep, as shown below, is completed using the transformations from \eqref{eq:trafo3}, the same procedure is repeated in the opposite direction, requiring again \eqref{eq:trafo1} and \eqref{eq:trafo2}. The example is depicted in \cref{fig:ex2}.
\begin{figure}
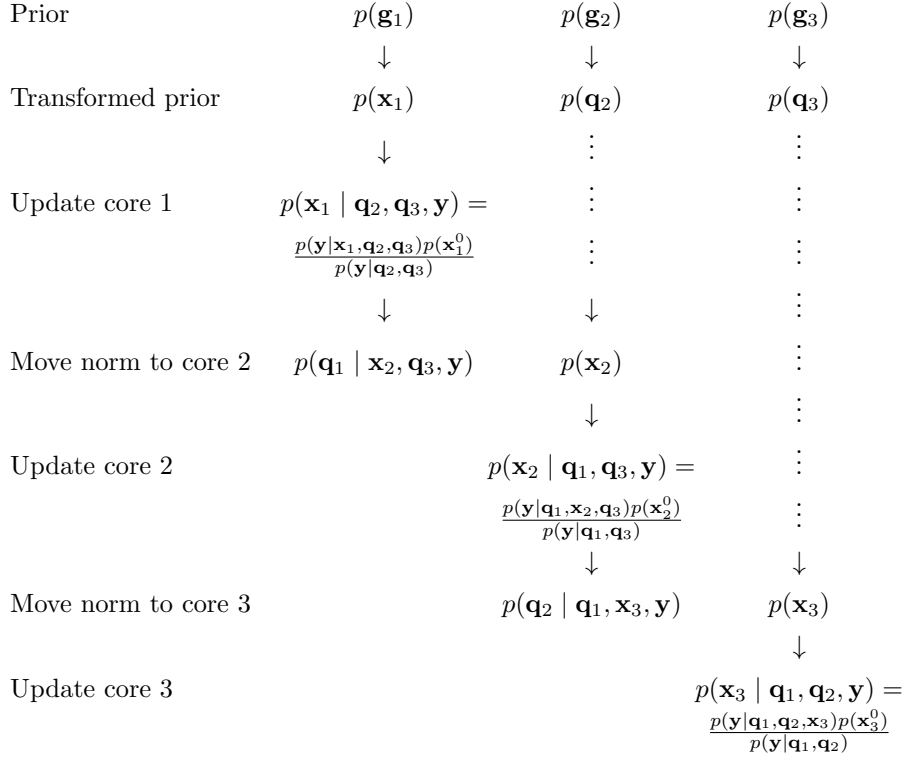

    \centering
    \renewcommand\arraystretch{1.33}
    \settowidth\mylen{300}
    \[
    \begin{array}{l @{{}{}} c @{{}{}} c @{{}{}} c}
    \text{Prior} & p(\g_1) & p(\g_2) & p(\g_3) \\
    & \downarrow & \downarrow  & \downarrow   \\
    \text{Transformed prior} & p(\x_1) & p(\q_2) & p(\q_3) \\
    & \downarrow & \vdots & \vdots  \\
    \text{Update core 1} & p(\x_1\mid\q_2,\q_3,\y)= & \vdots & \vdots\\
    & \frac{p(\y\mid\x_1,\q_2,\q_3)p(\x_1^0)}{p(\y\mid\q_2,\q_3)} & \vdots & \vdots\\
    & \downarrow  & \downarrow & \vdots \\
    \text{Move norm to core 2\;\;\;\;} & p(\q_1\mid\x_2,\q_3,\y) &  p(\x_2) & \vdots \\
     & & \downarrow & \vdots \\
    \text{Update core 2} & & p(\x_2\mid\q_1,\q_3,\y)= & \vdots\\
     & & \frac{p(\y\mid\q_1,\x_2,\q_3)p(\x_2^0)}{p(\y\mid\q_1,\q_3)} & \vdots \\
     & & \downarrow & \downarrow \\
    \text{Move norm to core 3} & & p(\q_2\mid\q_1,\x_3,\y) & p(\x_3)  \\
     & & & \downarrow \\
    \text{Update core 3} & & & p(\x_3\mid\q_1,\q_2,\y)=\\
     & & & \frac{p(\y\mid\q_1,\q_2,\x_3)p(\x_3^0)}{p(\y\mid\q_1,\q_2)}
    \end{array}
    \]
    \caption{Distribution updates with orthogonalization step for example with three TT-cores.}
    \label{fig:ex2}
\end{figure}
    
\end{example}

The ALS in a Bayesian framework with orthogonalization step has another difference compared to the one without orthogonalization. The update equations for the mean and covariance, \cref{eq:post_m} and \cref{eq:post_P}, are affected by the TT decomposition being in site-$n$-mixed-canonical form: the matrix $\U_{\setminus n}$ becomes orthogonal and the update equations simplify to
\begin{equation*}
\m^+_n=\underbrace{\left[(\cov_n^{0})^{-1}+\frac{\id}{\sigma^{2}}\right]^{-1}}_{\cov^+_n}\left[\frac{\U_{\setminus n}^\top \y}{\sigma^{2}}+(\cov_n^{0})^{-1} \m_n^{0}\right]. \label{eq:mPortho}
\end{equation*}
In this case $\U_{\setminus n}^\top \y$ corresponds to the update of the conventional ALS \eqref{eq:normaleq}, due to the orthogonality of $\U_{\setminus n}$.\\

\Cref{alg:ortho} summarizes the ALS in a Bayesian framework with orthogonalization step for a TT decomposition. The computational cost of one update in \cref{alg:ortho} is $\mathcal{O}(R^6I^3)$ for the inversion and $\mathcal{O}(R^3I^2)$ for the thin $\mathbf{QR}$-factorization and the storage requirement is $\mathcal{O}(R^2I+R^4I^2)$, where $R$ is the largest TT-rank and $I$ is the largest dimension of the approximated tensor. The only difference compared to the conventional ALS in terms of complexity, is the additionally required storage for the covariance matrices. Thus, the number of elements of one TD component, depending on the ranks, will be the limiting factor for the computational complexity. 

\begin{algorithm}[H]
\caption{ALS in Bayesian framework with orthogonalization step}
\label{alg:ortho}
\begin{algorithmic}[1]
\REQUIRE{Prior mean $\{\m_i^0\}$ and covariance $\{\cov^0_i\}$, $i=1,\dots N$, measurement $\y$ and noise variance $\sigma^2$.}
\ENSURE{Posterior mean $\{\m_i^+\}$ and covariance $\{\cov_i^+\}$, $i=1,\dots N$.}
\STATE{Transform random variables such that the corresponding TT decomposition is in site-$1$-mixed-canonical form.}
\setstretch{1.5}
\STATE{Compute $\U_{\setminus 1}$ with \cref{eq:multilin} for TT with $\{\m_i^0\}_{i\neq1}$.}
\STATE{Set $\{\m_i\}:= \{\m_i^0\}$, $\{\cov_i\}:= \{\cov_i^0\}$.}
\WHILE{stopping criterion is not true}
\FOR{$n=1,\dots,N,N-1,\dots 2$}
\STATE{$\cov_n^+ \leftarrow \left[(\cov_n^{0})^{-1}+\frac{\id}{\sigma^{2}}\right]^{-1}$}
\STATE{$\m_n^+ \leftarrow \cov_n^+ \left[ \frac{\U_{\setminus n}^\top \y}{\sigma^2} + \left( \cov_n^0 \right)^{-1}\m_n^0 \right]$}
\IF{next core is to the right}
\STATE{$\m_n^+ \leftarrow \operatorname{vec}(\Q_n^\mathrm{L})$, with $\Q_n^\mathrm{L}$ from thin $\mathbf{QR}$-factorization of $\G_n^\mathrm{L}$}
\STATE{$\cov_n^+ \leftarrow \left( \left( \mathbf{R}_n^\mathrm{L} \right)^{-\top} \otimes \id \right) \; \cov_n^+ \; \left( \left( \mathbf{R}_n^\mathrm{L} \right)^{-1} \otimes \id \right)$}
\STATE{$\m_{n+1}\leftarrow(\id\otimes \R^\mathrm{L}_n)\m_{n+1}$}
\STATE{$\cov_{n+1}\leftarrow (\id\otimes \R^\mathrm{L}_n)\; \cov_{n+1} \; (\id\otimes \left(\R^\mathrm{L}_n\right)^\top)$}
\STATE{Recompute $\U_{\setminus n+1}$ with $\m_n^+$.}
\ELSIF{next core is on the left}
\STATE{$\m_n^+ \leftarrow \operatorname{vec}\left(\left(\Q^\mathrm{R}_n\right)^\top\right)$, with $\Q_n^\mathrm{R}$ from thin $\mathbf{QR}$-factorization of $\G_n^\mathrm{R}$}
\STATE{$\cov_n^+ \leftarrow \left(\id\otimes\left(\R_n^\mathrm{R}\right)^{-\top}\right) \; \cov_n^+ \; \left(\id\otimes\left(\R_n^\mathrm{R}\right)^{-1}\right)$}
\STATE{$\m_{n-1}\leftarrow(\R_n^\mathrm{R}\otimes \id)\m_{n-1}$}
\STATE{$\cov_{n-1} \leftarrow(\R_n^\mathrm{R}\otimes \id)\; \cov_{n-1} \; \left(\left(\R_n^\mathrm{R}\right)^\top\otimes \id\right)$}
\STATE{Recompute $\U_{\setminus n-1}$ with $\m_n^+$.}
\ENDIF
\STATE{$\m_n\leftarrow \m_n^+$, $\cov_n\leftarrow \cov_n^+$}
\STATE{Apply the transformations of the lines 7-10 or 12-15 to $\m_n^0,\cov_n^0$.}
\ENDFOR
\ENDWHILE
\end{algorithmic}
\end{algorithm}

\section{Unscented transform in TT format}
\label{sec:UT}
In \cref{alg:noqr} and \cref{alg:ortho} we compute the posterior distributions of the TT-cores $p(\g_n \mid \{\g_i\}_{i\neq n}, \y)$. However, we are interested in computing the distribution of the low-rank tensor estimate $\Gt$, which is computed with a non-linear function dependent on the posterior distributions and is, therefore, not Gaussian. The unscented transform (UT) \cite{Julier2004} can approximate the mean $\m_\textrm{UT}$ and covariance $\cov_\textrm{UT}$ of the sought distribution. In this section, we show how we can perform the UT in TT format. In this way, the direct computation of the potentially large covariance matrix can be avoided.\\

Generally, the UT approximates the mean and covariance of a distribution that is a non-linear function of a known distribution $\h \sim \mathcal{N}(\m,\cov)$ with mean $\m\in\mathbb{R}^{M\times 1}$ and covariance $\cov\in\mathbb{R}^{M\times M}$ \cite[p.\ 81-84]{Sarkka2010}.
Firstly, $2M+1$ sigma points are formed with
\begin{align}
    \x^{(0)} &=\m, \label{eq:sigmapoints1}\\
    \x^{(i)} &=\m+\sqrt{M+\lambda}\;\left[\sqrt{\cov}\right]_{i}, \quad i=1, \ldots, M, \label{eq:sigmapoints2}\\
    \x^{(i+M)} &=\m-\sqrt{M+\lambda}\; \left[\sqrt{\cov} \right]_{i},\quad i=1, \ldots, M,
\label{eq:sigmapoints3}
\end{align}
where the square root of the covariance matrix $\sqrt{\cov}$ corresponds to the Cholesky factor, such that $\sqrt{\cov}\sqrt{\cov}^\top=\cov$, where $\left[\sqrt{\cov}\right]_{i}$ is the $i$-th column of that matrix. The scaling parameter $\lambda$ is defined as $\lambda=\alpha^2(M+\kappa)-M$, where $\alpha$ and $\kappa$ determine the spread of the sigma points around the mean. Secondly, the sigma points are propagated through the non-linearity and thirdly, the approximated mean $\m_\mathrm{UT}$ and covariance $\cov_\mathrm{UT}$ are computed as
\begin{align}
    \m_\mathrm{UT} &= \sum_{i=0}^{2 M} w_i^\m \mathbfcal{S}^{(i)}, \label{eq:UT_mean}\\
     \cov_\mathrm{UT}&=\sum_{i=0}^{2 M} w_i^\cov\left(\mathbfcal{S}^{(i)}-\m_\mathrm{UT}\right)\left(\mathbfcal{S}^{(i)}-\m_\mathrm{UT}\right)^\top,
\label{eq:covUT}
\end{align}
where $\mathbfcal{S}^{(i)}$ are the transformed sigma points. The scalars $w_i^\m$ and $w_i^\cov$ denote weighting factors, defined as
\begin{align*}
w_0^\m &=\frac{\lambda}{M+\lambda},\quad w_0^\cov =w_0^\m+\left(1-\alpha^{2}+\beta\right),\\
w_i^\m = w^\m &= w_i^\cov = w^\cov = \frac{1}{2(M+\lambda)}, \quad i=1, \ldots, 2 M. \label{eq:weights}
\end{align*}
Literature suggests $\alpha=0.001$, $\kappa=3-M$ \cite[p.\ 229]{Haykin2001} and for Gaussian distributions $\beta=2$ \cite[p.\ 229]{Sarkka2010}.\\

In order to use the UT in TT format, the known distribution $\h\sim \mathcal{N}(\m,\cov)$ is computed from the cores' mean and covariance as
\begin{align}
\h \sim \mathcal{N}( \mathbf{m},\cov) = \mathcal{N}\left(
\begin{bmatrix}
\m_1\\
\m_2\\
\vdots\\
\m_N
\end{bmatrix},
\begin{bmatrix}
\cov_1 & 0 & \dots & 0 \\
0 & \cov_2 & \ddots & \vdots \\
\vdots & \ddots & \ddots & 0 \\
0 & \dots & 0 &\cov_N
\end{bmatrix}
\right).\label{eq:g}
\end{align}
The mean, consisting of the stacked vectorized cores, is of size $M\times 1$ with
\begin{equation*}
    M=\sum_{n=1}^NR_nI_nR_{n+1}, \quad R_1=R_{N+1}=1. 
\end{equation*}

The covariance matrix of size $M\times M$ is block diagonal, since we assume the TT-cores to be statistically independent. The non-linear function for the UT in TT format is defined as
\begin{equation}
f_\mathrm{T}:\; \mathbb{R}^{M\times1} \rightarrow \mathbb{R}^{I_1\times I_2 \times ... \times I_N}\;\; \text{given by}\;\; \x \mapsto \mathbfcal{G},
\label{eq:UT_fT}
\end{equation}
where $\x$ is a vector of size $M\times1$. The transformation of a vector into a tensor and is depicted in \cref{fig:UT_nonlinfun}.\\

\begin{figure}
\centering
\begin{tikzpicture}

\draw  [line width=0.3mm](-4,1.7) ellipse (0.4 and 0.4);
\draw  [line width=0.3mm](-2.8,1.7) ellipse (0.4 and 0.4);
\draw [line width=0.3mm] (-0.7,1.7) ellipse (0.4 and 0.4);
\draw [line width=0.3mm](-3.6,1.7) -- (-3.2,1.7);
\draw [line width=0.3mm](-1.5,1.7) -- (-1.1,1.7);
\draw [line width=0.3mm](-4,1.3) -- (-4,0.8);
\draw [line width=0.3mm](-2.8,1.3) -- (-2.8,0.8);
\draw [line width=0.3mm] (-0.7,1.3) -- (-0.7,0.8);
\node at (-4,1.7) {$\mathbfcal{G}_1$};
\node at (-2.8,1.7) {$\mathbfcal{G}_2$};
\node at (-0.7,1.7) {$\mathbfcal{G}_N$};
\draw [line width=0.3mm] (-6.6,1.7) ellipse (0.4 and 0.4);
\node at (-6.6,1.7) {$\mathbf{x}$};
\draw[line width=0.3mm] (-6.6,1.3) -- (-6.6,0.8);
\node at (0,1.6) {=};
\draw [line width=0.3mm] (0.7,1.7) ellipse (0.4 and 0.4);
\node at (0.7,1.7) {$\mathbfcal{G}$};
\draw [line width=0.3mm](0.37,1.46) -- (0.37,0.8);
\draw [line width=0.3mm](0.62,1.3) -- (0.62,0.8);
\draw [line width=0.3mm](1.04,0.8) -- (1.04,1.45);
\draw[line width=0.3mm] [->] (-5.8,1.7) -- (-4.8,1.7);
\node at (-5.3,1.92) {$f_\mathrm{T}$};
\node at (-6.6,0.5) {$M$};
\node at (-3.8,1) {$I_1$};
\node at (-2.6,1) {$I_2$};
\node at (-0.43,1) {$I_N$};
\node at (0.34,0.6) {$I_1$};
\node at (0.62,0.6) {$I_2$};
\node at (1.05,0.6) {$I_N$};
\node at (-3.4,2) {$R_2$};
\node at (-1.3,2) {$R_N$};
\draw [line width=0.3mm](-2.4,1.7) -- (-2,1.7);
\node at (-2.2,2) {$R_3$};
\node at (-1.72,1.7) {$\dots$};
\node at (0.84,1) {$...$};
\end{tikzpicture}
\caption{Visual depiction of the non-linear transformation from vector $\x$ into a TT with cores $\Gt_i$, $i=1,\dots,N$ that represents tensor $\Gt$.}
\label{fig:UT_nonlinfun}
\end{figure}
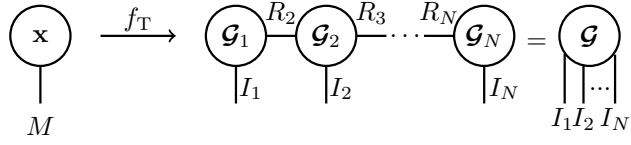

The formation and propagation of the sigma points works as follows. The first sigma point $\x^{(0)}$ is the mean $\m$ from \cref{eq:g} and propagated through the non-linearity, it corresponds to the TT represented by the distributions determined by \cref{alg:noqr}. To facilitate later steps, the remaining sigma points from \cref{eq:sigmapoints2} and \cref{eq:sigmapoints3} are organized into two matrices, according to

\begin{align}
    \mathbf{A}_+ &= \sum_{i=1}^M \x^{(i)}\e_i^\top
    \label{eq:sigmamatrix1}\\
    \mathbf{A}_- &= \sum_{i=1}^M \x^{(M+i)}\e_i^\top,
    \label{eq:sigmamatrix2}
\end{align}
where $\mathbf{e}_i$ denotes a vector with zeros everywhere except a 1 at location $i$. In this way, the propagation through the non-linearity of all sigma points then becomes a propagation of every summand $\x^{(i)}\e_i^\top$ and $\x^{(M+i)}\e_i^\top$, respectively. The propagation is achieved by forming TT-cores from $\x^{(i)}$ and $\x^{(M+i)}$ and summing over a rank-1 connection between the first core and the vector $\e_i^\top$ as shown in \cref{fig:UT_propagation}.

\begin{figure}
\centering
\begin{tikzpicture}
\draw [line width=0.3mm] (-3.4,2) ellipse (0.3 and 0.3);
\draw [line width=0.3mm] (-2.5,2) ellipse (0.3 and 0.3);
\draw [line width=0.3mm] (-1.6,2) ellipse (0.3 and 0.3);
\draw [line width=0.3mm] (-3.4,2.9) ellipse (0.3 and 0.3);
\draw [line width=0.3mm](-3.1,2) -- (-2.8,2);
\draw [line width=0.3mm](-2.2,2) -- (-1.9,2);
\draw [line width=0.3mm](-3.4,1.7) -- (-3.4,1.3);
\draw [line width=0.3mm](-2.5,1.7) -- (-2.5,1.3);
\draw [line width=0.3mm](-1.6,1.7) -- (-1.6,1.3);
\draw [line width=0.5mm,dotted](-3.4,2.3) -- (-3.4,2.6);
\draw [line width=0.3mm](-3.4,3.2) -- (-3.4,3.5);
\node at (-3.2,3.4) {$M$};
\node at (-3.4,2.9) {$\mathbf{e}_j^\top$};
\node at (-7.2,2.4) {$\mathbf{x}^{(i)}\mathbf{e}_j^\top=
\begin{bmatrix}  
0 \cdots 0 \;\;\;\;\;\; 0\cdots 0 \\
\vdots\;\;\;\;\;\;\;\;  \;\;\;\;\;\;\vdots\\
\mathbf{x}^{(i)}\\
\vdots\;\;\;\;\;\;\;\;  \;\;\;\;\;\;\vdots\\
0 \cdots 0 \;\;\;\;\;\; 0\cdots 0
\end{bmatrix}$};
\node at (-4.5,2.4) {$\rightarrow$};
\draw  [line width=0.3mm](-6.5,2.7) -- (-6.5,3.5);
\draw  [line width=0.3mm](-6.5,1.3) -- (-6.5,2.2);
\node at (0.8,2.67) {$i=1\dots 2M+1$};
\node at (0.39,2.17) {$j=1\dots M$};
\end{tikzpicture}
\caption{Propagation of each sigma point as column of matrix $ \x^{(i)}\e_i^\top$ through the non-linearity.}
\label{fig:UT_propagation}
\end{figure}
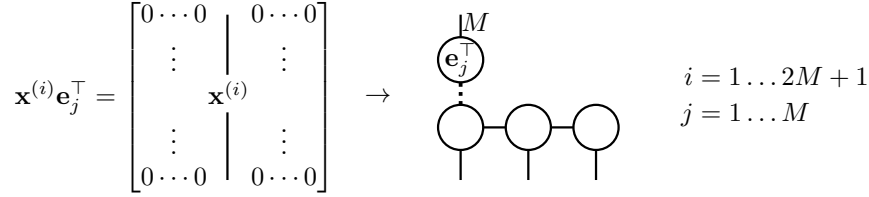

Then, all summands of $\mathbf{A}_+ \in \mathbb{R}^{M\times M}$ and $\mathbf{A}_- \in \mathbb{R}^{M\times M}$, are summed together, respectively, by stacking the cores according to \cite[p.\ 2308]{Oseledets2012}, as illustrated in \cref{fig:UT_mean} (left). The summation causes the ranks of the TTm to increase and a rounding procedure \cite[p.\ 2301-2305]{Oseledets2011} needs to be applied to reduce the ranks back to the required precision. Finally, the vector containing the weights $\mathbf{w^m}=w^\m\mathbf{1}_M$ is absorbed into the first core (\cref{fig:UT_mean}, right).\\

\begin{figure}
\centering
\begin{tikzpicture}
\draw  [line width=0.3mm](-1.5,-0.3) ellipse (0.3 and 0.3);
\draw  [line width=0.3mm](-0.6,-0.3) ellipse (0.3 and 0.3);
\draw  [line width=0.3mm](0.3,-0.3) ellipse (0.3 and 0.3);
\draw [line width=0.3mm](-1.2,-0.3) -- (-0.9,-0.3);
\draw [line width=0.3mm](-0.3,-0.3) -- (0,-0.3);
\draw [line width=0.3mm](-1.5,-0.6) -- (-1.5,-1);
\draw [line width=0.3mm](-0.6,-0.6) -- (-0.6,-1);
\draw [line width=0.3mm](0.3,-0.6) -- (0.3,-1);
\node at (-1.3,-0.9) {$I_1$};
\node at (-0.4,-0.9) {$I_2$};
\node at (0.5,-0.9) {$I_3$};
\draw  [line width=0.3mm](-1.5,0.6) ellipse (0.3 and 0.3);
\node at (-1.1,1.1) {$M$};
\draw [dotted,line width=0.5mm](-1.5,0.3) -- (-1.5,0);
\draw [line width=0.3mm](-1.5,0.9) -- (-1.5,1.2);
\node at (-1.5,0.6) {$\mathbf{e}_j^\top$};
\node at (-2.4,-0.2) {$\sum_{j=1}^{M}$};
\draw [line width=0.3mm] (1.9,-0.3) ellipse (0.3 and 0.3);
\draw [line width=0.3mm] (2.8,-0.3) ellipse (0.3 and 0.3);
\draw [line width=0.3mm] (3.7,-0.3) ellipse (0.3 and 0.3);
\draw  [line width=0.3mm](2.2,-0.3) -- (2.5,-0.3);
\draw  [line width=0.3mm](3.1,-0.3) -- (3.4,-0.3);
\draw  [line width=0.3mm](1.9,-0.6) -- (1.9,-1);
\draw  [line width=0.3mm](2.8,-0.6) -- (2.8,-1);
\draw  [line width=0.3mm](3.7,-0.6) -- (3.7,-1);
\draw  [line width=0.3mm](1.9,0) -- (1.9,0.3);
\node at (2.2,0.2) {$M$};
\node at (1.3,-0.3) {=};
\draw   [line width=0.3mm](5.1,-0.3) ellipse (0.3 and 0.3);
\draw   [line width=0.3mm](6,-0.3) ellipse (0.3 and 0.3);
\draw   [line width=0.3mm](6.9,-0.3) ellipse (0.3 and 0.3);
\draw  [line width=0.3mm](5.4,-0.3) -- (5.7,-0.3);
\draw  [line width=0.3mm](6.3,-0.3) -- (6.6,-0.3);
\draw  [line width=0.3mm](5.1,-0.6) -- (5.1,-1);
\draw  [line width=0.3mm](6,-0.6) -- (6,-1);
\draw  [line width=0.3mm](6.9,-0.6) -- (6.9,-1);
\draw  [line width=0.3mm](5.1,0) -- (5.1,0.4);
\node at (5.4,0.2) {$M$};
\node at (5.3,-0.9) {$I_1$};
\node at (6.2,-0.9) {$I_2$};
\node at (7.1,-0.9) {$I_3$};
\draw  [line width=0.3mm] (5.1,0.7) ellipse (0.3 and 0.3);
\node at (5.1,0.7) {$\mathbf{w}^\mathbf{m}$};

\draw (4.4,1.3) -- (4.4,-1);
\end{tikzpicture}
\caption{Visual depiction of \cref{eq:UT_mean} as sum over sigma points (left) and absorption of the weight vector $\mathbf{w^m}$ into the first core (right).}
\label{fig:UT_mean}
\end{figure}
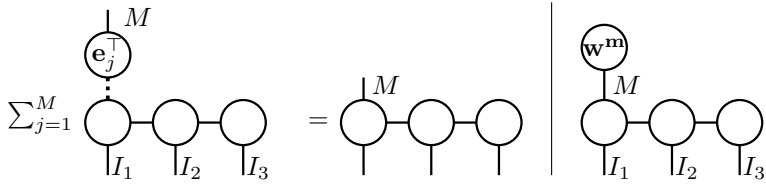

The computation of the covariance matrix from \cref{eq:covUT} is divided into two steps. Firstly, $\m_\mathrm{UT}$ is subtracted from the sum over the sigma points (\cref{fig:UT_mean}, left). This is achieved by creating a matrix, where $\m_\mathrm{UT}$ is stacked $M$ times next to each other. The visual depiction of this operation equals to the one in \cref{fig:UT_propagation} with the difference that the multiplied vector is $\mathbf{1}_M^\top$. Secondly, the result from the first step absorbs the weights into the first core, as depicted in \cref{fig:UT_op}. The resulting covariance matrix $\cov_\mathrm{UT}$ in the three-core example is a TT matrix that corresponds to a matrix of size $I_1I_2I_3\times I_1I_2I_3$.\\

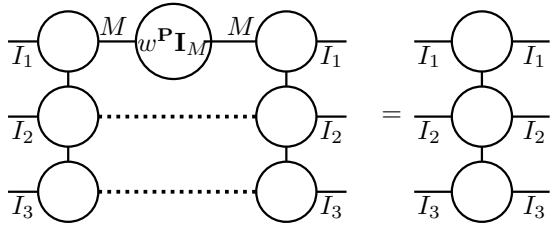
\begin{figure}
\centering
\begin{tikzpicture}
\draw [line width=0.3mm] (-6,3.7) ellipse (0.4 and 0.4);
\draw [line width=0.3mm] (-6,4.7) ellipse (0.4 and 0.4);
\draw [line width=0.3mm] (-6,2.7) ellipse (0.4 and 0.4);
\draw [line width=0.3mm] (-3.1,4.7) ellipse (0.4 and 0.4);
\draw [line width=0.3mm] (-3.1,3.7) ellipse (0.4 and 0.4);
\draw [line width=0.3mm] (-3.1,2.7) ellipse (0.4 and 0.4);
\draw [line width=0.3mm] (-4.6,4.7) ellipse (0.5 and 0.5);
\draw [line width=0.3mm](-3.1,4.3) -- (-3.1,4.1);
\draw [line width=0.3mm](-3.1,3.3) -- (-3.1,3.1);
\draw [line width=0.3mm](-6,4.3) -- (-6,4.1);
\draw [line width=0.3mm](-6,3.3) -- (-6,3.1);
\draw [line width=0.3mm](-6.4,4.7) -- (-6.8,4.7);
\draw [line width=0.3mm](-6.4,3.7) -- (-6.8,3.7);
\draw [line width=0.3mm](-6.4,2.7) -- (-6.8,2.7);
\node at (-6.6,4.5) {$I_1$};
\node at (-6.6,3.5) {$I_2$};
\node at (-6.6,2.5) {$I_3$};
\draw [line width=0.3mm](-4.2,4.7) -- (-3.5,4.7);
\node at (-3.7,4.9) {$M$};
\draw [line width=0.3mm](-2.7,3.7) -- (-2.3,3.7);
\draw [line width=0.3mm](-2.7,2.7) -- (-2.3,2.7);
\draw [line width=0.3mm](-2.7,4.7) -- (-2.3,4.7);
\node at (-2.5,4.5) {$I_1$};
\node at (-2.5,3.5) {$I_2$};
\node at (-2.5,2.5) {$I_3$};
\draw [line width=0.3mm](-5.6,4.7) -- (-5.1,4.7);
\node at (-5.4,4.9) {$M$};
\node at (-4.6,4.7) {$w^\mathbf{P}\mathbf{I}_M$};
\node at (-1.7,3.7) {=};
\draw  [line width=0.3mm](-0.5,4.7) ellipse (0.4 and 0.4);
\draw  [line width=0.3mm](-0.5,3.7) ellipse (0.4 and 0.4);
\draw  [line width=0.3mm](-0.5,2.7) ellipse (0.4 and 0.4);
\draw [line width=0.3mm](-0.5,4.3) -- (-0.5,4.1);
\draw [line width=0.3mm](-0.5,3.3) -- (-0.5,3.1);
\draw [line width=0.3mm](-0.1,4.7) -- (0.4,4.7);
\draw [line width=0.3mm](-0.1,3.7) -- (0.4,3.7);
\draw [line width=0.3mm](-0.1,2.7) -- (0.4,2.7);
\draw [line width=0.3mm](-0.9,4.7) -- (-1.4,4.7);
\draw [line width=0.3mm](-0.9,3.7) -- (-1.4,3.7);
\draw [line width=0.3mm](-0.9,2.7) -- (-1.4,2.7);
\node at (-1.2,4.5) {$I_1$};
\node at (-1.2,3.5) {$I_2$};
\node at (-1.2,2.5) {$I_3$};
\node at (0.2,4.5) {$I_1$};
\node at (0.2,3.5) {$I_2$};
\node at (0.2,2.5) {$I_3$};
\draw [line width=0.5mm,dotted](-5.6,3.7) -- (-3.5,3.7);
\draw [line width=0.5mm,dotted](-5.6,2.7) -- (-3.5,2.7);
\end{tikzpicture}
\caption{Visual depiction of \cref{eq:covUT} with $w^\cov\id_M$ as a diagonal matrix containing the weight factors on the diagonal.}
\label{fig:UT_op}
\end{figure}
The computation of the approximate mean and covariance with the UT in TT format is summarized in \cref{alg:UTalg}. The computational cost depends on the ranks of the TD, since $M$ is a function of the ranks. The bottleneck of \cref{alg:UTalg} is the rounding procedure necessary after performing summations in TT format. It has a cost of $\mathcal{O}(R^3IN)$.

\begin{algorithm}
\caption{Approximation of the low-rank tensor estimate's mean and covariance with the unscented transform in TT format.}
\label{alg:UTalg}
\begin{algorithmic}[1]
\REQUIRE{The mean and covariances of each TT-core $\{\m_i,\cov_i\}$, $i=1,\dots,N$ computed with the ALS in a Bayesian framework.}
\ENSURE{The approximated mean $\m_\mathrm{UT}$ and covariance $\cov_\mathrm{UT}$ in TT format of the low-rank tensor estimate's distribution.}
\STATE{Compute sigma point $\x^{(0)}$ with \cref{eq:sigmapoints1}.}
\STATE{Compute remaining sigma points with \cref{eq:sigmapoints2,eq:sigmapoints3} and organize them into groups according to \cref{eq:sigmamatrix1,eq:sigmamatrix2}.}
\STATE{Propagate sigma points through \cref{eq:UT_fT}, where groups from step 2 are propagated as shown in \cref{fig:UT_nonlinfun}.}
\STATE{Estimate the mean $\m_\mathrm{UT}$ with \cref{eq:UT_mean} as shown in \cref{fig:UT_mean}.}
\STATE{Estimate the covariance $\cov_\mathrm{UT}$ with \cref{eq:covUT} as shown in \cref{fig:UT_op}.}
\end{algorithmic}
\end{algorithm}

\section{Numerical experiments}
In this section, we present the numerical examples that test the algorithms. All experiments with exception of the last were performed with MATLAB R2020b on a Dell computer with processor Intel(R) Core(TM) i7-8650U CPU @ 1.90GHz 2.11 GHz and 8GB of RAM. The last experiment is performed on a Lenovo computer with processor Intel(R) Core(TM) i7-10700KF CPU @ 3.80GHz 3.79 GHz and 16GB of RAM. The implementation of the experiments can be found on \url{https://gitlab.tudelft.nl/cmmenzen/bayesian-als}.\\

The first three experiments are performed with a random TT, $\mathbfcal{G}$, that represents the ground truth and has the cores
\begin{equation*}
  \Gt_{1,\mathrm{truth}}\in\mathbb{R}^{1\times5\times3}, \quad \Gt_{2,\mathrm{truth}}\in\mathbb{R}^{3\times5\times3}\;\; \text{and}\;\; \Gt_{3,\mathrm{truth}}\in\mathbb{R}^{3\times5\times1}. 
\end{equation*} 
The entries of each TT-core are drawn from a standard normal distribution. After computing the tensor $\Yt_\mathrm{truth}\in\mathbb{R}^{5\times5\times5}$ from the TT-cores and vectorizing it, a noisy sample $\y$ is formed with
\begin{equation}
    \y=\y_\mathrm{truth}+\boldsymbol{\epsilon} \qquad \boldsymbol{\epsilon}\sim\mathcal{N}(\boldsymbol{0},\sigma^2\id),
    \label{eq:noisyinstances}
\end{equation}
where $\y_\mathrm{truth}$ denotes the vectorized ground truth and $\boldsymbol{\epsilon}$ is a realization of random noise. The noisy samples of the same underlying tensor formed with \cref{eq:noisyinstances} are uncorrelated. The covariance of the measurement noise is influenced by fixing the signal-to-noise ratio
\begin{equation*}
    \mathrm{SNR_{dB}} = 10 \log_{10} \frac{\|\y\|^2}{\|\boldsymbol{\epsilon}\|^2}.
\end{equation*}
If not stated otherwise, the signal-to-noise ratio is set to zero dB. Some experiments use multiple noisy samples $\y$, computed from \cref{eq:noisyinstances}. In this case, the estimate is recursively updated. Initially, the prior TT is inputted to \cref{alg:noqr} together with a sample $\y$. After the execution of \cref{alg:noqr}, the output mean and covariance is used as a prior for the next execution together with a new sample $\y$. This recursive updating is very suitable for the ALS in a Bayesian framework, because it can deal with prior knowledge on the TD components. For the conventional ALS, the estimate from an execution of the algorithm that computes a TT with the ALS is used as an initial TT for the next execution.  

\subsection{Convergence analysis of maximization problem} 
\label{sec:convergence}
In the ALS in a Bayesian framework, we solve the optimization problem given by \cref{eq:optproblem}. In this context, we define the relative error between the low-rank estimate $\g$ and the ground truth $\y_\mathrm{truth}$ as
\begin{equation*}
    \varepsilon_\mathrm{truth} = \frac{\|\y_\mathrm{truth}-\g\|}{\|\y_\mathrm{truth}\|}
\end{equation*}
and the relative error between the low-rank estimate $\g$ and the noisy sample $\y$ as
\begin{equation*}
    \varepsilon_\mathrm{meas} = \frac{\|\y-\g\|}{\|\y\|}.
\end{equation*}
In the first experiment, we look at the errors defined above in order to analyze the convergence of \cref{alg:noqr}. In addition, we look at the evolution of the log likelihood times the prior, since from \cref{thm:jointpost} it follows that the numerator of the logarithm of \cref{eq:Bayes} needs to be maximized to compute the posterior of all TD components. \\

In this experiment, only one noisy sample $\y$ is used. The prior mean is initialized randomly and the covariance for each core is set to $200^2\id$, meaning a low certainty on the prior mean. The experiment is performed 100 times with the same TT, $\mathbfcal{G}$, but with different priors. Then, the mean of the 100 results is plotted with a region of twice the standard deviation. \Cref{fig:convergence} (left) shows how both relative errors decrease rapidly and converge after approximately 5 iterations in \cref{alg:noqr}. \Cref{fig:convergence} (right) shows how the product of log likelihood and prior increases during the first approximately 6 iterations, converging to a fixed value. Both subfigures of \cref{fig:convergence} also show how the region of twice the standard deviation from the 100 trials, becomes smaller with an increasing number of iterations. Hence, it can be concluded that \cref{alg:noqr} converges and therefore also the optimization problem.

\begin{figure}
    \centering
%
%
\definecolor{mycolor2}{rgb}{0,0,0}%
\definecolor{mycolor3}{rgb}{0.7 0.7 0.7}%
\definecolor{mycolor1}{rgb}{0.8 0.8 0.8}%
\begin{tikzpicture}

\begin{axis}[%
width=1.7in,
height=1.5in,
at={(0.758in,0.481in)},
scale only axis,
xmin=1,
xmax=20,
xlabel style={font=\color{white!15!black}},
xlabel={Iterations},
ymin=0.4,
ymax=1.2,
ylabel style={font=\color{white!15!black}},
ylabel={Relative error},
axis background/.style={fill=white},
axis x line*=bottom,
axis y line*=left,
legend style={legend cell align=left, align=left, draw=white!15!black}
]

\addplot[area legend, draw=none, fill=mycolor1, forget plot]
table[row sep=crcr] {%
x	y\\
1	1.10370966338837\\
2	0.927405894670975\\
3	0.86907326932338\\
4	0.851088340563966\\
5	0.843386832956193\\
6	0.83827121180184\\
7	0.83451943159716\\
8	0.831932066225677\\
9	0.829953976545478\\
10	0.828309107146338\\
11	0.827033833795604\\
12	0.826275514129881\\
13	0.826078822469235\\
14	0.825959599980788\\
15	0.825810353486496\\
16	0.825512898182099\\
17	0.825220716332982\\
18	0.825061705147904\\
19	0.824967461262897\\
20	0.824892010859262\\
20	0.817315496216783\\
19	0.817097752708477\\
18	0.816840779022254\\
17	0.816490968294639\\
16	0.815976835567575\\
15	0.815434598676041\\
14	0.815018760781016\\
13	0.814600503876279\\
12	0.814058685114188\\
11	0.812821896517655\\
10	0.810819522810542\\
9	0.808074042417551\\
8	0.804464959443574\\
7	0.79958282694542\\
6	0.792906555140509\\
5	0.784200383785283\\
4	0.771714284293232\\
3	0.748033910934126\\
2	0.69740165481406\\
1	0.676702649523528\\
}--cycle;
\addplot [color=mycolor2, line width=1.0pt]
  table[row sep=crcr]{%
1	0.890206156455949\\
2	0.812403774742517\\
3	0.808553590128753\\
4	0.811401312428599\\
5	0.813793608370738\\
6	0.815588883471175\\
7	0.81705112927129\\
8	0.818198512834626\\
9	0.819014009481514\\
10	0.81956431497844\\
11	0.819927865156629\\
12	0.820167099622035\\
13	0.820339663172757\\
14	0.820489180380902\\
15	0.820622476081268\\
16	0.820744866874837\\
17	0.82085584231381\\
18	0.820951242085079\\
19	0.821032606985687\\
20	0.821103753538023\\
};
\addlegendentry{$\varepsilon_\mathrm{truth}$}

\addplot[area legend, draw=none, fill=mycolor3, forget plot]
table[row sep=crcr] {%
x	y\\
1	0.80009378048444\\
2	0.586284186936996\\
3	0.517428504544162\\
4	0.487593506027529\\
5	0.468072236034699\\
6	0.454407573846288\\
7	0.445283110440066\\
8	0.439082335078049\\
9	0.434812911710348\\
10	0.432404034895834\\
11	0.431011307980397\\
12	0.429397626736794\\
13	0.427529716133122\\
14	0.426531016134975\\
15	0.425641667424169\\
16	0.424722538323039\\
17	0.424039699636544\\
18	0.423753410988007\\
19	0.423687022502292\\
20	0.423664650059984\\
20	0.423465092803918\\
19	0.423457463716731\\
18	0.423410416035752\\
17	0.423153341378632\\
16	0.422515642390413\\
15	0.421655548973938\\
14	0.420839180470171\\
13	0.419952367110611\\
12	0.418268249826378\\
11	0.416904381820289\\
10	0.415933725359231\\
9	0.414329159909985\\
8	0.41156515452325\\
7	0.408077824956246\\
6	0.403874904840275\\
5	0.398997603733503\\
4	0.395213167157693\\
3	0.396170659924925\\
2	0.405787499947355\\
1	0.497101875828161\\
}--cycle;
\addplot [color=mycolor2, dashed, line width=1.0pt]
  table[row sep=crcr]{%
1	0.648597828156301\\
2	0.496035843442175\\
3	0.456799582234543\\
4	0.441403336592611\\
5	0.433534919884101\\
6	0.429141239343282\\
7	0.426680467698156\\
8	0.425323744800649\\
9	0.424571035810167\\
10	0.424168880127532\\
11	0.423957844900343\\
12	0.423832938281586\\
13	0.423741041621867\\
14	0.423685098302573\\
15	0.423648608199053\\
16	0.423619090356726\\
17	0.423596520507588\\
18	0.423581913511879\\
19	0.423572243109511\\
20	0.423564871431951\\
};
\addlegendentry{$\varepsilon_\mathrm{meas}$}

\end{axis}
\end{tikzpicture}%
%
%
\definecolor{mycolor2}{rgb}{0 0 0}%
\definecolor{mycolor1}{rgb}{0.8,0.8,0.8}%
\begin{tikzpicture}

\begin{axis}[%
width=1.7in,
height=1.5in,
at={(0.758in,0.481in)},
scale only axis,
xmin=1,
xmax=20,
xlabel style={font=\color{white!15!black}},
xlabel={Iterations},
ymin=-1370,
ymax=-1352,
ylabel style={font=\color{white!15!black}},
ylabel={$\log \mathrm{Likelihood } \cdot \mathrm{ Prior}$},
axis background/.style={fill=white},
axis x line*=bottom,
axis y line*=left,
legend style={legend cell align=left, align=left, draw=white!15!black}
]

\addplot[area legend, draw=none, fill=mycolor1]
table[row sep=crcr] {%
x	y\\
1	-1356.23815463544\\
2	-1354.13940552114\\
3	-1354.0369577568\\
4	-1353.98843482358\\
5	-1354.05418023885\\
6	-1354.14653913299\\
7	-1354.22295032678\\
8	-1354.28460107797\\
9	-1354.33031164905\\
10	-1354.35150929058\\
11	-1354.36111136741\\
12	-1354.37855193202\\
13	-1354.39605484385\\
14	-1354.40144102587\\
15	-1354.40568939751\\
16	-1354.40850314598\\
17	-1354.40852313814\\
18	-1354.4069505602\\
19	-1354.40505608938\\
20	-1354.40316317362\\
20	-1354.70264275136\\
19	-1354.71464335315\\
18	-1354.72795294627\\
17	-1354.74320251425\\
16	-1354.7620355132\\
15	-1354.78587845913\\
14	-1354.81369832854\\
13	-1354.84624466752\\
12	-1354.89594507504\\
11	-1354.95094879168\\
10	-1355.00682914936\\
9	-1355.08944943336\\
8	-1355.22127374708\\
7	-1355.40852720843\\
6	-1355.67867051953\\
5	-1356.08086310414\\
4	-1356.6610472595\\
3	-1357.56588705418\\
2	-1359.86686918479\\
1	-1368.3856649943\\
}--cycle;

\addplot [color=mycolor2, line width=1.0pt]
  table[row sep=crcr]{%
1	-1362.31190981487\\
2	-1357.00313735296\\
3	-1355.80142240549\\
4	-1355.32474104154\\
5	-1355.0675216715\\
6	-1354.91260482626\\
7	-1354.8157387676\\
8	-1354.75293741253\\
9	-1354.70988054121\\
10	-1354.67916921997\\
11	-1354.65603007955\\
12	-1354.63724850353\\
13	-1354.62114975569\\
14	-1354.6075696772\\
15	-1354.59578392832\\
16	-1354.58526932959\\
17	-1354.57586282619\\
18	-1354.56745175323\\
19	-1354.55984972126\\
20	-1354.55290296249\\
};

\end{axis}
\end{tikzpicture}%
    \caption{Left: Evolution of the relative errors during 20 iterations in \cref{alg:noqr}. Right: Evolution of log likelihood times the prior during 20 iterations in \cref{alg:noqr}.}
    \label{fig:convergence}
\end{figure}
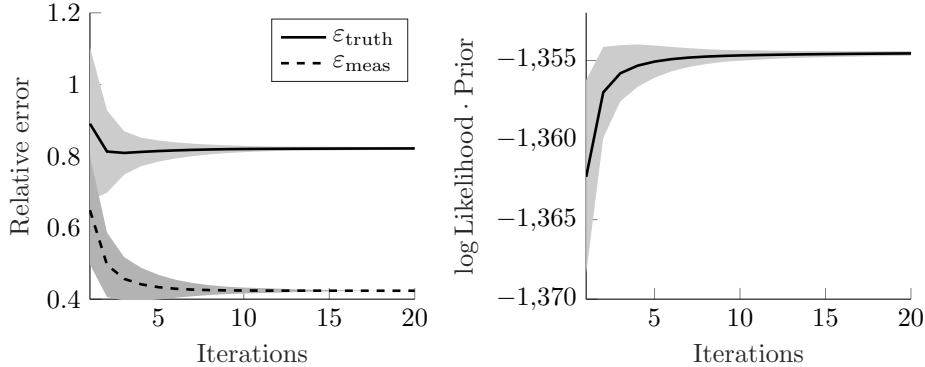


\subsection{Analysis of covariance matrices}
In the second experiment, we look at how the covariance matrix of each core changes throughout the iterations in \cref{alg:ortho}. We also examine the covariance matrix of the low-rank tensor estimate, computed with the unscented transform in TT format. The experiment is performed 100 times with the same TT, $\mathbfcal{G}$, but with different priors, as in \cref{sec:convergence}. Then, the mean of the 100 results is plotted with a region of twice the standard deviation. \Cref{fig:Pn} shows the trace and Frobenius norm of the covariance matrix of the core that will be updated next, after the norm is moved to this core. Both the trace and Frobenius norm of each core's covariance matrix decrease and converge to a fixed value. For the first and third core, the values are smaller than for the second, because the second core has a larger number of elements. The convergence behavior is also shown in \cref{fig:P}, where the trace and Frobenius norm of the covariance matrix of the low-rank tensor estimate converge quickly to a fixed value. The decreasing and converging values of the trace (\cref{fig:convergence} top) and the Frobenius norm (\cref{fig:convergence} bottom) indicate that the uncertainty of the mean decreases and then remains constant with an increasing number of iterations. In the next experiments, we will use the information of the covariance matrices to visualize a confidence interval for our estimate. 

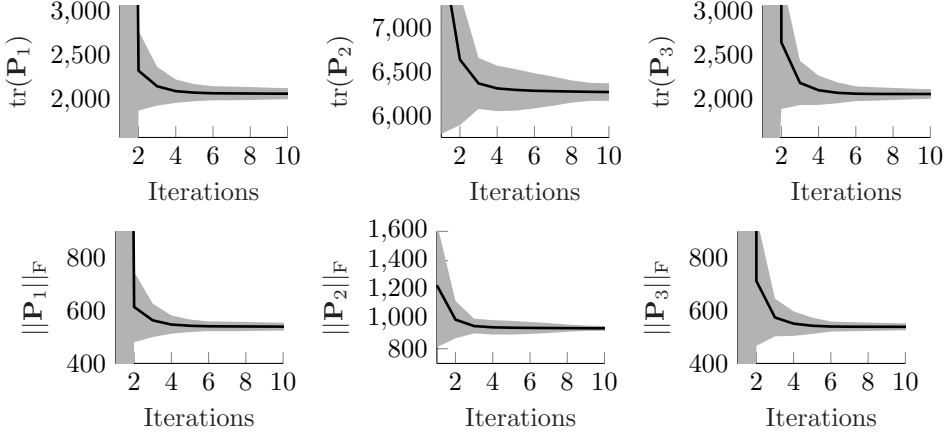
\begin{figure}
    \centering
    \setlength\fwidth{0.18\textwidth}
%
%
\definecolor{mycolor2}{rgb}{0 0 0}%
\definecolor{mycolor1}{rgb}{0.7 0.7 0.7}%
\begin{tikzpicture}

\begin{axis}[%
width=0.951\fwidth,
height=0.75\fwidth,
at={(0\fwidth,0\fwidth)},
scale only axis,
xmin=1,
xmax=10,
xlabel style={font=\color{white!13!black}},
xlabel={Iterations},
ymin=1590.70621648517,
ymax=3090.70621648517,
ylabel style={font=\color{white!13!black}},
ylabel={$\mathrm{tr}(\mathbf{P}_1)$},
axis background/.style={fill=white},
axis x line*=bottom,
axis y line*=left,
legend style={legend cell align=left, align=left, draw=white!15!black}
]

\addplot[area legend, draw=none, fill=mycolor1]
table[row sep=crcr] {%
x	y\\
1	76798.8863935678\\
2	2802.19392510574\\
3	2392.41397190747\\
4	2249.10821675521\\
5	2196.58846645939\\
6	2172.48563418929\\
7	2168.67452772808\\
8	2163.21579070996\\
9	2155.97095090073\\
10	2150.39248563863\\
11	2146.30689723514\\
12	2144.58547598525\\
13	2144.72694294915\\
14	2142.34375741339\\
15	2140.82278357662\\
16	2138.50822169083\\
17	2139.20385315164\\
18	2137.53973040452\\
19	2138.68700988403\\
20	2137.63066866389\\
20	2029.65370257834\\
19	2028.79549624577\\
18	2030.63992521783\\
17	2029.27400461228\\
16	2030.83776392434\\
15	2029.2371031735\\
14	2029.19645952335\\
13	2027.95945114683\\
12	2028.73557620191\\
11	2027.7851018805\\
10	2024.9308901066\\
9	2021.08061029756\\
8	2016.12933541114\\
7	2013.00891470231\\
6	2012.34940582844\\
5	2002.19061417495\\
4	1983.71009151126\\
3	1951.74206416024\\
2	1896.10788949522\\
1	-23172.5015280163\\
}--cycle;

\addplot [color=mycolor2, line width=1.0pt]
  table[row sep=crcr]{%
1.96355649951556	3240.70621648517\\
2	2349.15090730048\\
3	2172.07801803386\\
4	2116.40915413323\\
5	2099.38954031717\\
6	2092.41752000886\\
7	2090.8417212152\\
8	2089.67256306055\\
9	2088.52578059915\\
10	2087.66168787262\\
11	2087.04599955782\\
12	2086.66052609358\\
13	2086.34319704799\\
14	2085.77010846837\\
15	2085.02994337506\\
16	2084.67299280758\\
17	2084.23892888196\\
18	2084.08982781118\\
19	2083.7412530649\\
20	2083.64218562111\\
};

\end{axis}
\end{tikzpicture}%
%
%
\definecolor{mycolor2}{rgb}{0 0 0}%
\definecolor{mycolor1}{rgb}{0.7 0.7 0.7}%
\begin{tikzpicture}

\begin{axis}[%
width=0.951\fwidth,
height=0.75\fwidth,
at={(0\fwidth,0\fwidth)},
scale only axis,
xmin=1,
xmax=10,
xlabel style={font=\color{white!13!black}},
xlabel={Iterations},
ymin=5783.17260711981,
ymax=7283.17260711981,
ylabel style={font=\color{white!13!black}},
ylabel={$\mathrm{tr}(\mathbf{P}_2)$},
axis background/.style={fill=white},
axis x line*=bottom,
axis y line*=left,
legend style={legend cell align=left, align=left, draw=white!15!black}
]

\addplot[area legend, draw=none, fill=mycolor1]
table[row sep=crcr] {%
x	y\\
1	10087.9219527652\\
2	7412.19894375499\\
3	6686.81698488517\\
4	6598.8572853575\\
5	6558.81142854173\\
6	6513.25733745271\\
7	6475.02729476663\\
8	6431.41790438586\\
9	6403.67544522148\\
10	6398.65556191281\\
11	6401.85418278003\\
12	6404.58620397328\\
13	6405.29304464641\\
14	6407.15604282258\\
15	6403.91024689278\\
16	6405.76572121994\\
17	6402.11962369008\\
18	6404.26405043244\\
19	6400.52364832772\\
20	6399.97754491833\\
20	6190.8674557443\\
19	6190.41200455714\\
18	6186.16028985751\\
17	6189.11072430303\\
16	6185.06404348093\\
15	6187.8499488142\\
14	6184.74327812357\\
13	6188.10110967026\\
12	6189.97833350976\\
11	6194.26855586356\\
10	6200.03913779564\\
9	6199.48906096311\\
8	6178.80200623961\\
7	6143.07552820965\\
6	6113.52538840666\\
5	6087.73146835306\\
4	6081.99080944117\\
3	6107.44330279721\\
2	5923.79223345328\\
1	5822.32684549314\\
}--cycle;

\addplot [color=mycolor2, line width=1.0pt]
  table[row sep=crcr]{%
1.40551636148757	7433.17260711981\\
2	6667.99558860414\\
3	6397.13014384119\\
4	6340.42404739934\\
5	6323.27144844739\\
6	6313.39136292969\\
7	6309.05141148814\\
8	6305.10995531273\\
9	6301.58225309229\\
10	6299.34734985423\\
11	6298.0613693218\\
12	6297.28226874152\\
13	6296.69707715833\\
14	6295.94966047308\\
15	6295.88009785349\\
16	6295.41488235044\\
17	6295.61517399656\\
18	6295.21217014498\\
19	6295.46782644243\\
20	6295.42250033131\\
};

\end{axis}
\end{tikzpicture}%
%
%
\definecolor{mycolor2}{rgb}{0 0 0}%
\definecolor{mycolor1}{rgb}{0.7 0.7 0.7}%
\begin{tikzpicture}

\begin{axis}[%
width=0.951\fwidth,
height=0.75\fwidth,
at={(0\fwidth,0\fwidth)},
scale only axis,
xmin=1,
xmax=10,
xlabel style={font=\color{white!13!black}},
xlabel={Iterations},
ymin=1591.97289046178,
ymax=3091.97289046178,
ylabel style={font=\color{white!13!black}},
ylabel={$\mathrm{tr}(\mathbf{P}_3)$},
axis background/.style={fill=white},
axis x line*=bottom,
axis y line*=left,
legend style={legend cell align=left, align=left, draw=white!15!black}
]

\addplot[area legend, draw=none, fill=mycolor1]
table[row sep=crcr] {%
x	y\\
1	10000\\
2	3429.78039075948\\
3	2463.81810555114\\
4	2294.75108974286\\
5	2218.19610945778\\
6	2173.00722944592\\
7	2163.15581635104\\
8	2154.76260652751\\
9	2146.39677715455\\
10	2140.15707517753\\
11	2136.43464725\\
12	2134.85600909592\\
13	2134.73732275473\\
14	2132.36606264616\\
15	2131.92521129643\\
16	2129.64049179772\\
17	2130.82518472201\\
18	2128.96274748336\\
19	2130.47410504258\\
20	2129.96895651433\\
20	2037.27814444784\\
19	2036.91699197054\\
18	2039.19254113396\\
17	2037.1900364113\\
16	2039.24736941047\\
15	2037.16107063453\\
14	2038.0782037884\\
13	2036.04737106524\\
12	2036.09441714344\\
11	2034.62561371528\\
10	2031.08087856724\\
9	2025.12159571934\\
8	2017.22871502576\\
7	2009.93145040341\\
6	2004.94486880647\\
5	1980.49965098294\\
4	1961.78472416061\\
3	1960.14551709733\\
2	1915.29032530212\\
1	-3000\\
}--cycle;

\addplot [color=mycolor2, line width=1.0pt]
  table[row sep=crcr]{%
1.99999635461427	3241.97289046645\\
2	2672.5353580308\\
3	2211.98181132423\\
4	2128.26790695173\\
5	2099.34788022036\\
6	2088.9760491262\\
7	2086.54363337723\\
8	2085.99566077664\\
9	2085.75918643694\\
10	2085.61897687238\\
11	2085.53013048264\\
12	2085.47521311968\\
13	2085.39234690999\\
14	2085.22213321728\\
15	2084.54314096548\\
16	2084.4439306041\\
17	2084.00761056665\\
18	2084.07764430866\\
19	2083.69554850656\\
20	2083.62355048108\\
};

\end{axis}
\end{tikzpicture}%
%
%
\definecolor{mycolor2}{rgb}{0 0 0}%
\definecolor{mycolor1}{rgb}{0.7 0.7 0.7}%
\begin{tikzpicture}

\begin{axis}[%
width=0.951\fwidth,
height=0.75\fwidth,
at={(0\fwidth,0\fwidth)},
scale only axis,
xmin=1,
xmax=10,
xlabel style={font=\color{white!13!black}},
xlabel={Iterations},
ymin=400,
ymax=900,
ylabel style={font=\color{white!13!black}},
ylabel={$||\mathbf{P}_1||_\mathrm{F}$},
axis background/.style={fill=white},
axis x line*=bottom,
axis y line*=left,
legend style={legend cell align=left, align=left, draw=white!15!black}
]

\addplot[area legend, draw=none, fill=mycolor1]
table[row sep=crcr] {%
x	y\\
1	31360.5169263042\\
2	747.09388056007\\
3	626.227359151724\\
4	581.620931236552\\
5	566.199816520752\\
6	558.895702591155\\
7	558.142698638076\\
8	557.016859915676\\
9	555.346896162247\\
10	554.107678210794\\
11	553.160995420822\\
12	552.796079039127\\
13	552.958470951754\\
14	552.330754227146\\
15	551.764251601093\\
16	551.215708524746\\
17	551.274132046157\\
18	550.935609880855\\
19	551.105296226113\\
20	550.88748696349\\
20	525.332824560325\\
19	525.176204576116\\
18	525.506720146787\\
17	525.264278767709\\
16	525.531264310241\\
15	525.18824295966\\
14	525.00372903246\\
13	524.687212333756\\
12	525.011885517074\\
11	524.863318896185\\
10	524.272583252051\\
9	523.533726483725\\
8	522.52792679715\\
7	522.06207729749\\
6	522.173851868168\\
5	518.695431118133\\
4	512.53494775216\\
3	499.641130846195\\
2	480.352037591919\\
1	-11617.0674720107\\
}--cycle;

\addplot [color=mycolor2, line width=1.0pt]
  table[row sep=crcr]{%
1.8837657346603	1689.81999310138\\
2	613.722959075995\\
3	562.93424499896\\
4	547.077939494356\\
5	542.447623819442\\
6	540.534777229662\\
7	540.102387967783\\
9	539.440311322986\\
10	539.190130731423\\
11	539.012157158503\\
12	538.903982278101\\
13	538.822841642755\\
14	538.667241629803\\
15	538.476247280376\\
17	538.269205406933\\
18	538.221165013822\\
19	538.140750401114\\
20	538.110155761908\\
};

\end{axis}
\end{tikzpicture}%
%
%
\definecolor{mycolor2}{rgb}{0 0 0}%
\definecolor{mycolor1}{rgb}{0.7 0.7 0.7}%
\begin{tikzpicture}

\begin{axis}[%
width=0.951\fwidth,
height=0.75\fwidth,
at={(0\fwidth,0\fwidth)},
scale only axis,
xmin=1,
xmax=10,
xlabel style={font=\color{white!13!black}},
xlabel={Iterations},
ymin=700,
ymax=1600,
ylabel style={font=\color{white!13!black}},
ylabel={$||\mathbf{P}_2||_\mathrm{F}$},
axis background/.style={fill=white},
axis x line*=bottom,
axis y line*=left,
legend style={legend cell align=left, align=left, draw=white!15!black}
]

\addplot[area legend, draw=none, fill=mycolor1]
table[row sep=crcr] {%
x	y\\
1	1654.15801016324\\
2	1125.25442572812\\
3	1004.80683694189\\
4	995.141301617852\\
5	990.090874390746\\
6	982.815783646402\\
7	974.738790483238\\
8	965.223324339939\\
9	958.338817569146\\
10	955.699428867709\\
11	955.219034194049\\
12	955.190323145843\\
13	955.092431577418\\
14	955.122470842804\\
15	954.765233367831\\
16	954.853361444929\\
17	954.499058929109\\
18	954.62356888353\\
19	954.280591546826\\
20	954.202283551178\\
20	923.364572995825\\
19	923.300019503839\\
18	922.873644169245\\
17	923.126398465392\\
16	922.706215584054\\
15	922.941195230844\\
14	922.601020266737\\
13	922.86540952514\\
12	922.95099303892\\
11	923.171825143061\\
10	923.116139217032\\
9	921.241699316786\\
8	915.601850629368\\
7	907.521502825066\\
6	900.967112887531\\
5	896.861322753453\\
4	897.054086584954\\
3	904.541757883045\\
2	871.453473890369\\
1	810.491510889114\\
}--cycle;

\addplot [color=mycolor2, line width=1.0pt]
  table[row sep=crcr]{%
1	1232.32476052618\\
2	998.353949809246\\
3	954.674297412468\\
4	946.097694101403\\
5	943.4760985721\\
6	941.891448266967\\
7	941.130146654152\\
8	940.412587484654\\
9	939.790258442966\\
10	939.40778404237\\
11	939.195429668555\\
12	939.070658092381\\
13	938.978920551279\\
14	938.861745554771\\
15	938.853214299337\\
16	938.779788514491\\
17	938.812728697251\\
18	938.748606526387\\
19	938.790305525333\\
20	938.783428273501\\
};

\end{axis}
\end{tikzpicture}%
%
%
\definecolor{mycolor2}{rgb}{0 0 0}%
\definecolor{mycolor1}{rgb}{0.7 0.7 0.7}%
\begin{tikzpicture}

\begin{axis}[%
width=0.951\fwidth,
height=0.75\fwidth,
at={(0\fwidth,0\fwidth)},
scale only axis,
xmin=1,
xmax=10,
xlabel style={font=\color{white!13!black}},
xlabel={Iterations},
ymin=400,
ymax=900,
ylabel style={font=\color{white!13!black}},
ylabel={$||\mathbf{P}_3||_\mathrm{F}$},
axis background/.style={fill=white},
axis x line*=bottom,
axis y line*=left,
legend style={legend cell align=left, align=left, draw=white!15!black}
]

\addplot[area legend, draw=none, fill=mycolor1]
table[row sep=crcr] {%
x	y\\
1	17700.168957\\
2	957.0284467497\\
3	645.045102124702\\
4	596.947969514599\\
5	573.450395708042\\
6	558.600961883994\\
7	556.328003794033\\
8	554.701554713727\\
9	553.0364488109\\
10	551.71505603283\\
11	550.87077034611\\
12	550.515009098917\\
13	550.544790632084\\
14	549.95808031886\\
15	549.781082273484\\
16	549.214210643099\\
17	549.48838990772\\
18	549.044681499803\\
19	549.399449297175\\
20	549.278422432715\\
20	526.856938040262\\
19	526.775122296917\\
18	527.313676742571\\
17	526.846976457114\\
16	527.334973143791\\
15	526.834497174163\\
14	527.00350381309\\
13	526.516136834223\\
12	526.588346409571\\
11	526.273263416557\\
10	525.502737481507\\
9	524.304789375108\\
8	522.841065299734\\
7	521.586415240738\\
6	520.663607646951\\
5	511.627572605496\\
4	504.269155465192\\
3	502.597307757744\\
2	466.76948787526\\
1	-6020.7814972\\
}--cycle;

\addplot [color=mycolor2, line width=1.0pt]
  table[row sep=crcr]{%
1.99998325174465	1690.14548817277\\
2	711.89896731248\\
3	573.821204941223\\
4	550.608562489895\\
5	542.538984156769\\
6	539.632284765472\\
7	538.957209517386\\
8	538.771310006731\\
9	538.670619093004\\
10	538.608896757168\\
11	538.572016881334\\
13	538.530463733154\\
14	538.480792065975\\
15	538.307789723824\\
16	538.274591893445\\
17	538.167683182417\\
18	538.179179121187\\
19	538.087285797046\\
20	538.067680236489\\
};

\end{axis}
\end{tikzpicture}%
    \caption{Top: Trace, bottom: Frobenius norm of covariance matrix of $\Gt_1$, $\Gt_2$ and $\Gt_3$.}
    \label{fig:Pn}
\end{figure}

\begin{figure}
    \centering
%
%
\definecolor{mycolor2}{rgb}{0 0 0}%
\definecolor{mycolor1}{rgb}{0.7 0.7 0.7}%
\begin{tikzpicture}

\begin{axis}[%
width=1.2in,
height=1.2in,
at={(0.758in,0.481in)},
scale only axis,
xmin=1,
xmax=5,
xlabel style={font=\color{white!15!black}},
xlabel={Number of iterations},
ymode=log,
ymin=0.01,
ymax=1000000,
yminorticks=true,
ylabel style={font=\color{white!15!black}},
ylabel={$\mathrm{tr}(\mathbf{P}_\mathrm{UT})$},
axis background/.style={fill=white},
legend style={legend cell align=left, align=left, draw=white!15!black}
]

\addplot[area legend, draw=none, fill=mycolor1]
table[row sep=crcr] {%
x	y\\
1	162275.388160692\\
2	11533.1674974475\\
3	10765.8449600642\\
4	10561.8241394912\\
5	10519.8673231031\\
5	10418.6485494752\\
4	10411.1393942641\\
3	10330.301303268\\
2	10241.0662727897\\
1	0.01\\
}--cycle;

\addplot [color=mycolor2, line width=1.0pt]
  table[row sep=crcr]{%
1	63134.0509357772\\
2	10887.1168851186\\
3	10548.0731316661\\
4	10486.4817668776\\
5	10469.2579362892\\
};

\end{axis}

\begin{axis}[%
width=1.2in,
height=1.2in,
at={(3.327in,0.481in)},
scale only axis,
xmin=1,
xmax=5,
xlabel style={font=\color{white!15!black}},
xlabel={Number of iterations},
ymode=log,
ymin=0.01,
ymax=100000,
yminorticks=true,
ylabel style={font=\color{white!15!black}},
ylabel={$||\mathbf{P}_\mathrm{UT}||_\mathrm{F}$},
axis background/.style={fill=white},
legend style={legend cell align=left, align=left, draw=white!15!black}
]

\addplot[area legend, draw=none, fill=mycolor1]
table[row sep=crcr] {%
x	y\\
1	60807.9040626897\\
2	1712.53421784249\\
3	1557.46644088027\\
4	1521.5108444637\\
5	1513.12241680125\\
5	1496.63988531634\\
4	1494.50494518111\\
3	1480.79154538019\\
2	1456.71404336698\\
1	0.01\\
}--cycle;

\addplot [color=mycolor2, line width=1.0pt]
  table[row sep=crcr]{%
1	21168.6542103251\\
2	1584.62413060473\\
3	1519.12899313023\\
4	1508.00789482241\\
5	1504.88115105879\\
};

\end{axis}
\end{tikzpicture}%
    \caption{Left: Trace, right: Frobenius norm of covariance matrix of low-rank tensor estimate.}
    \label{fig:P}
\end{figure}
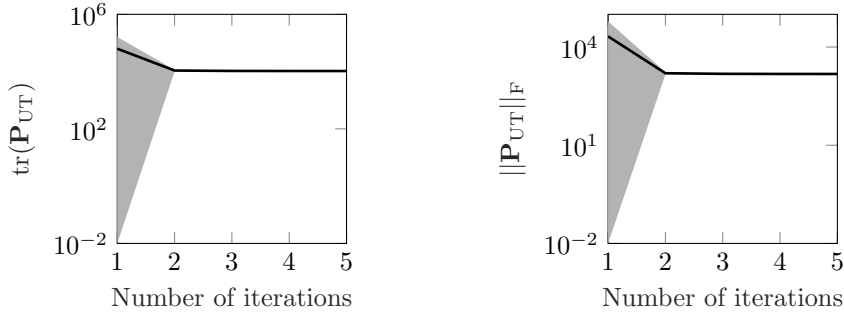


\subsection{Comparison to conventional ALS}
The main benefits of the ALS in a Bayesian framework are the uncertainty quantification of the low-rank tensor estimate, as well as the incorporation of prior knowledge. In the third experiment, we show the benefits by comparing the ALS in a Bayesian framework to the conventional ALS. \\

\Cref{fig:signal} depicts the vectorized low-rank tensor estimate for the ALS and the mean of the ALS in a Bayesian framework's estimate with a $95\%$ confidence interval in comparison with the ground truth. The uncertainty measure is computed from the diagonal elements of $\cov_\mathrm{UT}$. The top figure shows the estimate using one noisy sample $\y$ for the ALS in a Bayesian framework and the bottom using 100 noisy samples. While the ALS does not improve when taking into account multiple noisy samples, the ALS in a Bayesian framework improves in two aspects. The error between the mean and the truth becomes smaller and the estimate becomes more certain. Also, in the top figure, where the uncertainty is relatively large, the ground truth almost always lies inside the confidence interval and therefore the ALS in a Bayesian framework provides more information than the conventional ALS.\\

\begin{figure}
\centering
%
%
\definecolor{mycolor1}{rgb}{0.60784,0.74902,0.73725}%
\definecolor{mycolor2}{rgb}{0.97255,0.23529,0.36471}%
\definecolor{mycolor3}{rgb}{0.96863,0.70196,0.20000}%
\definecolor{mycolor4}{rgb}{0.01176,0.50196,0.58431}%
\begin{tikzpicture}

\begin{axis}[%
width=4in,
height=1.3in,
at={(0.758in,0.481in)},
scale only axis,
xmin=0,
xmax=70,
ymin=-10,
ymax=10,
axis background/.style={fill=white},
axis x line*=bottom,
axis y line*=left,
legend style={at={(0.03,0.03)}, anchor=south west, legend cell align=left, align=left, draw=white!15!black}
]

\addplot[area legend, draw=none, fill=mycolor1, forget plot]
table[row sep=crcr] {%
x	y\\
1	0.9788640380011\\
2	4.81848970570207\\
3	0.750889147677242\\
4	3.00028725545229\\
5	3.30792374983945\\
6	0.562826775999863\\
7	3.82246860736621\\
8	5.33662730180692\\
9	3.19092442244756\\
10	1.10017513533643\\
11	3.75013330652833\\
12	2.82977433747253\\
13	1.40844390403487\\
14	2.1716966256177\\
15	1.49458645813979\\
16	3.18571862809223\\
17	1.25338516647826\\
18	3.86266497391371\\
19	1.14661597813633\\
20	1.46402179645402\\
21	1.19443492234227\\
22	0.751973476043353\\
23	1.33907374259747\\
24	4.97848483030973\\
25	3.73585807585891\\
26	1.14241434707273\\
27	4.32929718599835\\
28	2.1358026369648\\
29	2.85496666933278\\
30	1.42698827277763\\
31	3.15171382460281\\
32	0.882934080340183\\
33	2.79816817346404\\
34	4.16783424239412\\
35	2.97768443855743\\
36	1.97616155978761\\
37	-0.00540675176692718\\
38	3.17651756739171\\
39	-0.241244526830279\\
40	8.77190826003315\\
41	7.68346991409798\\
42	3.16398246130708\\
43	8.62905573997013\\
44	5.38884389369882\\
45	9.00967603176845\\
46	4.15727815792847\\
47	9.87869131722827\\
48	0.138929071393746\\
49	1.66689381309392\\
50	4.45509895641493\\
51	1.55885588535163\\
52	1.39555576274365\\
53	0.248546169551324\\
54	1.74860225083604\\
55	0.140057681863178\\
56	7.56780095905406\\
57	6.36777372829957\\
58	2.03232706170777\\
59	7.20361250359473\\
60	3.89859983186929\\
61	6.47642349917718\\
62	3.06483465964838\\
63	7.07935752128378\\
64	-0.326352502044081\\
64	-8.04007049452089\\
63	0.235357629353636\\
62	-2.28480968063517\\
61	0.0260349317653232\\
60	-2.34761377819763\\
59	0.207101948747641\\
58	-3.82918636350263\\
57	-0.113075859678178\\
56	0.262405479612118\\
55	-5.9127467018523\\
54	-3.67304496014924\\
53	-5.49655185063285\\
52	-4.65854191719878\\
51	-4.80402043523916\\
50	-1.07463218513509\\
49	-4.17324380313482\\
48	-9.88167655315749\\
47	1.00895065426568\\
46	-3.28861303158893\\
45	0.609912265273289\\
44	-2.98827480741026\\
43	0.200950003471073\\
42	-4.00747360607076\\
41	-0.139558367737359\\
40	-0.158378015220241\\
39	-7.5855296415496\\
38	-3.28002775886492\\
37	-6.89197251092226\\
36	-5.16616898242097\\
35	-3.78750994288425\\
34	-1.38646342487556\\
33	-3.2438034483747\\
32	-4.85159847043081\\
31	-1.66540589402107\\
30	-1.15560215003189\\
29	-1.49432282052292\\
28	-1.71267604211958\\
27	-0.936408196891024\\
26	-2.70733937633515\\
25	-0.924871818752896\\
24	-0.899336764237121\\
23	-3.22133348133576\\
22	-3.15211444337872\\
21	-3.06132857381712\\
20	-3.68198950228737\\
19	-4.62368888331346\\
18	-1.14482146966043\\
17	-4.02505485237748\\
16	-4.72097333728071\\
15	-4.90890226642023\\
14	-2.05255352841892\\
13	-4.33299245109834\\
12	-3.59297591707926\\
11	-2.75660912504902\\
10	-3.67529743105817\\
9	-2.5169818428626\\
8	-2.9318453018602\\
7	-2.71857700862563\\
6	-4.96929368499301\\
5	-2.72094293712192\\
4	-4.9486218279776\\
3	-6.87598784926068\\
2	-1.76347395454177\\
1	-6.0198865973878\\
}--cycle;
\addplot [color=mycolor2, line width=1.0pt, mark=o, mark options={solid, mycolor2}]
  table[row sep=crcr]{%
1	-2.23132257355782\\
2	1.19174142356624\\
3	-2.06081266125176\\
4	-1.52429465388157\\
5	1.21564361529743\\
6	-2.0380990079169\\
7	1.64079569226637\\
8	1.36133239616704\\
9	-0.691063994276686\\
10	-1.86070901916833\\
11	0.0226990655866579\\
12	0.769773135769144\\
13	-1.88574094866571\\
14	1.18273129087673\\
15	-2.11989324557536\\
16	-0.641992152799128\\
17	-3.56964680685343\\
18	2.46721834155735\\
19	-3.79861114756737\\
20	-1.98552646851621\\
21	-0.498179150521262\\
22	-3.72965654251128\\
23	0.171210889887878\\
24	3.1464877404155\\
25	1.44546023226686\\
26	-2.84952326489917\\
27	2.26816236577447\\
28	1.42070684918466\\
29	1.35874153716065\\
30	0.367354793119547\\
31	1.93034012900507\\
32	-1.95997779663879\\
33	-0.704316542587256\\
34	3.99266264834337\\
35	-1.75323255149884\\
36	-2.42518609067285\\
37	-3.57688937160771\\
38	-1.31391118878023\\
39	-3.89936284983234\\
40	2.89090044173383\\
41	4.61025771616246\\
42	-0.273540076446117\\
43	3.62123956808239\\
44	1.90286895450925\\
45	4.98308850107258\\
46	-1.59673298636703\\
47	6.42517551113512\\
48	-2.06612707018935\\
49	-2.20567507022799\\
50	3.22750777906316\\
51	-2.91358059538016\\
52	-2.03478404700769\\
53	-2.433561161181\\
54	-2.6723718755664\\
55	-2.22643912419982\\
56	3.18686744792909\\
57	3.42375400235467\\
58	-1.60174802867315\\
59	3.34161446367326\\
60	1.64979275719611\\
61	3.96678801690217\\
62	-0.874511841840128\\
63	5.13522643293798\\
64	-2.18037713348578\\
};

\addplot [color=mycolor3, line width=1.0pt]
  table[row sep=crcr]{%
1	-3.92625442305173\\
2	4.58413549891743\\
3	-3.82387275370818\\
4	-4.25302285787598\\
5	-0.992055695312937\\
6	-2.0330403769656\\
7	1.16198826413729\\
8	2.5598267654344\\
9	-1.19578081491336\\
10	-0.137830762257944\\
11	-2.06383082983868\\
12	2.95558899961063\\
13	-5.81368330237354\\
14	2.51262608774712\\
15	-3.37289951662496\\
16	-0.696917386146849\\
17	-8.58772301514086\\
18	-0.989928413700006\\
19	-0.967272997128159\\
20	3.17876943978\\
21	-1.19034574503656\\
22	-5.08666972950711\\
23	2.78912586033481\\
24	6.56893658602984\\
25	-3.62059342690334\\
26	-4.15147070555584\\
27	2.01917448059508\\
28	5.67506052527377\\
29	-2.14166409182206\\
30	-1.81965394891446\\
31	4.41776818843056\\
32	-2.1188516392601\\
33	1.09947623871688\\
34	1.55648633846973\\
35	-1.56495737607638\\
36	-1.07734456109956\\
37	-4.76834528464546\\
38	1.68028177059894\\
39	-4.21118744464952\\
40	2.0743797078277\\
41	4.88160736040086\\
42	-2.51102657019507\\
43	2.07738697396996\\
44	2.37385569825635\\
45	2.61002805944267\\
46	-2.13450625539552\\
47	4.27163125747154\\
48	-1.25847497689877\\
49	-0.289086720160388\\
50	1.75908186948132\\
51	-2.17136153648917\\
52	-0.705219766916973\\
53	-6.35127512637931\\
54	1.14691127229196\\
55	-4.85126880170612\\
56	3.96630086257654\\
57	5.54412246890372\\
58	-4.05918365623781\\
59	3.09089075943561\\
60	4.16689300003159\\
61	2.96630907778702\\
62	-3.127708818472\\
63	6.40196306720396\\
64	-2.03680237878013\\
};

\addplot [color=mycolor4, line width=1.0pt]
  table[row sep=crcr]{%
1	-2.52051127969335\\
2	1.52750787558015\\
3	-3.06254935079172\\
4	-0.974167286262653\\
5	0.29349040635876\\
6	-2.20323345449657\\
7	0.551945799370288\\
8	1.20239099997336\\
9	0.336971289792481\\
10	-1.28756114786087\\
11	0.496762090739653\\
12	-0.381600789803363\\
13	-1.46227427353173\\
14	0.0595715485993899\\
15	-1.70715790414022\\
16	-0.76762735459424\\
17	-1.38583484294961\\
18	1.35892175212664\\
19	-1.73853645258856\\
20	-1.10898385291667\\
21	-0.933446825737423\\
22	-1.20007048366768\\
23	-0.941129869369145\\
24	2.0395740330363\\
25	1.40549312855301\\
26	-0.782462514631211\\
27	1.69644449455366\\
28	0.211563297422612\\
29	0.680321924404933\\
30	0.135693061372869\\
31	0.74315396529087\\
32	-1.98433219504531\\
33	-0.222817637455333\\
34	1.39068540875928\\
35	-0.404912752163405\\
36	-1.59500371131668\\
37	-3.44868963134459\\
38	-0.0517550957366087\\
39	-3.91338708418994\\
40	4.30676512240646\\
41	3.77195577318031\\
42	-0.421745572381837\\
43	4.4150028717206\\
44	1.20028454314428\\
45	4.80979414852087\\
46	0.43433256316977\\
47	5.44382098574698\\
48	-4.87137374088187\\
49	-1.25317499502045\\
50	1.69023338563992\\
51	-1.62258227494376\\
52	-1.63149307722756\\
53	-2.62400284054076\\
54	-0.9622213546566\\
55	-2.88634450999456\\
56	3.91510321933309\\
57	3.1273489343107\\
58	-0.898429650897427\\
59	3.70535722617118\\
60	0.775493026835829\\
61	3.25122921547125\\
62	0.390012489506603\\
63	3.65735757531871\\
64	-4.18321149828249\\
};

\end{axis}
\end{tikzpicture}%
%
%
\definecolor{mycolor1}{rgb}{0.60784,0.74902,0.73725}%
\definecolor{mycolor2}{rgb}{0.97255,0.23529,0.36471}%
\definecolor{mycolor3}{rgb}{0.96863,0.70196,0.20000}%
\definecolor{mycolor4}{rgb}{0.01176,0.50196,0.58431}%
\begin{tikzpicture}

\begin{axis}[%
width=4in,
height=1.3in,
at={(0.758in,0.481in)},
scale only axis,
xmin=0,
xmax=70,
ymin=-10,
ymax=8,
axis background/.style={fill=white},
axis x line*=bottom,
axis y line*=left,
legend style={at={(.4,-.6)}, anchor=south west, legend cell align=left, align=left, draw=white!15!black}
]

\addplot[area legend, draw=none, fill=mycolor1, forget plot]
table[row sep=crcr] {%
x	y\\
1	-1.11275346072123\\
2	2.7647930714467\\
3	-1.54892904138569\\
4	-0.0127896858063741\\
5	1.66842758132793\\
6	-1.46632042921788\\
7	2.26502243033061\\
8	2.01920747489087\\
9	0.404881806822531\\
10	-0.637411250186932\\
11	0.672029943726255\\
12	1.35182931612801\\
13	-1.0580143291618\\
14	1.20848346936862\\
15	-1.21003023804617\\
16	0.457873223394149\\
17	-1.23329989977932\\
18	3.33586626470459\\
19	-1.74439186228302\\
20	-0.558615838155751\\
21	-0.0561140186420788\\
22	-1.20814155341508\\
23	0.232804704558252\\
24	3.2931108986195\\
25	2.72706174250072\\
26	-0.724665053859724\\
27	3.16957535734444\\
28	2.01787259563934\\
29	2.13179352189213\\
30	1.45659181754138\\
31	2.31371951751727\\
32	-0.807174215481155\\
33	0.527313585071919\\
34	3.93602416309817\\
35	0.083214474663915\\
36	-0.795968122551079\\
37	-2.42465013577272\\
38	0.866037447892547\\
39	-2.83918939913292\\
40	4.0984610040981\\
41	5.93130065280092\\
42	1.06251272114238\\
43	6.2576308783139\\
44	3.11251601856371\\
45	6.62675240031057\\
46	1.68072452820315\\
47	7.28851328185777\\
48	-1.61247283023162\\
49	-1.33017862090764\\
50	4.18995269922197\\
51	-1.98477778507207\\
52	-1.47190202209441\\
53	-2.09535763735886\\
54	-1.07440218440589\\
55	-2.14721804879775\\
56	4.50903545223182\\
57	5.18965832667059\\
58	-0.71294488144403\\
59	5.75984873434707\\
60	2.7806402998868\\
61	5.42587705228283\\
62	1.73567278908299\\
63	5.95498542730421\\
64	-2.12247383877532\\
64	-4.18095073547329\\
63	3.65820125354987\\
62	-0.319084296198586\\
61	3.17797525904611\\
60	0.871405679304817\\
59	3.9589926397549\\
58	-2.91889725472529\\
57	3.31314969867954\\
56	2.44386790853383\\
55	-4.05864707496161\\
54	-3.33820683890014\\
53	-3.89543315490969\\
52	-3.6597266517371\\
51	-3.96238047408464\\
50	1.68970655030764\\
49	-3.40769626661883\\
48	-4.15719550123173\\
47	4.62935733507603\\
46	-0.87783153434038\\
45	4.01474670573149\\
44	0.519984196115302\\
43	3.83343337034529\\
42	-1.77088919081133\\
41	3.45570900974343\\
40	1.7466833817008\\
39	-4.95510726006581\\
38	-1.67888860663655\\
37	-4.43398779006708\\
36	-3.3386185649616\\
35	-2.18548857511011\\
34	1.10483685472365\\
33	-1.82338826089704\\
32	-2.74753187323568\\
31	0.155516177251462\\
30	-0.473407861338891\\
29	0.0491869071075435\\
28	-0.0720367511256276\\
27	1.21494231610803\\
26	-3.08347439109692\\
25	0.726972742943471\\
24	0.974899223076087\\
23	-1.91608168664959\\
22	-3.69997045562073\\
21	-2.08243763347548\\
20	-2.96450791600044\\
19	-3.92787895567782\\
18	0.647338366340515\\
17	-3.48593097976902\\
16	-1.59567412172407\\
15	-3.40682210844437\\
14	-0.858428792995479\\
13	-3.19279538746327\\
12	-0.76575627061469\\
11	-1.23925809591047\\
10	-3.04065651650582\\
9	-1.56610135945272\\
8	-0.171428403147798\\
7	0.326343265314588\\
6	-3.85988474558249\\
5	-0.14781539956223\\
4	-2.37087417857366\\
3	-3.61189915194052\\
2	0.101038817327322\\
1	-3.26090698491581\\
}--cycle;
\addplot [color=mycolor2, line width=1.0pt, mark=o, mark options={solid, mycolor2}]
  table[row sep=crcr]{%
1	-2.23132257355782\\
2	1.19174142356624\\
3	-2.06081266125176\\
4	-1.52429465388157\\
5	1.21564361529743\\
6	-2.0380990079169\\
7	1.64079569226637\\
8	1.36133239616704\\
9	-0.691063994276686\\
10	-1.86070901916833\\
11	0.0226990655866579\\
12	0.769773135769144\\
13	-1.88574094866571\\
14	1.18273129087673\\
15	-2.11989324557536\\
16	-0.641992152799128\\
17	-3.56964680685343\\
18	2.46721834155735\\
19	-3.79861114756737\\
20	-1.98552646851621\\
21	-0.498179150521262\\
22	-3.72965654251128\\
23	0.171210889887878\\
24	3.1464877404155\\
25	1.44546023226686\\
26	-2.84952326489917\\
27	2.26816236577447\\
28	1.42070684918466\\
29	1.35874153716065\\
30	0.367354793119547\\
31	1.93034012900507\\
32	-1.95997779663879\\
33	-0.704316542587256\\
34	3.99266264834337\\
35	-1.75323255149884\\
36	-2.42518609067285\\
37	-3.57688937160771\\
38	-1.31391118878023\\
39	-3.89936284983234\\
40	2.89090044173383\\
41	4.61025771616246\\
42	-0.273540076446117\\
43	3.62123956808239\\
44	1.90286895450925\\
45	4.98308850107258\\
46	-1.59673298636703\\
47	6.42517551113512\\
48	-2.06612707018935\\
49	-2.20567507022799\\
50	3.22750777906316\\
51	-2.91358059538016\\
52	-2.03478404700769\\
53	-2.433561161181\\
54	-2.6723718755664\\
55	-2.22643912419982\\
56	3.18686744792909\\
57	3.42375400235467\\
58	-1.60174802867315\\
59	3.34161446367326\\
60	1.64979275719611\\
61	3.96678801690217\\
62	-0.874511841840128\\
63	5.13522643293798\\
64	-2.18037713348578\\
};
\addlegendentry{Truth}

\addplot [color=mycolor3, line width=1.0pt]
  table[row sep=crcr]{%
1	-0.278110854729711\\
2	-2.24808729662734\\
3	-0.179977403754196\\
4	1.94998751549571\\
5	-0.267220365416024\\
6	-2.91675195201749\\
7	-0.741655415629391\\
8	3.49501787038707\\
9	-0.656706051843484\\
10	-5.49860654341877\\
11	-0.547403684962554\\
12	4.97935904526014\\
13	-3.04320851779034\\
14	0.0361975829455417\\
15	-2.27847346139821\\
16	-1.4782416189673\\
17	-2.18863694641259\\
18	-1.29533742392378\\
19	-0.927459397097054\\
20	-0.945275821966731\\
21	0.0954095053347401\\
22	-5.30950804375984\\
23	-0.00231639440133331\\
24	5.18603103891372\\
25	2.09616864845522\\
26	-3.02342707219307\\
27	3.18289147267474\\
28	1.28626499254591\\
29	1.42001377725187\\
30	1.676343635599\\
31	1.95628407866262\\
32	-2.33413182525371\\
33	-1.05095958283285\\
34	3.78387254156418\\
35	-1.96458359174278\\
36	-2.20354221932175\\
37	-4.25006069160384\\
38	3.00775076262472\\
39	-8.2250052394565\\
40	3.16667108793674\\
41	5.10545519059017\\
42	1.4664599988476\\
43	0.318961206872195\\
44	6.61312718780052\\
45	5.02655975597285\\
46	0.2616817879742\\
47	2.44553799407216\\
48	4.23573691931315\\
49	-3.83677171992167\\
50	0.082840069041881\\
51	-1.98167816396324\\
52	-3.31739561762379\\
53	-0.778300538437508\\
54	-6.5644103707921\\
55	-1.69357582172271\\
56	7.60991569400557\\
57	5.27148869231498\\
58	-1.56070930763718\\
59	5.82340951612942\\
60	0.990129451014009\\
61	5.49291890355289\\
62	2.90173592460899\\
63	5.29371781545209\\
64	-2.01945390167096\\
};
\addlegendentry{ALS}

\addplot [color=mycolor4, line width=1.0pt]
  table[row sep=crcr]{%
1	-2.18683022281852\\
2	1.43291594438701\\
3	-2.5804140966631\\
4	-1.19183193219002\\
5	0.76030609088285\\
6	-2.66310258740019\\
7	1.2956828478226\\
8	0.923889535871538\\
9	-0.580609776315093\\
10	-1.83903388334638\\
11	-0.283614076092106\\
12	0.293036522756662\\
13	-2.12540485831254\\
14	0.175027338186572\\
15	-2.30842617324527\\
16	-0.56890044916496\\
17	-2.35961543977417\\
18	1.99160231552255\\
19	-2.83613540898042\\
20	-1.7615618770781\\
21	-1.06927582605878\\
22	-2.4540560045179\\
23	-0.841638491045666\\
24	2.13400506084779\\
25	1.72701724272209\\
26	-1.90406972247832\\
27	2.19225883672624\\
28	0.972917922256858\\
29	1.09049021449984\\
30	0.491591978101243\\
31	1.23461784738437\\
32	-1.77735304435842\\
33	-0.648037337912561\\
34	2.52043050891091\\
35	-1.0511370502231\\
36	-2.06729334375634\\
37	-3.4293189629199\\
38	-0.406425579372003\\
39	-3.89714832959936\\
40	2.92257219289945\\
41	4.69350483127217\\
42	-0.354188234834471\\
43	5.0455321243296\\
44	1.81625010733951\\
45	5.32074955302103\\
46	0.401446496931386\\
47	5.9589353084669\\
48	-2.88483416573167\\
49	-2.36893744376323\\
50	2.9398296247648\\
51	-2.97357912957836\\
52	-2.56581433691575\\
53	-2.99539539613428\\
54	-2.20630451165301\\
55	-3.10293256187968\\
56	3.47645168038282\\
57	4.25140401267506\\
58	-1.81592106808466\\
59	4.85942068705098\\
60	1.82602298959581\\
61	4.30192615566447\\
62	0.708294246442202\\
63	4.80659334042704\\
64	-3.1517122871243\\
};
\addlegendentry{\cref{alg:noqr}}

\end{axis}
\end{tikzpicture}%
\caption{Ground truth with ALS estimate and mean of estimate from the ALS in a Bayesian framework with confidence region of $95\%$. Top: Estimate using one noisy sample. Bottom: Estimate using 100 noisy samples.}
\label{fig:signal}
\end{figure}
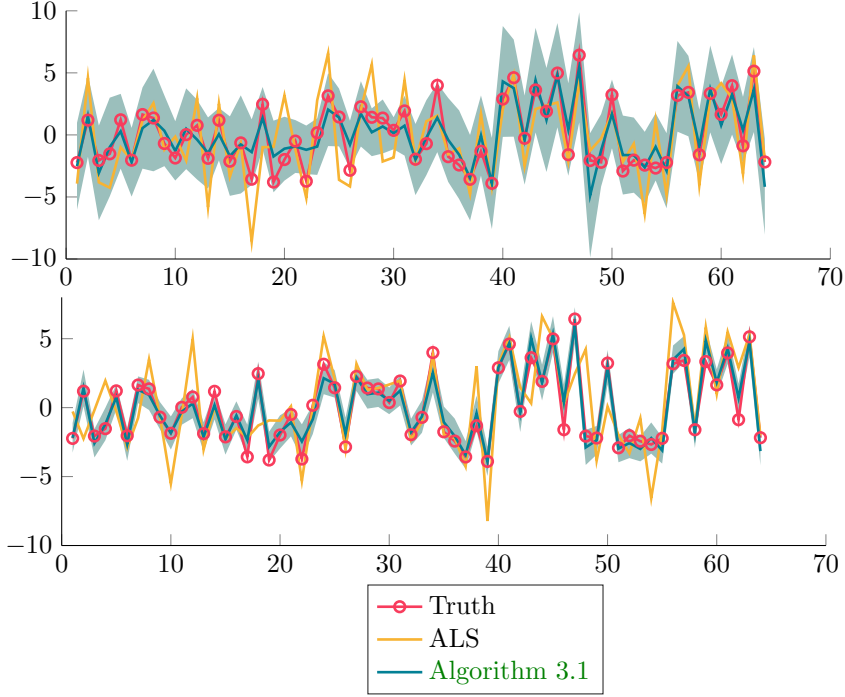

Now, we analyze the influence of the prior quality on the relative error. \Cref{fig:errorcontour} shows the relative error $\varepsilon_\mathrm{truth}$ of the ALS in a Bayesian framework for different priors. The prior mean is computed from
\begin{equation}
    \m_i^0 = \g_{i,\mathrm{truth}}+a\;\mathcal{N}(\mathbf{0},\id), \quad i=1,2,3
    \label{eq:mvar}
\end{equation}
where $\g_{i,\mathrm{truth}}$ denotes the vectorization of $\Gt_{i,\mathrm{truth}}$ and $a$ is a number that is set to values between 0 and 5. It determines how different the prior mean is from the ground truth. The prior covariance is computed from
\begin{equation}
    \cov_i^0 = b^2\;\id, , \quad i=1,2,3 \label{eq:covvar}
\end{equation}
by setting $b$ to values between 0 and 5. A small value means a high certainty and a large value means a low certainty on the prior mean. \Cref{fig:errorcontour} shows that the error is small if the prior mean is close to the ground truth and the covariance is small. For a bad prior and a small covariance, the error is a 100 percent or larger, since a high certainty for a bad prior is assumed. For comparison, the isoline (dashed line) corresponding to the mean relative error of the conventional ALS is shown in the graph, which is almost independent of the prior information. \\

\Cref{fig:errorcomparison} (left) shows the relative error of the reconstructed tensor versus the signal-to-noise ratio for a prior mean from \cref{eq:mvar} with $a=10^{-1}$ and prior covariance from \cref{eq:covvar} with $b = 10^{-1}$, meaning a good prior and a high certainty on the prior. While the ALS performs poorly for high noise, the ALS in a Bayesian framework results in small relative errors. For an increasing SNR, the relative error of the ALS in a Bayesian framework converges to the one of the ALS.\\

Further, \cref{fig:errorcomparison} (right) shows the ALS in comparison with the ALS in a Bayesian framework for multiple noisy samples. While the relative error $\varepsilon_\mathrm{truth}$ decreases for the ALS in a Bayesian framework, the conventional ALS does not improve, when more noisy samples become available.\\

\begin{figure}
    \centering
    \includegraphics[width = \textwidth]{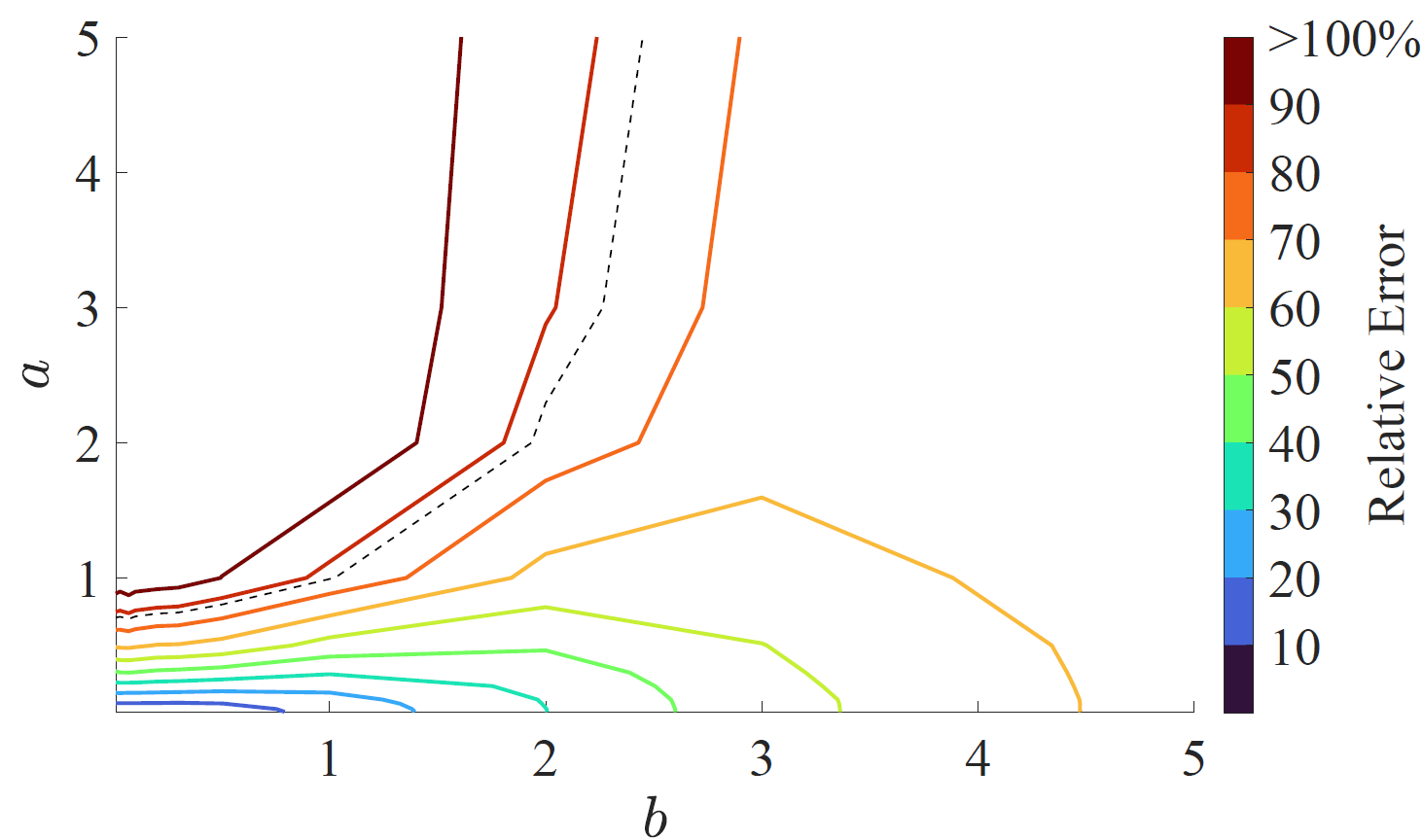}
    \caption{Relative error of the ALS in a Bayesian framework for different priors for $\mathrm{SNR_{dB}} = 0$. The $y$-axis indicates the similarity of the prior mean to the ground truth and the $x$-axis indicates the certainty on the prior mean. The dashed line corresponds to the isoline corresponding to the mean error of the conventional ALS.}
    \label{fig:errorcontour}
\end{figure}

\begin{figure}
    \centering
%
%
\definecolor{mycolor1}{rgb}{0.96471,0.91765,0.77255}%
\definecolor{mycolor2}{rgb}{0.96863,0.70196,0.20000}%
\definecolor{mycolor3}{rgb}{0.60784,0.74902,0.73725}%
\definecolor{mycolor4}{rgb}{0.01176,0.50196,0.58431}%
\begin{tikzpicture}

\begin{axis}[%
width=1.5in,
height=1.5in,
at={(0.758in,0.481in)},
scale only axis,
xmin=0,
xmax=25,
xlabel style={font=\color{white!15!black}},
xlabel={SNR [dB]},
ymode=log,
ymin=0.0362079046811488,
ymax=1,
yminorticks=true,
ylabel style={font=\color{white!15!black}},
ylabel={$\varepsilon_\mathrm{truth}$ [-]},
axis background/.style={fill=white},
axis x line*=bottom,
axis y line*=left,
legend style={at={(0.05,-.6)}, anchor=south west, legend cell align=left, align=left, draw=white!15!black}
]

\addplot[area legend, draw=none, fill=mycolor1, forget plot]
table[row sep=crcr] {%
x	y\\
0	0.876847591924684\\
2	0.697321506630369\\
4	0.53201400996354\\
6	0.434000145098293\\
8	0.305077153922594\\
10	0.249228321208557\\
12	0.20503670542\\
14	0.151406685569126\\
16	0.123383778742612\\
18	0.0968519543011521\\
20	0.0788328855871202\\
22	0.0604260611735893\\
24	0.049624105405521\\
24	0.036752437096316\\
22	0.0479850004601975\\
20	0.0583642826383213\\
18	0.0759275335039192\\
16	0.0949626922408405\\
14	0.120277244081322\\
12	0.158908480188789\\
10	0.200246698391029\\
8	0.260656883188378\\
6	0.321034492669589\\
4	0.429683739050378\\
2	0.52528358160361\\
0	0.665654877925023\\
}--cycle;
\addplot [color=mycolor2, line width=1.0pt]
  table[row sep=crcr]{%
0	0.771251234924854\\
2	0.611302544116989\\
4	0.480848874506959\\
6	0.377517318883941\\
8	0.282867018555486\\
10	0.224737509799793\\
12	0.181972592804394\\
14	0.135841964825224\\
16	0.109173235491726\\
18	0.0863897439025357\\
20	0.0685985841127208\\
22	0.0542055308168934\\
24	0.0431882712509185\\
};
\addlegendentry{ALS}

\addplot[area legend, draw=none, fill=mycolor3, forget plot]
table[row sep=crcr] {%
x	y\\
0	0.220518013842314\\
2	0.209167653306451\\
4	0.201952105636459\\
6	0.183657440429724\\
8	0.166618016912241\\
10	0.145192545093765\\
12	0.140673005660588\\
14	0.116255244387256\\
16	0.100493378050189\\
18	0.0878752041125155\\
20	0.074367535583596\\
22	0.0586689202779857\\
24	0.0486774994506944\\
24	0.0362079046811488\\
22	0.0454847112397888\\
20	0.05422335213618\\
18	0.0667572202581809\\
16	0.0779678223091345\\
14	0.0864481570697237\\
12	0.0977294457128061\\
10	0.105177275392085\\
8	0.102324716757585\\
6	0.123638393662739\\
4	0.106434604003889\\
2	0.132859843166898\\
0	0.131548763157419\\
}--cycle;
\addplot [color=mycolor4, line width=1.0pt]
  table[row sep=crcr]{%
0	0.176033388499867\\
2	0.171013748236675\\
4	0.154193354820174\\
6	0.153647917046232\\
8	0.134471366834913\\
10	0.125184910242925\\
12	0.119201225686697\\
14	0.10135170072849\\
16	0.0892306001796615\\
18	0.0773162121853482\\
20	0.064295443859888\\
22	0.0520768157588872\\
24	0.0424427020659216\\
};
\addlegendentry{\cref{alg:noqr}}

\end{axis}
\end{tikzpicture}%
    \input{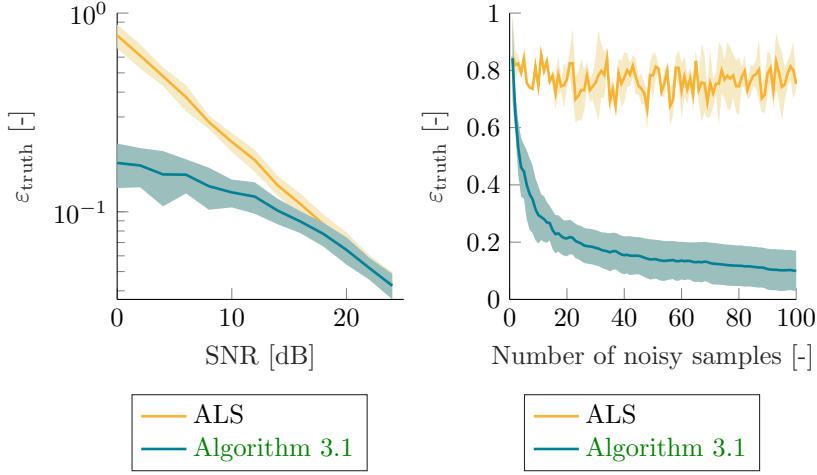}
    \caption{Left: Relative error $\varepsilon_\mathrm{truth}$ vs. signal-to-noise ratio with prior mean from \cref{eq:mvar} with $a=10^{-1}$ and prior covariance from \cref{eq:covvar} with to $b=10^{-1}$. Right: Comparison of the the relative error $\varepsilon_\mathrm{truth}$ between the ALS and the ALS in a Bayesian framework for different numbers of noisy samples.}
    \label{fig:errorcomparison}
\end{figure}

As shown, the ALS in a Bayesian framework gives better results if a good prior is available and it provides a measurement of the uncertainty and therefore additional valuable information. Also, if multiple noisy samples are available, ALS in a Bayesian framework significantly improves the estimate.

\subsection{Reconstruction of noisy image}
\label{sec:cat1}
To test \cref{alg:noqr} on an image processing problem, a cat image is reconstructed from an image corrupted with noise. \cref{fig:image2truth} shows the steps before applying \cref{alg:noqr}. The original image of size $256\times 256$ pixel is reshaped into an 8-way tensor, where each mode is of dimension 4. To obtain the TT-ranks, here we use the TT-SVD algorithm \cite{Oseledets2011}. It finds a TD that approximates the given tensor by setting an upper bound for the relative error. With an upper bound of $0.1$, the TT-ranks, depicted in \cref{fig:image2truth} are obtained. Finally, the ground truth is computed as the vectorized contracted TT. Now, ten noisy samples are formed with \cref{eq:noisyinstances} with a signal-to-noise ratio of $\mathrm{SNR}_\mathrm{dB}=0$.\\

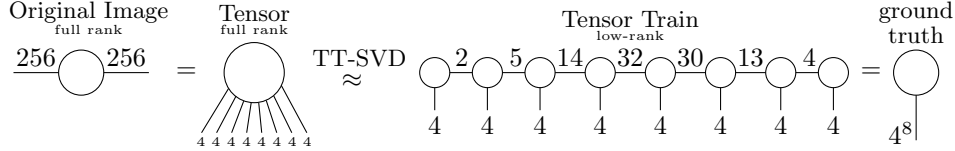
\begin{figure}
    \centering
    \begin{tikzpicture}
\draw  (-0.1,0.7) ellipse (0.3 and 0.3);
\draw (-0.4,0.7) -- (-1,0.7);
\draw (0.2,0.7) -- (0.8,0.7);
\node at (0.5,0.9) {256};
\node at (-0.7,0.9) {256};
\draw  (2.2,0.7) ellipse (0.4 and 0.4);
\draw (1.86,0.49) -- (1.5,-0.1);
\draw (1.96,0.39) -- (1.7,-0.1);
\draw (2.06,0.32) -- (1.9,-0.1);
\draw (2.16,0.3) -- (2.1,-0.1);
\draw (2.26,0.3) -- (2.3,-0.1);
\draw (2.36,0.33) -- (2.5,-0.1);
\draw (2.46,0.39) -- (2.7,-0.1);
\draw (2.56,0.53) -- (2.9,-0.1);
\node at (2.2,-0.22) {\tiny 4 4 4 4 4 4 4 4};
\draw  (4.6,0.7) ellipse (0.2 and 0.2);
\draw  (5.3,0.7) ellipse (0.2 and 0.2);
\draw  (6,0.7) ellipse (0.2 and 0.2);
\draw  (6.8,0.7) ellipse (0.2 and 0.2);
\draw  (7.6,0.7) ellipse (0.2 and 0.2);
\draw  (8.4,0.7) ellipse (0.2 and 0.2);
\draw  (9.2,0.7) ellipse (0.2 and 0.2);
\draw  (9.9,0.7) ellipse (0.2 and 0.2);
\draw (4.8,0.7) -- (5.1,0.7);
\draw (5.5,0.7) -- (5.8,0.7);
\draw (6.2,0.7) -- (6.6,0.7);
\draw (7,0.7) -- (7.4,0.7);
\draw (7.8,0.7) -- (8.2,0.7);
\draw (8.6,0.7) -- (9,0.7);
\draw (9.4,0.7) -- (9.7,0.7);
\draw (4.6,0.5) -- (4.6,0.2);
\draw (5.3,0.5) -- (5.3,0.2);
\draw (6,0.5) -- (6,0.2);
\draw (6.8,0.5) -- (6.8,0.2);
\draw (7.6,0.5) -- (7.6,0.2);
\draw (8.4,0.5) -- (8.4,0.2);
\draw (9.2,0.5) -- (9.2,0.2);
\draw (9.9,0.5) -- (9.9,0.2);

\node at (0,1.5) {\small Original Image};
\node at (0,1.28) {\tiny full rank};
\node at (2.2,1.5) {\small Tensor};
\node at (2.2,1.28) {\tiny full rank};
\node at (7.2,1.42) {\small Tensor Train};
\node at (7.2,1.2) {\tiny low-rank};
\node at (3.6,0.9) {\small TT-SVD};
\node at (4.96,0.9) {2};
\node at (5.68,0.9) {5};
\node at (6.4,0.9) {14};
\node at (7.2,0.9) {32};
\node at (8,0.9) {30};
\node at (8.8,0.9) {13};
\node at (9.58,0.9) {4};
\node at (4.6,0) {4};
\node at (5.3,0) {4};
\node at (6,0) {4};
\node at (6.8,0) {4};
\node at (7.6,0) {4};
\node at (8.4,0) {4};
\node at (9.2,0) {4};
\node at (9.9,0) {4};

\draw  (11,0.7) ellipse (0.3 and 0.3);
\draw (11,0.4) -- (11,-0.2);
\node at (10.8,-0.1) {$4^8$};
\node at (1.3,0.7) {=};
\node at (10.4,0.7) {=};
\node at (3.5,0.6) {$\approx$};
\node at (11,1.5) {\small ground};
\node at (11,1.2) {\small truth};
\end{tikzpicture}
    \caption{Computation of ground truth from original image: The original image of size $256\times 256$ pixel is reshaped into an 8-way tensor, where each mode is of dimension 4. Then, the TT-SVD algorithm \cite{Oseledets2011} with an upper bound for the relative error of $0.1$ is applied, resulting in the depicted TT-ranks. Finally, the ground truth is obtained as the vectorized contracted TT.}
    \label{fig:image2truth}
\end{figure}

\cref{fig:catinput}a) shows the original image and \cref{fig:catinput}b) the low-rank image, which is obtained by reshaping the low-rank TT from the TT-SVD into the size of the original image. \cref{fig:catinput}c) shows one exemplary noisy sample $\y$ reshaped into the dimensions of the original image. As a stopping criterion, we used the maximum number of iterations of 3. \Cref{fig:catCla} left shows the reconstruction of the image with the conventional ALS using one noisy sample and on the right using ten noisy samples. \Cref{fig:catBay} shows the reconstruction of the image inputting a random prior mean and a prior covariance on each core of $1000^2\id$ and using one and ten noisy samples. For the ALS in a Bayesian framework, it is shown that the image gets clearer with a higher number of noisy samples $\y$, confirmed by the decreasing relative error $\varepsilon_\mathrm{truth}$ from 0.3127 to 0.1478. The relative error of the conventional ALS only decreases slightly from 0.3664 to 0.3088.
\begin{figure}
    \centering
    \includegraphics[width=\textwidth]{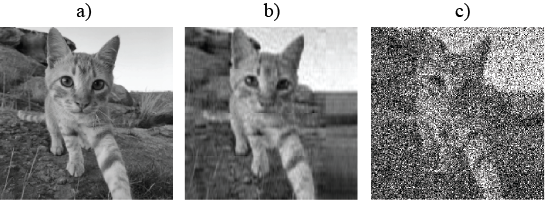}
    \caption{a) Original image, b) Image approximated with the TT-SVD algorithm \cite{Oseledets2011} with an upper bound for the relative error of $0.1$, c) One noisy sample (low-rank image corrupted with random noise).}
    \label{fig:catinput}
\end{figure}

\begin{figure}
    \centering
    \includegraphics[width=.6\textwidth]{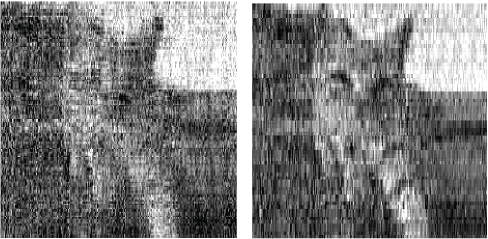}
    \caption{Reconstructed image with conventional ALS algorithm. Left: using one noisy sample. Right: using ten noisy samples.}
    \label{fig:catCla}
\end{figure}

\begin{figure}
    \centering
    \includegraphics[width=.6\textwidth]{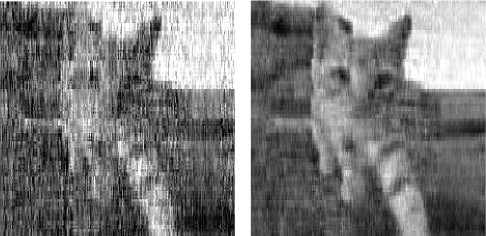}
    \caption{Reconstructed image with ALS in a Bayesian framework (\cref{alg:noqr}). Left: using one noisy sample. Right: using ten noisy samples.}
    \label{fig:catBay}
\end{figure}

\subsection{Large-scale experiment}
In this experiment, we demonstrate that \cref{alg:noqr} also works with larger tensors. The cat image from \cref{sec:cat1} in color is up-scaled via bi-cubic interpolation to obtain a $6000\times4000\times3$ tensor as depicted in \cref{fig:catlarge}a). Next, we find a low-rank approximation of the image by first applying the TKPSVD algorithm~\cite{Batselier2017b}. The TKPSVD decomposes a tensor $\mathbfcal{A}$ into a sum of multiple Kronecker products of $N$ tensors $\mathbfcal{A}^{(n)}_r$
\begin{equation*}
    \mathbfcal{A} = \sum_{r=1}^R \lambda_r\, \mathbfcal{A}^{(N)}_r \otimes \cdots \otimes \mathbfcal{A}^{(1)}_r
\end{equation*}
where $\lambda_r\in\mathbb{R}$. We approximate the image by taking only the term with the largest $\lambda_r$ and $N=5$,
\begin{equation*}
    \mathbfcal{A} \approx \lambda_{\max}\, \mathbfcal{A}^{(5)}_1 \otimes \mathbfcal{A}^{(4)}_1 \otimes \mathbfcal{A}^{(3)}_1 \otimes \mathbfcal{A}^{(2)}_1 \otimes \mathbfcal{A}^{(1)}_1,
\end{equation*}
where $\lambda_{\max}=\lambda_1$. The resulting Kronecker products is of dimensions
\begin{equation*}
    ( 375 \times 250 \times 3)\otimes (2 \times 2 \times 1)\otimes (2 \times 2 \times 1) \otimes (2\times 2 \times 1) \otimes (2 \times 2 \times 1),
\end{equation*}
as depicted in the top part of \cref{fig:catranks}. Secondly, $\mathbfcal{A}_1^{(5)}\in\mathbb{R}^{375\times250\times3}$ is further decomposed with the TT-SVD algorithm with an upper bound of the relative error of $0.08$, where the dimensions are factorized as shown in the lower part of \cref{fig:catranks}. The resulting low-rank approximation of the image is shown in \cref{fig:catlarge}b) and the noisy image, created with an $\mathrm{SNR}=-22$, is shown in \cref{fig:catlarge}c). We use \cref{alg:noqr} with a random prior mean and $\cov_i^0 = 10000^2\id$. \Cref{fig:catlarge}d) shows the reconstructed image after 30 iterations in \cref{alg:noqr}. The main computational bottleneck is the inversion of the covariance matrix of the largest TD component (line 4 of \cref{alg:noqr}). Thus, the number of elements of a TD component, dependent on its ranks, is the limiting factor for the computational complexity. In this case, the largest TT-core has $13\cdot25\cdot18=5850$ elements, see second last TT-core in the lower part of \cref{fig:catranks}.

\begin{figure}
    \centering
    \includegraphics[width=\textwidth]{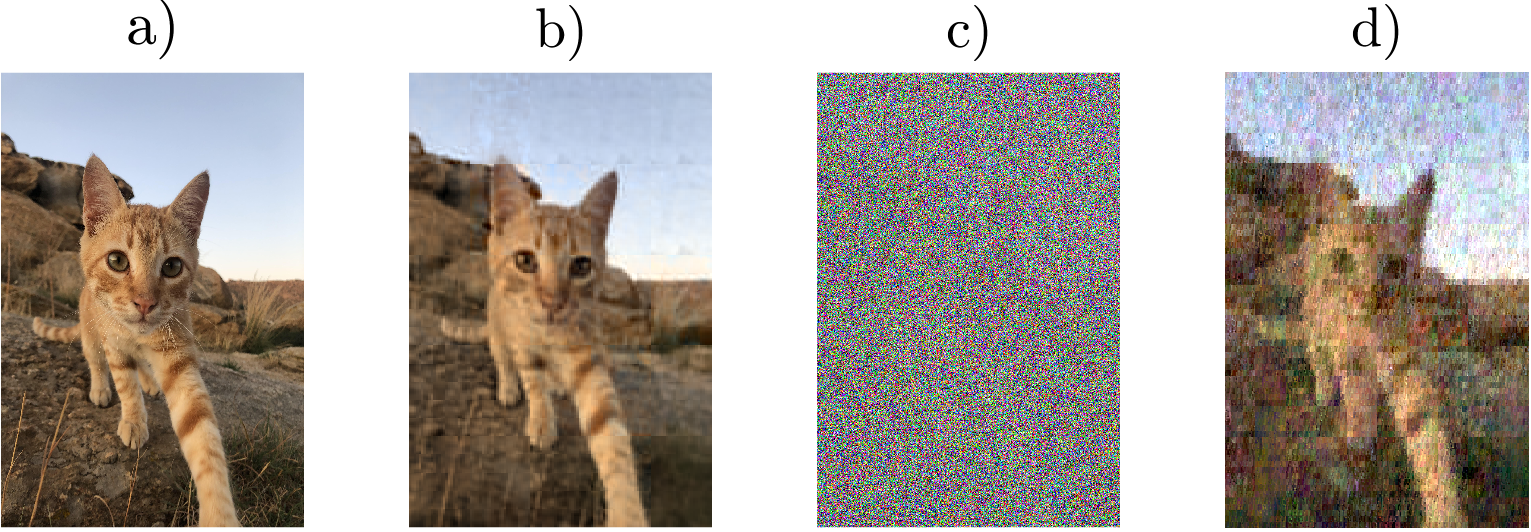}
    \caption{a) Original image, b) Low-rank image, c) Noisy image (low-rank image corrupted with random noise, with an $\mathrm{SNR=-22}$), d) Reconstructed image.}
    \label{fig:catlarge}
\end{figure}

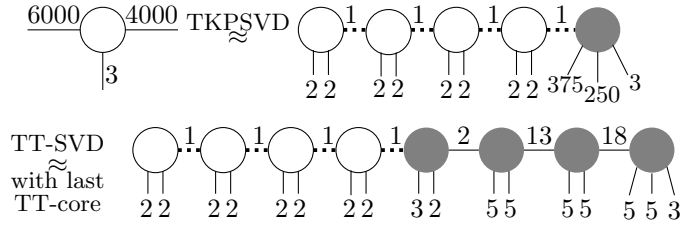
\begin{figure}
    \centering
    \begin{tikzpicture}
\draw  (-0.1,0.7) ellipse (0.3 and 0.3);
\draw (-0.4,0.7) -- (-1.1,0.7);
\draw (0.2,0.7) -- (0.9,0.7);
\node at (0.55,0.9) {4000};
\node at (-0.75,0.9) {6000};

\draw  (0.6,-0.9) ellipse (0.3 and 0.3);
\draw  (1.5,-0.9) ellipse (0.3 and 0.3);
\draw  (2.4,-0.9) ellipse (0.3 and 0.3);
\draw  (3.3,-0.9) ellipse (0.3 and 0.3);
\fill [color=gray]   (4.2,-0.9) ellipse (0.3 and 0.3);
\fill [color=gray]   (5.2,-0.9) ellipse (0.3 and 0.3);
\fill [color=gray]   (6.2,-0.9) ellipse (0.3 and 0.3);
\fill [color=gray]   (7.2,-0.9) ellipse (0.3 and 0.3);
\draw [dotted,line width=0.5mm](0.9,-0.9) -- (1.2,-0.9);
\draw [dotted,line width=0.5mm](1.8,-0.9) -- (2.1,-0.9);
\draw [dotted,line width=0.5mm](2.7,-0.9) -- (3,-0.9);
\draw [dotted,line width=0.5mm](3.6,-0.9) -- (3.9,-0.9);
\draw (4.5,-0.9) -- (4.9,-0.9);
\draw (5.5,-0.9) -- (5.9,-0.9);
\draw (6.5,-0.9) -- (6.9,-0.9);
\draw (0.7,-1.18) -- (0.7,-1.48);
\draw (1.6,-1.18) -- (1.6,-1.48);
\draw (2.5,-1.18) -- (2.5,-1.48);
\draw (3.4,-1.18) -- (3.4,-1.48);
\draw (4.3,-1.18) -- (4.3,-1.48);
\draw (5.3,-1.18) -- (5.3,-1.48);
\draw (6.3,-1.18) -- (6.3,-1.48);
\draw (7,-1.13) -- (6.9,-1.5);

\node at (-0.7,-0.8) {\small TT-SVD};

\node at (1.06,-0.7) {1};
\node at (1.98,-0.7) {1};
\node at (2.9,-0.7) {1};
\node at (3.8,-0.7) {1};
\node at (4.7,-0.7) {2};
\node at (5.7,-0.7) {13};
\node at (6.7,-0.7) {18};
\node at (0.72,-1.68) {2};
\node at (1.62,-1.68) {2};
\node at (2.52,-1.68) {2};
\node at (3.42,-1.68) {2};
\node at (4.32,-1.68) {2};
\node at (5.32,-1.68) {5};
\node at (6.32,-1.68) {5};

\node at (-0.7,-1.1) {$\approx$};
\draw (-0.1,0.4) -- (-0.1,-0.1);
\node at (0.02,0.1) {3};
\draw (7.2,-1.21) -- (7.2,-1.48);
\draw (7.4,-1.12) -- (7.5,-1.5);
\draw (0.5,-1.18) -- (0.5,-1.48);
\draw (1.4,-1.18) -- (1.4,-1.48);
\draw (2.3,-1.18) -- (2.3,-1.48);
\draw (3.2,-1.17) -- (3.2,-1.5);
\draw (4.1,-1.18) -- (4.1,-1.48);
\draw (5.1,-1.18) -- (5.1,-1.48);
\draw (6.1,-1.18) -- (6.1,-1.48);
\node at (0.48,-1.68) {2};
\node at (1.38,-1.68) {2};
\node at (2.28,-1.68) {2};
\node at (3.18,-1.68) {2};
\node at (4.08,-1.68) {3};
\node at (5.08,-1.68) {5};
\node at (6.08,-1.68) {5};

\draw  (2.8,0.7) ellipse (0.3 and 0.3);
\draw  (3.7,0.67) ellipse (0.3 and 0.3);
\draw  (4.6,0.7) ellipse (0.3 and 0.3);
\draw  (5.5,0.7) ellipse (0.3 and 0.3);
\fill  [color=gray](6.48,0.7) ellipse (0.3 and 0.3);
\draw [dotted,line width=0.5mm](3.1,0.7) -- (3.4,0.7);
\draw [dotted,line width=0.5mm](4,0.7) -- (4.3,0.7);
\draw [dotted,line width=0.5mm](4.9,0.7) -- (5.2,0.7);
\draw [dotted,line width=0.5mm](5.8,0.7) -- (6.18,0.7);
\draw (2.7,0.43) -- (2.7,0.1);
\draw (2.9,0.43) -- (2.9,0.1);
\draw (3.6,0.4) -- (3.6,0.08);
\draw (3.8,0.4) -- (3.8,0.08);
\draw (4.5,0.43) -- (4.5,0.1);
\draw (4.7,0.43) -- (4.7,0.1);
\draw (5.4,0.43) -- (5.4,0.1);
\draw (5.6,0.43) -- (5.6,0.1);
\draw (6.28,0.49) -- (6.08,0.1);
\draw (6.48,0.41) -- (6.48,0);
\draw (6.68,0.5) -- (6.88,0.1);
\node at (3.2,0.9) {1};
\node at (4.1,0.9) {1};
\node at (5.1,0.9) {1};
\node at (6,0.9) {1};
\node at (2.68,-0.1) {2};
\node at (2.92,-0.1) {2};
\node at (3.58,-0.08) {2};
\node at (3.82,-0.08) {2};
\node at (4.47,-0.1) {2};
\node at (4.71,-0.1) {2};
\node at (5.37,-0.1) {2};
\node at (5.61,-0.1) {2};
\node at (6.05,-0.03) {\small 375};
\node at (6.54,-0.16) {\small 250};
\node at (6.98,0) {\small 3};
\node at (1.7,0.6) {$\approx$};
\node at (1.7,0.8) {\small TKPSVD};
\node at (-0.7,-1.3) {\small with last };

\node at (6.9,-1.7) {5};
\node at (7.5,-1.7) {3};
\node at (7.2,-1.7) {5};
\node at (-0.7,-1.6) {\small TT-core};
\end{tikzpicture}
    \caption{Determination of the TT-ranks by computing a low-rank decomposition with the TKPSVD and then decomposing the last TT-core (gray) further with the TT-SVD.}
    \label{fig:catranks}
\end{figure}
\section{Conclusions}
We approached the computation of low-rank tensor decomposition from a Bayesian perspective. Assuming Gaussian priors for the TD components and Gaussian measurement noise and by applying a block coordinate descent, we were able to perform a tractable inference and compute the posterior joint distribution of the TD components. This leads to a probabilistic interpretation of the ALS. The distribution of the underlying low-rank tensor was computed with the unscented transform in tensor train format. We found that the relative error of the resulting low-rank tensor approximation depends strongly on the quality of the prior distribution. In addition, our method opens up for a recursive estimation of a tensor from a sequence of noisy measurements of the same underlying tensor. If no useful prior information is available, the method gives the same result as the conventional ALS. Our method will perform worse than the conventional ALS, if a small covariance in assumed for a bad prior mean. Future work could focus on incorporating the inference of the ranks which for the ALS are fixed and therefore need to be decided beforehand. Also, the method could be extended to a non-Gaussian prior and the UT algorithm could be further developed, e.g.\ by parallelizing the code to make it computationally more efficient for large data sets. 

\bibliographystyle{siamplain}
\bibliography{main.bib}
\end{document}